\documentclass[twoside,11pt]{article}

\RequirePackage[OT1]{fontenc}
\RequirePackage{amsthm,amsmath}
\RequirePackage[numbers]{natbib}
\usepackage{amsfonts,graphicx,ulem}
\usepackage[ruled]{algorithm2e}
\usepackage{algorithmic}
\usepackage{graphicx}
\usepackage{ulem}
\usepackage{jmlr2e}
\usepackage{xcolor} 

\usepackage{tikz}
\usetikzlibrary{shapes,arrows}
\usepackage{fancybox}

\jmlrheading{x}{}{x-xx}{x/xx}{xx/xx}{}{Zhigang Yao and Yuqing Xia}

\numberwithin{equation}{section}

\newtheorem{prop}[theorem]{Proposition}

\newcommand{\M}{{\cal M}}

\newcommand{\mP}{{\cal P}}
\newcommand{\I}{{\cal I}}
\newcommand{\R}{\mathbb{R}}
\newcommand{\E}{\mathbb{E}}
\newcommand{\Vol}{{\rm Vol}}
\newcommand{\YQ}[1]{{\textcolor{black}{#1}}}

\graphicspath{{figures/}}

\ShortHeadings{Manifold Fitting under Unbounded Noise}{Yao and Xia}
\firstpageno{1}

\begin{document}

\title{Manifold Fitting under Unbounded Noise}

\author{\name Zhigang Yao \email zhigang.yao@nus.edu.sg\\ 
       \addr Department of Statistics and Data Science \email \hspace*{0pt}\hfill zhigang.yao@cmsa.fas.harvard.edu\\
       \addr National University of Singapore\\
       21 Lower Kent Ridge Road\\
       Singapore 117546\\\\
      Center of Mathematical Sciences and Applications  \\ Harvard University\\
     20 Garden Street\\
     Cambridge USA   02138\\\\
     \name Yuqing Xia\email staxiay@nus.edu.sg\\
      \addr Department of Statistics and Data Science\\
      National University of Singapore\\
      21 Lower Kent Ridge Road\\
 Singapore 117546}

\editor{xxxx}

\maketitle
\begin{abstract}
In the field of non-Euclidean statistical analysis, a trend has emerged in recent times, of attempts to recover a low dimensional structure, namely a manifold, underlying the high dimensional data. Recovering the manifold requires the noise to be of a certain concentration and prevailing methods address this requirement by constructing an approximated manifold that is based on the tangent space estimation at each sample point. Although theoretical convergence for these methods is guaranteed, the samples are either noiseless or the noise is bounded. However, if the noise is unbounded, as is commonplace, the tangent space estimation at the noisy samples will be blurred – an undesirable outcome since fitting a manifold from the blurred tangent space might be more greatly compromised in terms of its accuracy.
In this paper, we introduce a new manifold-fitting method, whereby the output manifold is constructed by directly estimating the tangent spaces at the projected points on the latent manifold, rather than at the sample points, thus reducing the error caused by the noise. Assuming the noise is unbounded, our new method has a high probability of achieving theoretical convergence, in terms of the upper bound of the distance between the estimated and latent manifold. The smoothness of the estimated manifold is also evaluated by bounding the supremum of twice difference above. Numerical simulations are conducted as part of this new method to help validate our theoretical findings and demonstrate the advantages of our method over other relevant manifold fitting methods. Finally, our method is applied to real data examples.
\end{abstract}

\section{Introduction} \label{sec:intro}
Linearity has been viewed as a cornerstone in the development of statistical methodology. For decades, prominent progress in statistics has been made with regard to linearizing the data and the way we analyze them. More recently, the phenomenon of high-throughput data, which share a high dimensional characteristic in their varying forms, has become more commonplace. Although each data point usually represents itself as a long vector or a large matrix, in principle they all can be viewed as points on or near an intrinsic manifold. Moreover, modern data sets no longer comprise samples of real vectors in a real vector space but samples of much more complex structures, assuming values in spaces that are naturally not (Euclidean) vector spaces. We are verily witnessing an explosion in the volumn of “complex data'' with a geometric structure and, therefore, a growing need for statistical analysis, utilizing the nature of the data space.

The manifold hypothesis has been carefully studied in \cite{fefferman2016testing}. Here, we only present several relevant examples to make sense of that hypothesis intuitively: the high dimensional data samples tend to lie near a lower dimensional manifold embedded in the ambient space. The classical Coil20 dataset \citep{nene1996columbia}, which contains images of 20 objects, may be used as an example. For each object, images are taken every 5 degrees as the object is rotated on a turntable, and each image is of size $32 \times 32$. In this case, the dimension of ambient space is the number of pixels, which is 1024, while the latent intrinsic structure can be compactly described with the angle of rotation. In addition to Coil20, such a structure can be found in many other data collections. In seismology, two-dimensional coordinates of earthquake epicenters are located along a one-dimensional fault line. In face recognition, high-dimensional facial images are dependent on lighting conditions \citep{GeBeKr01} or head orientations \citep{happy2012video}.

Given this form of data collection, \YQ{a natural problem arises: how can we fit a manifold to this data collection?} The aim of manifold fitting
is to represent the latent manifold as an embedded sub-manifold of the ambient space. Once the latent manifold is learned, various types of analyses can be carried out based on it, such as denoising the observed samples by projecting them to the learned manifold \citep{gong2010locally}, generating new data samples from the manifold \citep{radford2015unsupervised}, classifying samples according to the manifold \citep{Yao2019principal}, and detecting fault lines for seismological purposes \citep{Yao2019Fixed}. \YQ{These manifold-based techniques represent powerful tools for understanding and working with complex data structures.}

\YQ{In addition to manifold fitting, dimension reduction constitutes another crucial branch of manifold learning. Over the past two decades, a litany of dimension reduction methods have emerged, each aiming to uncover the intrinsic structure of data by identifying its lower-dimensional embedding, as discussed in the review by  \cite{ma2011manifold}. Unlike manifold fitting, however, these methods primarily focus on mapping data from the ambient space to a lower-dimensional one. Consequently, the outputs of most dimension reduction methods consist of low-dimensional embeddings rather than points in the ambient space, although for applications such as denoising and data generation, relying solely on low-dimensional embeddings may not suffice.}

\YQ{The limitations of dimension reduction and the potential applications of manifold fitting underline the value of  formulating the manifold fitting problem,} as follows. Suppose the observed data samples $X=\{x_i \in \R^D\}_{i=1}^N$ are in the form
\[
    x_i = y_i + \xi_i,
\]
where $y_1, \cdots, y_N$ are unobserved variables drawn from the uniform distribution supported on the latent manifold $\M$ with dimension $d < D$. Generally, $\M$ is assumed to be a compact and smooth sub-manifold embedded in the ambient space $\R^D$. The precise conditions on $\M$ will be detailed in Section \ref{sec:notation}. The uniform distribution assumption of $y_i$ sampled from $\M$ is the same as those used in the related works \citep{genovese2012c,genovese2014nonparametric,mohammed2017manifold}.  Here, $\xi_1, \cdots \xi_N$ are drawn from a distribution $G$. The assumptions about the noisy distribution $G$ differ among the related work. The simplest assumption is that the observed samples are noiseless \citep{fefferman2016testing,mohammed2017manifold}. However, some literature assumes that the noise is distributed in a bounded region centered at the origin, which means that the observed samples are located in a tube centered at $\M$ \citep{genovese2012b}. Other literature, such as \cite{genovese2012c,genovese2014nonparametric,pmlr-v75-fefferman18a}, assumes $G$ to be a Gaussian distribution supported on $\R^D$, whose density at $\xi$ is 
\begin{align}\label{func:noise}
(\frac{1}{2\pi\sigma^2})^{\frac{D}{2}} \exp (-\frac{\|\xi\|_2^2}{2\sigma^2}).
\end{align}
The tail of the Gaussian distribution might make the theoretical analysis more challenging than in the previous two cases. Strictly speaking, previous manifold fitting methods have not directly addressed this problem, nor have they proved the convergence of the fitted manifold under this assumption. In this paper, we are concerned with the Gaussian assumption of noisy distribution, denoting it as $G_\sigma$ to stress the deviation parameter $\sigma$ hereafter. Under the above settings, the goal is to produce a smooth manifold $\M_{\rm out}$ convergent to $\M$. \YQ{Specifically, if $\sigma$ is sufficiently small, one could derive $\M_{\rm out}$ such that $d(x,\M) \leq O(\sigma)$ holds for any arbitrary $x \in \M_{\rm out}$.}
In particular, $\M_{\rm out}$ \YQ{ converges to $\M$} when $\sigma \to 0$. \YQ{The convergence with respect to $\sigma$} is the domain that \cite{pmlr-v75-fefferman18a} is built on, \YQ{although its final theoretical result is expressed through sample complexity.}

\subsection{Related work}
Methodological studies for manifold fitting can be traced back to works from several decades ago on the principal curve \citep{hastie1989principal}, with every point on the principal curve/surface defined as the conditional mean value of the points in the orthogonal subspace of the principal curve. Based on \cite{hastie1989principal}, many other principal-curve algorithms have been proposed, such as those of \cite{banfield1992ice, stanford2000finding, verbeek2002k}, each attempting to achieve lower estimation bias and improved robustness. More recently, \cite{ozertem2011locally} describes the principal curve in a seemingly different way albeit in a probabilistic sense. In \YQ{the work by} \cite{ozertem2011locally}, every point on the principal curve/surface is the local maximum, not the expected value, of the probability density in the local orthogonal subspace. This definition of the principal curve/surface is formulated as a ridge of the probability density. Although it has been demonstrated that these proposed methods produce acceptably accurate estimates in many simulated cases, they do not, however, provide a theoretical analysis for estimating accuracy nor the curvature of the output manifold in general cases, with the exception of special cases such as elliptical distributions.

Recently, some works have focused on the theoretical analysis for manifold fitting. In particular, \cite{genovese2012b} and \cite{genovese2012c} establish the upper bounds on the Hausdorff distance between the output and latent manifold under various noise settings, although they do not offer any practical estimators. \cite{genovese2012a} proposes an estimator which is computationally simple, and whose convergence is guaranteed but its conclusions hold only when the noise is supported on a compact set. \cite{genovese2014nonparametric} focuses on the ridge of the probability density introduced by \cite{ozertem2011locally}, and proposes a convergent algorithm. It is worth noting that the data in \cite{genovese2014nonparametric} was assumed to be blurred by homogeneous Gaussian noise, an assumption that is more general than that made in \cite{genovese2012a}. 
\cite{Boissonnat2014} proposes an algorithm based upon Delauney complexes, whose 
convergence was analyzed by \cite{Aamari2018}.
\cite{Aamari2019} presented an algorithm to estimate a point on the manifold, its tangent and second form. Based on these, they approximated the latent manifold by a mere union of polynomial patches and gave convergence rate for noise-free and tubular noise models. \cite{Aizenbud2021} presents an algorithm that showed convergence to the manifold and its tangent bundle, even with tubular noise. 
However, none of the methods outlined above are guaranteed to output an actual $d$-dimensional manifold with certain smoothness.

To overcome this issue, some studies on manifold fitting have sought to determine how curved the output manifold is. In the spirit of \cite{ozertem2011locally} and \cite{genovese2014nonparametric}, \cite{fefferman2016testing} and \cite{mohammed2017manifold} also took the ridge set into consideration, the former focusing on theoretical analysis, the latter on practical algorithms. Specifically, rather than focusing on the probability density function, they both chose to work with the approximate square-distance functions (asdf), approximating the latent manifold by the ridge of the asdf. The theoretical bounds for the manifold fitting have also been considered in \cite{fefferman2016testing} and \cite{mohammed2017manifold}, but for only noiseless data; that is, as long as the asdf meets certain regularity conditions, the researchers show that the output of the algorithm is a manifold with bounded reach, and the output manifold is arbitrarily close in Hausdorff to the latent one.

To deal with manifold fitting with noise, \cite{pmlr-v75-fefferman18a} proposes a new approach to fit a putative manifold under Gaussian noise.  Unlike other methods, which use the entire sample set, the method of \cite{pmlr-v75-fefferman18a} involves subsampling first such that the number of used samples can be bounded above by $e^D$. Under this constraint, the noise is supported on a bounded set with high probability. Given this, the application of \cite{pmlr-v75-fefferman18a} is feasible with the bounded noise although  the constraint on the sample size is problematic in that the upper bound does not go to zero even with a sufficient number of available samples and the variance of Gaussian noise diminishes. Therefore, the \YQ{problem} is not essentially addressed when the support of noise is unbounded and so creates room for the manifold-fitting problem to arise, especially from the theoretical side.

\subsection{Motivation}\label{sec:motivation}
In this paper, we attempt to evaluate the convergence and smoothness of $\M_{\rm out}$. Of the aforementioned works,  it is those of \cite{mohammed2017manifold} and \cite{pmlr-v75-fefferman18a} are most relevant to our study. \YQ{This section explains these two methods geometrically and analyzes their limitations, which impels us to establish a more accurate manifold-fitting method.}


\begin{figure}[htbp]
\centering
\includegraphics[width=0.48\textwidth]{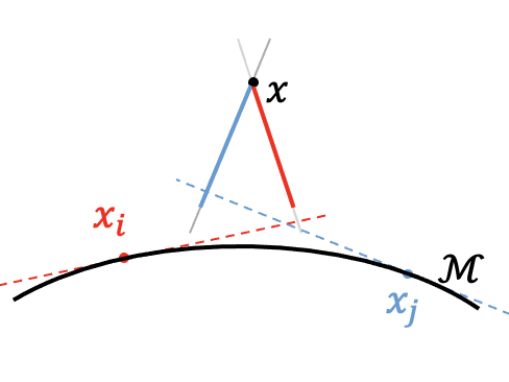}
\includegraphics[width=0.48\textwidth]{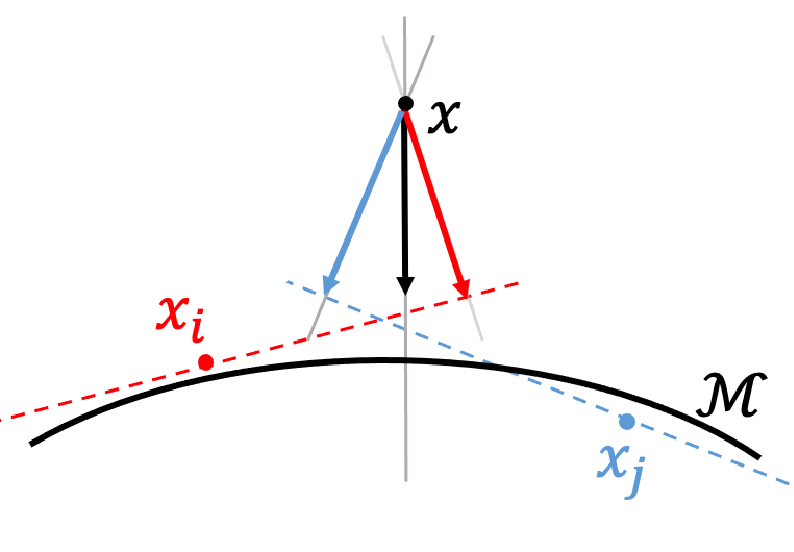}
\caption{A toy example to illustrate the methods \YQ{by}  \cite{mohammed2017manifold} (left panel) and \cite{pmlr-v75-fefferman18a} (right panel), \YQ{where the black curve is a local part of $\M$,  $x$ is a point off $\M$, and the dots $x_i$ and $x_j$ represent two samples in the neighborhood of $x$. Unlike those in the right panel, the samples in the left panel are on $\M$, as \cite{mohammed2017manifold} focus on the noiseless case.}}\label{fig:comparison}
\end{figure}

\YQ{The left panel of Figure \ref{fig:comparison} illustrates the method of \cite{mohammed2017manifold}, whose essence is to define an approximate squared-distance function (asdf) to $\M$ and estimates $\M$ by the ridge set of the asdf. Specifically, for a given point $x$, the asdf at $x$ is defined as the weighted average of squared distances from $x$ to the discs (the dashed lines) centered at the sample points (the blue and red dots). The right panel of Figure \ref{fig:comparison} illustrates the method of  \cite{pmlr-v75-fefferman18a}.} Its key idea is to approximate the bias from $x$ to $\M$ for any arbitrary $x$ and define the output manifold as points with zero bias.  To obtain the approximation of bias from $x$ to $\M$, \cite{pmlr-v75-fefferman18a} calculates the \YQ{weighted average bias from $x$ to the discs (the dashed lines) centered at the sample points}, and projects the average bias by the estimated orthogonal projection onto the normal space of $\M$ at $x^*$ (the gray solid line).



\YQ{Due to the usage of discs centered at the sample points, the effectiveness of each method depends on just how accurately these discs capture the local structure of the manifold. However, if the sample points are significantly perturbed by unbounded noise and deviate significantly from the latent manifold $\M$, these discs also deviate far from $\M$}. Hence, in this scenario with unbounded noise, the methods proposed by \cite{mohammed2017manifold} and \cite{pmlr-v75-fefferman18a} \YQ{may encounter difficulties} in fitting a manifold.


\YQ{Even if the sample points lie on the manifold, say $x_i \in \M$, the disc centered at $x_i$ captures the local manifold at $x_i$ rather than the local manifold at $x^*$. The deviation between $x_i$ and $x^*$ also introduces an approximation error in the distance/bias from $x$ to $\M$. As shown in Figure \ref{fig:comparison}, both the solid red line and the solid red arrow are shorter than the distance/bias from $x$ to $\M$, and the average between the red and blue one cannot address this issue.} 


\YQ{The limitations of the two aforementioned methods suggest that estimating the manifold based on discs centered at sample points may reduce the fitting accuracy, especially in the scenarios with unbounded noise. It is this finding that compels us to invent a new manifold-fitting method.}



\subsection{Main contribution}
From a statistical viewpoint, there is a pressing need for the development of a practical estimator with theoretical bounds satisfying the following requirements simultaneously, and which improves \YQ{on} the requirements of \cite{pmlr-v75-fefferman18a}:
\begin{itemize}
    \item The support of noise is unbounded.
    \item The estimator shares a similar geometric property to $\M$.
    \item For any arbitrary $x \in \M_{\rm out}$, the distance between $x$ and $\M$ is bounded above provided $N$ is sufficiently large and $\sigma$ is sufficiently small. In particular, the distance goes to zero as noise disappears.
    \item The smoothness of $\M_{\rm out}$ is mathematically guaranteed.
\end{itemize}
In this paper, we propose a novel approximation $f(x)$ to the bias from any point $x$ to $\M$ and fit the latent manifold $\M$ in the ambient space as the points with $f(x) = \bf{0}$, \YQ{where $\bf{0}$ presents a zero vector.}
Practically, such an output manifold can be achieved by solving the minimization $\|f(x)\|_2^2$ via gradient descent. 
This paper provides two main contributions in this aspect, the first being the theoretical analysis satisfying the four requirements above as follows:
\begin{itemize}
    \item The noise is assumed to be drawn from the Gaussian distribution $G_\sigma$ defined in (\ref{func:noise}).
    \item Any arbitrary neighborhood of $\M_{\rm out}$ is a $d$-dimensional manifold.
    \item For any $x \in \M_{\rm out}$,  $d(x, \M) \leq O(\sigma)$ given a large-enough dataset. Thus, $\M_{\rm out}$ converges to $\M$ for an increasingly large sample size and diminishing noise.
    \item The twice difference of $\M_{\rm out}$ is bounded above by $O(\frac{1}{\sqrt{\sigma}})$.
\end{itemize}
The second important contribution of this paper is the performance of our estimator in practice. As illustrated in Figures \ref{fig:comparison} and \ref{toyexample}, the bias from a point $x$ to $\M$ is approximated better than by the other two relevant methods. Numeric results in Section \ref{sec:experiment} demonstrate the improved performance, which further suggests that our method outputs the approximated manifold to the latent one.

\subsection{Organization}
The rest of the paper is organized as follows. Section \ref{sec:our_method} includes the formulation of our approximation $\M_{\rm out}$ to the latent manifold $\M$ . After that, the convergence and smoothness of $\M_{\rm out}$ \YQ{are} analyzed in Theorem \ref{thm:Hdist} and Theorem \ref{thm:reach_out}, respectively.
Section \ref{sec:bounds_f} studies the function $f$ defined in (\ref{fun:out_ours}) and determines the properties of its kernel space, the first and second derivatives. Based on these properties of $f$, the proofs of Theorem \ref{thm:Hdist} and Theorem \ref{thm:reach_out} are derived in 
Section \ref{sec:bounds_M} with numeric examples listed in Section \ref{sec:experiment} .

\section{Proposed method} \label{sec:our_method}

\subsection{Content and notations} \label{sec:notation}

Throughout this paper, the latent manifold is denoted as $\M$ and our approximation to $\M$ is denoted as $\M_{\rm out}$. For a set $A \subset \R^D$ and a point $x \in \R^D$, $\Pi_A x$ denotes the projection of $x$ onto $A$, namely the nearest point in $A$ to $x$. Hence $\Pi_{\M} x$ is the projection of $x$ onto the latent manifold. If there is no ambiguity, we might use $x^*$ instead of $\Pi_\M x$ for simplicity. The distance between $x$ and $A$, denoted by $d(x,A)$, is the Euclidean distance between $\Pi_A x$ and $x$. For any $x^* \in \M$, $T_{x^*} \M$ denotes the tangent space of $\M$ at $x^*$ and \YQ{$\Pi_{x^*}$} denotes the orthogonal projection onto the normal space of $\M$ at $x^*$.
\YQ{We will make frequent use of the lower-cases $c, c_0, c_1,$ etc. and upper-cases $C, C_0, C_1$ etc., in the rest of this paper with the lower-cases  denoting generic constants less than $1$, while the upper-cases denote generic constants greater than $1$. Values of the generic constants may change from line to line.}
By constants, we mean they are independent of the radius $r$, the standard deviation $\sigma$ or $x$, \YQ{while the constants may depend on some other constants used to characterize the manifold, such as the reach of $\M$.}

We denote $B_D(x, r)$ as the Euclidean ball in $\mathbb{R}^D$ centered at $x$ of radius $r$, which defines a neighborhood of $x$. The index set $I_{x,r}$ is defined as the indices of the sample points in $B_D(x,r)$, and $|I_{x,r}|$ denotes the cardinality of $I_{x,r}$. As given in (\ref{func:noise}), $\sigma$ represents the standard deviation of noise. Throughout this paper, we assume
\begin{align}\label{ass:r_sigma}
r = O(\sqrt{\sigma}), \quad \sigma < 1
\end{align}
without loss of generality, otherwise the data could be rescaled so that $\sigma < 1$ holds. \YQ{Here $r=O(\sqrt{\sigma})$ means that there exist constants $c$ and $C$ such that $c\sqrt{\sigma} \leq r \leq C\sqrt{\sigma}$. Noticing $\sigma < 1$, we obtain
\begin{align}\label{ineq:bound_r}
r \leq C\sqrt{\sigma} < C.
\end{align}
This means $r$ can be bounded above by certain constant. In subsequent proofs, under the premise that it does not affect the final precision of conclusive upper bound, we will relax some $r$ to $C$ in order to simplify the proof.} 

The latent manifold $\M$ is supposed to be boundaryless, compact, $d$-dimensional, and twice differentiable, with a reach bounded by $\tau>0$. The concept reach is a measure of the regularity of the manifold, first introduced by Federer \citep{federer1959curvature} as follows:
\begin{definition}[Reach]\label{def:reach}
Let $\M$ be a closed subset of $\R^D$. The reach of $\M$, denoted by ${\rm reach}(\M)$, is the largest number $\tau$ to have the property that any point at a distance $r < \tau$ from $\M$ has a unique nearest point in $\M$.
\end{definition}

An important understanding of reach is that it is a twice differentiable quantity if the manifold is treated as a function. Specifically, if $\gamma$ is an arc-length parametrized geodesic of $\M$, then for all $t$, $\|\gamma''(t)\|\leq 1/\tau$ according to \cite{niyogi2008finding}. As a twice differentiable quantity, it is easy to understand that the reach describes how flat the manifold is locally. For example, the reach of a sharp cusp is zero, and the reach of a linear subspace is infinite. Thus, it is natural that the reach measures how close a manifold is to the tangent space locally. The following proposition by \cite{federer1959curvature} explains this phenomenon:
\begin{prop}\label{prop:reach}
\begin{align}
{\rm reach}(\M)^{-1} = \sup \bigg\{\frac{2d(y,T_x\M)}{\|x-y\|_2^2}| x,y \in \M, x \neq y\bigg\}
\end{align}
\end{prop}
We emphasize that if ${\rm reach}(\M) > 0$, the error between $\M$ and $T_x\M$ at $y$ is of a higher order than $\|x-y\|_2$. Thus, in a small-enough neighbor of $x$, we can estimate $\M$ by $T_x\M$ with negligible error, which is the foundation of our approximation. 

The approximation $\M_{\rm out}$ is defined using the noisy sample points $\{x_i\}_{i=1}^N$. The number of sample points should be sufficiently large such that $B_D(x,r)$ contains enough sample points. Proposition \ref{prop:samplesize} claims the relationship between $|I_{x,r}|$ and $N$.
\begin{prop}\label{prop:samplesize}
Suppose $x$ satisfies $d(x,\M)\leq cr$ with some $c<1$. There exist constants $c'$ and $C$ such that $|I_{x, r}| \geq c'r^dN$ \YQ{with} probability at least $1 - C/\sqrt{N}$.
\end{prop}
Proof of proposition \ref{prop:samplesize} is given in Appendix \ref{proof:samplesize}. Based on this proposition, the requirement on $|I_{x,r}|$ can be transformed to the requirement on $N$ for further analysis in later sections. Specifically, $N$ is required to be a sufficiently large quantity in the order of $O(r^{-(d+2)})$.

\subsection{Definition of the approximated manifold} 

\YQ{This section introduces a novel method for estimating the bias $f(x)$ from a point $x$ to the latent manifold $\M$, and defines the approximated manifold $\M_{\rm out}$ as the points satisfying $f(x) = \bf{0}$. Unlike the aforementioned methods, which rely on discs centered at sample points to estimate the bias $f(x)$, here we build upon the fact that a Riemannian manifold can be locally treated as an affine space and calculate the bias from $x$ to $T_{x^*}\M$ as an equivalent measure of the bias $f(x)$ from $x$ to $\M$. Thus, the key to addressing such a bias is to find an affine space $\{x': \YQ{\Psi_x^\alpha}(x'-\bf{b})\}$ approximating $T_{x^*}\M$, where $\YQ{\Psi_x^\alpha}$ estimates the orthogonal projection onto the normal space at $x^*$ and $\bf{b}$ estimates one points in $T_{x^*}\M$.
}

In order to approximate the orthogonal projection onto the normal space of $\M$ at $x^*$, $\Psi_x^\alpha$ is defined as the weighted average of $\{P_{x_i}\}_{i \in I_{x,r}}$, where $P_{x_i}$ is the orthogonal projection perpendicular to the first $d$ principal components in $B_D(x_i, r^{\prime})$. Mathematically, $P_{x_i} = V_{\bot}V_{\bot}^T$, where $V_\bot$ is the orthogonal component of $V$ and $V$ is the $D \times d$ matrix whose columns are the eigenvectors corresponding to the largest $d$ eigenvalues of $\sum_{j \in I_{x_i, r'}} (x_j-x_i)(x_j-x_i)^T$. The radius $r^\prime$ should be sufficiently large, so that the intersection of $B_D(x_i,r')$ and $\M$ is nonempty. Further analysis in Section \ref{sec:Pxi} explains that we need $r' \geq 2r$.

As the weighted average of $\{P_{x_i}\}_{i \in I_{x,r}}$, 
\begin{align}\label{equ:Pi_x_ours}
   \YQ{\Psi_x^\alpha} = \Pi_{\rm hi}(A_x), \quad A_x= \sum_{i \in I_{x,r}} \alpha_i(x) P_{x_i}.
\end{align}
Here, $\Pi_{\rm hi}(A)$ denotes the projection onto the span of the eigenvectors corresponding to the largest $D-d$ eigenvalues of $A$. Specifically, $\Pi_{\rm hi}\big(A\big) = VV^T$, $V$ is a $D\times(D-d)$ matrix whose columns are the eigenvectors corresponding to the largest $D-d$ eigenvalues of $A$. 
Further, the weights $\alpha_i:\R^{D} \to \R$ in (\ref{equ:Pi_x_ours}) are defined as follows:
\begin{align}\label{def:alpha}
    \tilde{\alpha}_i(x) =
    \begin{cases}
    \left(1 - \frac{\|x - x_i\|_2^2}{r^2}\right)^\beta, & x \in B_D(x_i,r) \\
    0, & {\mbox otherwise}
    \end{cases}
    , \quad
    \tilde{\alpha}(x) = \sum_i \tilde{\alpha}_i(x), \quad
    \alpha_i(x) = \frac{\tilde{\alpha}_i(x)}{\tilde{\alpha}(x)},
\end{align}
 with $\beta \geq 2$ a fixed integer guaranteeing $f(x)$ in (\ref{fun:out_ours}) to be twice differentiable. 
 
\YQ{Under the assumption that a manifold can be approximated well by an affine space locally, samples in the neighborhood of $x$ lie close to $T_{x^*}\M$, with the exception of noise. Therefore, a convex combination of these samples also lies close to $T_{x^*}\M$. Thus, we can estimate $\bf{b}$ using the average of sample points in the neighborhood of $x$.} Recalling the weights in (\ref{def:alpha}), we formulate ${\textbf b} = \sum_{i \in I_{x,r}} \alpha_i(x)x_i$ as the weighted average of sample points in the neighborhood of $x$. Then the bias from $x$ to the space $\{x': \YQ{\Psi_x^\alpha}(x'-\bf{b})\}$ is
\begin{align}\label{fun:out_ours}
f(x): \mathbb{R}^D \to \mathbb{R}^D, \quad f(x) = \YQ{\Psi_x^\alpha} \bigg(x - \sum_{i \in I_{x,r}} \alpha_i(x)x_i \bigg).
\end{align}
Finally, the approximation is defined as
\begin{align}\label{out:ours}
    \M_{\rm out} = \{x: d(x,\M)\leq cr, \ f(x) = {\bf 0}, c<1 \},
\end{align}
that is, the points with zero bias. 
By Definition 11 of \cite{fefferman2016testing}, $\tilde{\M} = \{x : d(x,\M)\leq cr \}$ is a manifold. Restrict $f$ to $\tilde{\M}$. When ${\bf 0}$ is regular, the preimage $f^{-1}({\bf 0}) = \M_{\rm out} \subset \tilde{\M}$ is a smooth submanifold. So we call $\M_{\rm out}$ as the approximated manifold in the paper. Further characterization of the approximated manifold will be discussed in Theorem \ref{thm:manifold}.

The definition of $\M_{\rm out}$ is practical. Theorem \ref{thm:Hdist} in the next section claims that $\M_{\rm out}$ approximates $\M$ in the order of $O(r^2)$. This means if we have an initial estimator of $\M$ with error $cr$, then we could achieve a better estimator of $\M$ using the definition of $\M_{\rm out}$. In practice, we solve the minimization $\|f(x)\|_2^2$ via the gradient descent method given the initial estimator, and the output of  the gradient descent method approximates $\M$ in the order of $O(r^2)$ better than the initial guess.

\begin{figure}[htbp]
\centering\includegraphics[width=3.5in]{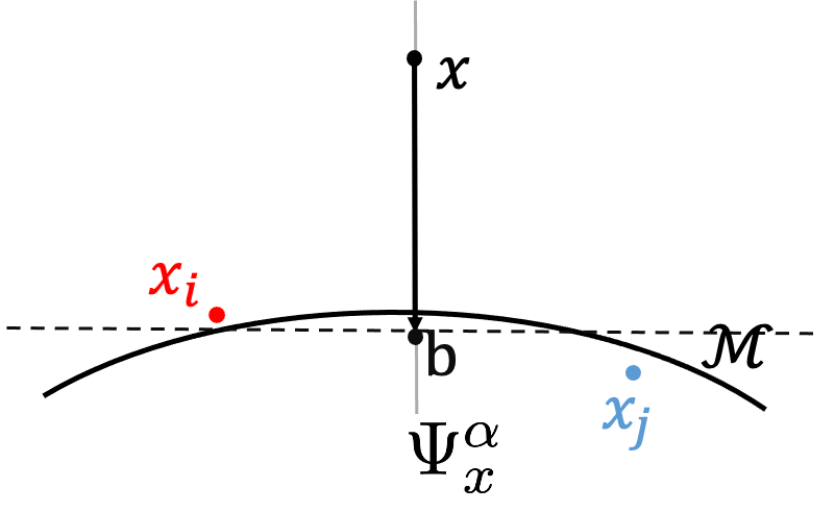}
\caption{A toy example to illustrate \YQ{the methods} in our method. $\YQ{\Psi_x^\alpha}$ is used to estimate the orthogonal projection onto the normal space of $\M$ at $x^*$, the black dot $\bf{b}$ is used to estimate a point in $T_{x^*}\M$. Then the space $\{x^{\prime}: \YQ{\Psi_x^\alpha}(x^{\prime}-\bf{b})\}$,  illustrated as the black dashed line, approximates $T_{x^*}\M$, and the bias from $x$ to the black dashed line is the estimated bias from $x$ to $\M$, geometrically illustrated as the black arrow.} \label{toyexample}
\end{figure}

The advantages of our method are twofold. \YQ{First, we introduce $\YQ{\Psi_x^\alpha}$  directly, thus capturing the local structure of manifold at $x^*$, while the aforementioned methods capture the local structure of manifold near the sample points. Second}, we approximate the manifold using a space passing $\bf{b}$ instead of any sample point. Benefitting from the mutual offset of noise, $\bf{b}$ hardly deviates far away from $T_{x^*}\M$ even if the noise is unbounded. As a result, we can expect $\{x^{\prime}: \YQ{\Psi_x^\alpha} (x^{\prime} - \bf{b})\}$ to be a better approximation to the local structure of manifold at $x^*$, which guarantees the bias from $x$ to $\{x^{\prime}: \YQ{\Psi_x^\alpha} (x^{\prime} - \bf{b})\}$ is a better approximation to the bias from $x$ to $\M$. The toy example in Figure \ref{toyexample} illustrates the superiority of our method. The black arrow in Figure \ref{toyexample} is almost the bias from $x$ to $\M$, while both the average length of the solid lines in the left panel of Figure \ref{fig:comparison} and the black arrow in the right panel of Figure \ref{fig:comparison} are shorter than the ideal one.

\subsection{Convergence and smoothness of the approximated manifold}
In Theorem \ref{thm:manifold}, we prove any arbitrary neighborhood of $\M_{\rm out}$ is a $d$-dimensional manifold in high probability. In Theorem \ref{thm:Hdist}, we characterize the convergence of $\M_{\rm out}$ in the probability $\delta_0(1-\delta)^2$, where we denote 
\begin{align}\label{eq:delta_0}
\YQ{\delta_0 = 1-d\exp\{\frac{-c{r}^{d+2}N}{2\ln 2}\}}
\end{align}
for convenience. When $N$ is sufficiently large as we set, $\delta_0$ is a high probability. Theorem \ref{thm:Hdist} tells us that if $r = O(\sqrt{\sigma})$ is sufficiently small, $\M_{\rm out}$ is a good estimator to $\M$. Moreover, Corollary \ref{coro:Hdist} tells us that the approximated manifold $\M_{\rm out}$ converges to the latent manifold $\M$ as $\sigma \to 0$.

\begin{theorem}\label{thm:manifold}
Given $\delta > 0$ and any arbitrary $x \in \M_{\rm out}$, there exists $\epsilon$ such that $\M_{\rm out} \cap B_D(x,\epsilon)$ is a $d$-dimensional manifold \YQ{with} probability $\delta_0(1-\delta)^2\big(1-(1-cr^d)^N \big)$.
\end{theorem}

\begin{theorem}\label{thm:Hdist}
Given $\delta > 0$, there exists a constant $C$ such that $d(x, \M) \leq Cr^2 $ for any arbitrary $x \in \M_{\rm out}$ \YQ{with} probability at least $\delta_0(1-\delta)^2$.
\end{theorem}

We point out that Theorem \ref{thm:Hdist} holds assuming $\sigma<1$ and $r = O(\sqrt{\sigma})$ as (\ref{ass:r_sigma}) claims. If we further assume $\sigma \to 0$, we achieve the following corollary:
\begin{corollary}\label{coro:Hdist}
 For any arbitrary $x \in \M_{\rm out}$, $d(x, \M) \to 0$ as $\sigma \to 0$ \YQ{with} probability at least $\delta_0(1-\delta)^2$.
\end{corollary}
\begin{proof}
    Given $r = O(\sqrt{\sigma})$, there exists $C_0$ such that $r = C_0 \sqrt{\sigma}$. For any $\varepsilon> 0$, let $\sigma = \frac{\varepsilon}{CC_0^2} $, and then $d(x, \M) \leq Cr^2 = CC_0^2\sigma = \varepsilon$.
\end{proof}

Generally speaking,  \YQ{the fraction} $d(y, T_x\M_{\rm out}) / \|y-z\|_2^2$ characterizes the twice differentiable quantity, which controls how flat $\M_{\rm out}$ is locally. Therefore, the lower bound of $d(y, T_x\M_{\rm out}) / \|y-z\|_2^2$ guarantees the smoothness of $\M_{\rm out}$. Recalling Proposition \ref{prop:reach}, such a quantity is related to the reach of a manifold, which \YQ{characterizes} the smoothness of a manifold.

\begin{theorem}\label{thm:reach_out}
Given $\delta > 0$, there exists a constant $C$ such that 
\[
\frac{\|z-x\|_2^2}{d(z, T_x\M_{\rm out})} \geq cr
\]
for any arbitrary $x$ and $z$ in $\M_{\rm out}$ \YQ{with} probability at least $\delta_0^2(1-\delta)^4\big(1-(1-cr^d)^N \big)$.
\end{theorem}

\YQ{The proofs of Theorem \ref{thm:manifold}, Theorem \ref{thm:Hdist} and Theorem \ref{thm:reach_out} are organized in the following way.  First we explore the properties of $P_{x_i}$ for given $x_i$ through Theorem \ref{thm:P_z} in subsection \ref{sec:Pxi}, reveal the properties of weights $\{\alpha_i\}$ through Proposition \ref{prop:alpha_bound}  and discuss the concentration phenomenon of Gaussian noise through Lemma \ref{lma:xi} in subsection \ref{sec:phix}. Based on the conclusions above, we prove in Theorem \ref{thm:bound_Pix} an upper bound on the approximation error of $\Psi_x^{\alpha}$, as a weighted sum of $\{P_{x_i}\}_{i \in I_{x,r}}$ , to $\Pi_{x^*}$. Subsequently, in subsection \ref{sec:f} and subsection \ref{sec:df}, we obtain the upper bounds on $\|f(x)\|_2$ and the first and second derivative of $f(x)$ through Theorem \ref{thm:bound_fM}, Theorem \ref{thm:first_der_f} and Theorem \ref{thm:second_der_f}. Finally, the main conclusions, namely Theorem \ref{thm:manifold}, Theorem \ref{thm:Hdist} and Theorem \ref{thm:reach_out}, are proved in Section \ref{sec:bounds_M}, using the upper bounds regarding $f(\cdot)$ defined by (\ref{fun:out_ours}). The dependency of the above theorems, lemmas, and propositions is demonstrated in Figure \ref{fig:dependency}.}


\begin{figure}[bh]
\centering
\includegraphics[width=\linewidth]{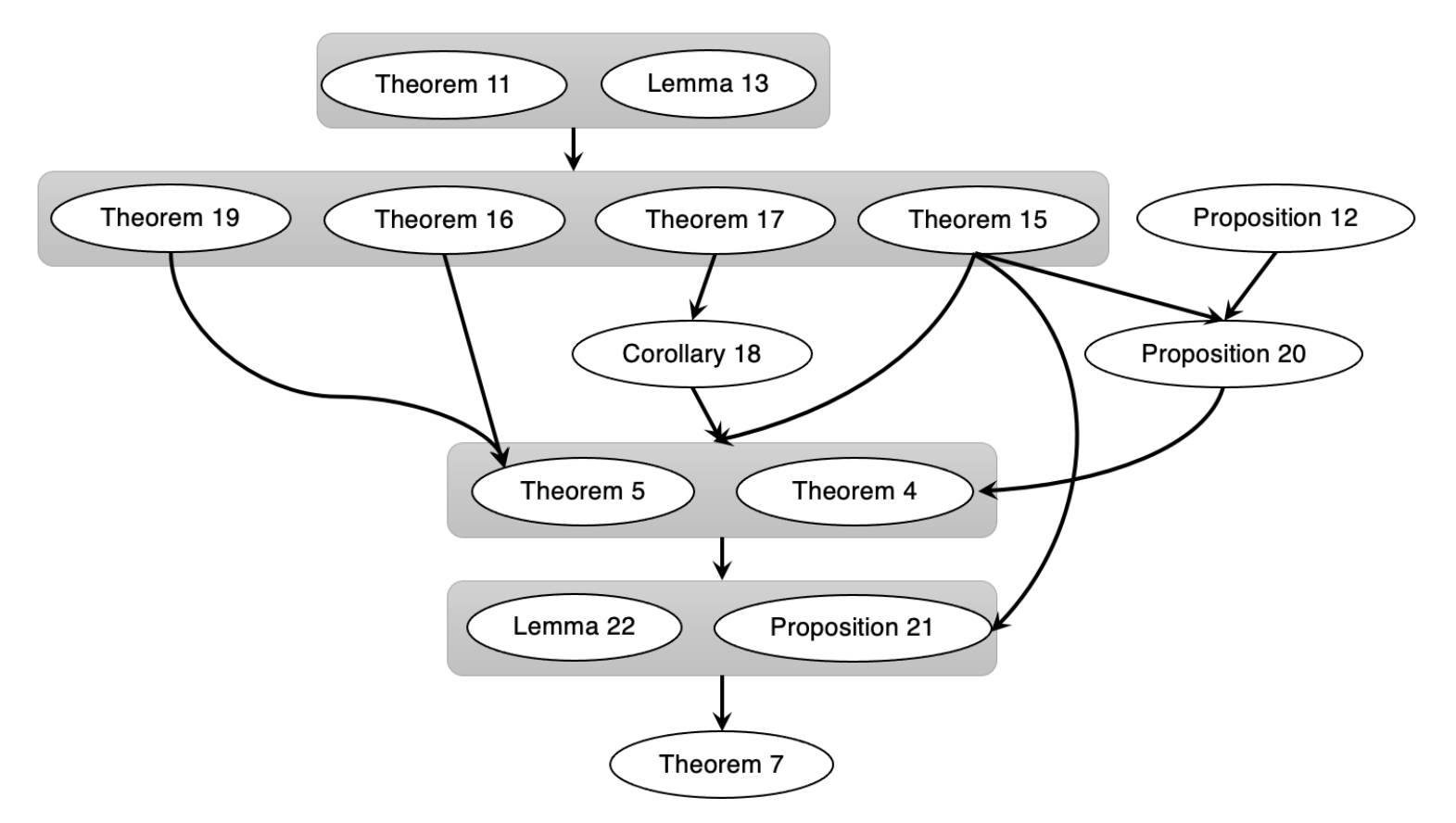}
\caption{The dependency of the core theorems, lemmas, and propositions.}\label{fig:dependency}
\end{figure}

\section{Bounds regarding the function $f$} \label{sec:bounds_f}
\YQ{To analyze $\Psi_x^\alpha$, we first explore the properties of $P_{x_i}$, where $x_i$ is any arbitrary sample point in $B_D(x,r)$. Next, properties of $\YQ{\Psi_x^\alpha}$ can be analyzed as the weighted average of $\{P_{x_i} \}_{i \in I_{x,r}}$. }
Finally, we successively bound $\|f(x)\|_2$, the first derivative of $f(x)$ and the second derivative of $f(x)$ above using bounds regarding $\YQ{\Psi_x^\alpha}$.

\subsection{Properties of $P_{x_{i}}$} \label{sec:Pxi}
To make the notations clearer, we replace $x_{i}$ with $z$ in this section. Recalling the notations in Section \ref{sec:notation}, $z^*$ is the closest point on $\M$ to $z$ and $\YQ{\Pi_{z^*}}$ is the orthogonal projection onto the normal space of $\M$ at $z^*$. The aim of this section is to bound the error $\|P_z - \YQ{\Pi_{z^*}}\|_F$.

Figure \ref{fig:Pi} illustrates the variables used for the discussion of $P_z$ and the related proof. The $z$ (black dot) is an observed noisy point of the manifold $\M$ and the blue ball is $B_D(z,r')$, centered at $z$ with radius $r'$. The subsequent proof requires $B_D(z,r') \cap \M \neq \emptyset$, which is equal to $d(z,\M) \leq r'$. Given $d(x, \M)\leq cr$ and $z \in B_D(x,r)$, we have
\begin{align*}
d(z, \M) \leq \|z - x^*\|_2 \leq \|z - x\|_2 + \|x-x^*\|_2 = \|z - x\|_2 + d(x, \M) \leq (c+1)r.
\end{align*}
Therefore, for any $r' \geq 2r$, $d(z,\M) < r'$. In this paper, we set $r' = 2r$ for convenience.
The $z_i$ (red dot) is a noisy sample located in $B_D(z,r')$,
satisfying $z_i = y_i + \xi_i$, $z_i^* = \Pi_{\M} z_i$, and $p_i$ is the projection of $z_i$ onto $T_{z^*} \M$.
\begin{figure}[ht]
    \centering
    \includegraphics[width = 0.7\textwidth]{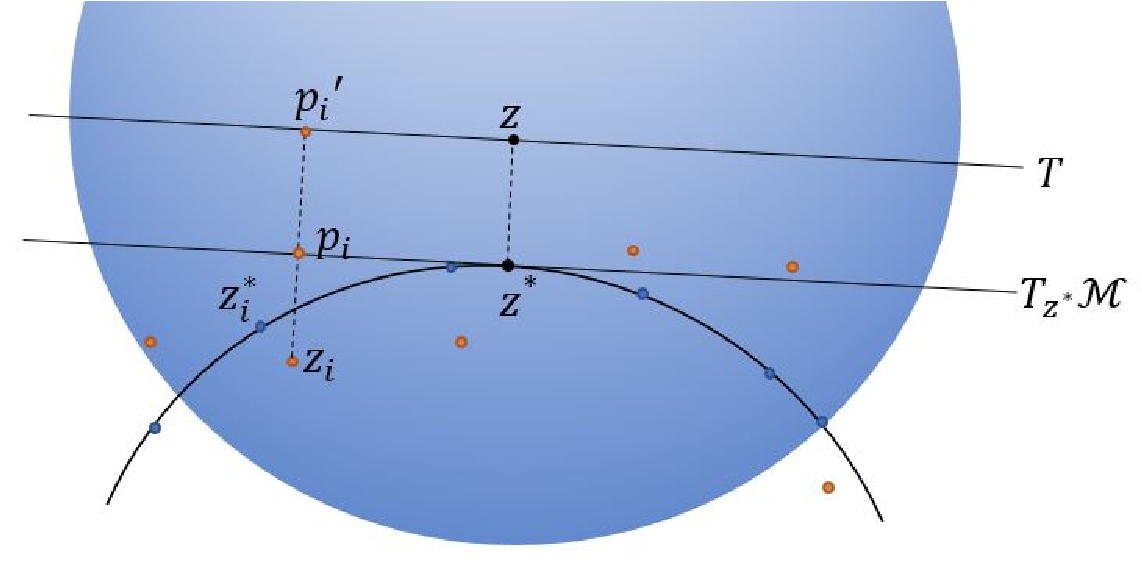}
    \caption{Diagram of variables used for the discussion of $P_z$.}
    \label{fig:Pi}
\end{figure}
The space $T$ is the translation of $T_{z^*}\M$ passing $z$, and $p_i'$ is the projection of $z_i$ onto $T$.

Consider the symmetric matrix 
$\Lambda = \frac{1}{|I_{z,r'}|}\sum_{i \in I_{z,r'}}(p_i-z^*)(p_i-z^*)^T$. Since both $p_i$ and $z^*$ are located in $T_{z^*}\M$, the spanning space of $\Lambda$ is contained in $T_{z^*}\M$. Thus, ${\rm rank}(\Lambda) \leq {\rm dim}(T_{z^*}\M) = d$ and thereby the $(d+1)$-th largest eigenvalue of $\Lambda$ is $0$. Setting columns of $U$ be the eigenvectors of $\Lambda$ corresponding to the $d$ largest eigenvalues, $U$ is also a basis of $T_{z^*}\M$. Let $U_{\bot}$ be the orthogonal complement of $U$, and we have $\YQ{\Pi_{z^*}} = U_{\bot}U_{\bot}^T$. Recalling $P_z = V_\bot V_\bot^T$, where $V_\bot$ is the orthogonal component of $V$ and columns of $V$ are the eigenvectors corresponding to the $d$ largest eigenvalues of \[
\hat{\Lambda} = \frac{1}{|I_{z,r'}|} \sum_{i \in I_{z,r'}} (z_i - z)(z_i-z)^T,
\]
we obtain
\begin{align}\label{eq:matrix_perturbation}
\|P_z - \YQ{\Pi_{z^*}}\|_F = \|V_{\bot}V_{\bot}^T-U_{\bot}U_{\bot}^T\|_F = \|VV^T- \YQ{UU^T}\|_F \leq \frac{2\sqrt{2}\|\Lambda-\hat{\Lambda}\|_F}{\lambda_d}
\end{align}
by the following Lemma:
\begin{lemma} \label{lma:Matrix}
Let $\Lambda$, $\hat \Lambda \in \R^{n \times n}$ be symmetric, with eigenvalues $\lambda_1 \geq \cdots \geq \lambda_n$ and $\hat \lambda_1 \geq \cdots \YQ{\geq} \hat \lambda_n$ respectively. Let $1 \leq d \leq n$ and assume $\lambda_d > 0$, $\lambda_{d+1} = 0$. Let $U$, $\hat U \in \R^{n \times d}$ be eigenvectors corresponding to the first $d$ eigenvalues of $\Lambda$ and $\hat \Lambda$, respectively. Then
\[
\|UU^T- \hat{U} \hat{U}^T\|_F = \sqrt{2} \|\sin \theta (\hat{U}, U)\|_F
 \leq \frac{2\sqrt{2}\|\hat \Lambda- \Lambda\|_F}{\lambda_d}
\]
by the Davis-Kahan $\sin\theta$ theorem, where $\theta(\hat{U},U)$ is the $n \times n$ diagonal matrix, whose diagonal comprises the principal angles between the column spaces of $\hat{U}$ and $U$, and $\sin \theta (\hat{U}, U)$ is defined entrywise.
\end{lemma}

We require the upper bound on $\|\hat{\Lambda} - \Lambda\|_F$ and the lower bound on the $d$-th eigenvalue of $\Lambda$, deriving them both in Lemma \ref{lma:covariance} and Lemma \ref{lma:lambda} as follows:

\begin{lemma}\label{lma:covariance}
Suppose $r' = O(\sqrt{\sigma})$ and $d(z,\M) \leq cr'$ with some $c<1$. 
There exists $C$ such that $\| \frac{1}{|I_{z,r'}|} \big( \sum_{i \in I_{z,r'}} (z_i - z)(z_i-z)^T - \sum_{i \in I_{z,r'}} (p_i - z^*)(p_i-z^*)^T \big)\|_F$ is bounded above by
\begin{align*}
\frac{C}{|I_{z,r'}|} \sum_{i \in I_{z,r'}}  \bigg( \|\xi_i\|_2^4 + \|\xi_i\|_2^3 +  \|\xi_i\|_2^2 + r'\|\xi_i\|_2  \bigg)+ C\bigg({r'}^3 + r'\|z-z^*\|_2+\|z-z^*\|_2^2\bigg).
\end{align*}
\end{lemma}

\begin{lemma}\label{lma:lambda}
The $d$-th eigenvalue of $\frac{1}{|I_{z,r'}|} \sum_{i \in I_{z,r'}} (p_i-z^*)(p_i-z^*)^T$ is bounded below by
$\lambda_d \geq c {r'}^2$,
\YQ{with} probability $\delta_0$.
\end{lemma}
Proofs of Lemma \ref{lma:covariance} and \ref{lma:lambda} appear in Appendix \ref{app:B}. Plugging the upper bound of Lemma \ref{lma:covariance} and the lower bound of Lemma \ref{lma:lambda} into (\ref{eq:matrix_perturbation}), we can obtain the following theorem:

\begin{theorem}\label{thm:P_z}
Suppose $r' = O(\sqrt{\sigma})$ and $d(z,\M) \leq r'$. 
For any given $\delta$ ,there exists $C$ such that the difference between $P_z$ and $\YQ{\Pi_{z^*}}$ is bounded by
\begin{align*}
    \|P_z - \YQ{\Pi_{z^*}}\|_F 
  &\leq \frac{C}{{r'}^2} \frac{1}{|I_{z,r'}|} \sum_{i \in I_{z,r'}}  \bigg( \|\xi_i\|_2^4 + \|\xi_i\|_2^3 +  \|\xi_i\|_2^2 + r'\|\xi_i\|_2  \bigg) \\
  &+C\bigg(r'+\frac{\|z-z^*\|_2}{r'} + \frac{\|z-z^*\|_2^2}{{r'}^2}\bigg),
  \quad {\rm in \ probability \ } \delta_0.
\end{align*}
\end{theorem}

The term $\|z-z^*\|_2^2$ in Theorem \ref{thm:P_z} tells us that $P_z$ cannot approximate $\YQ{\Pi_{z^*}}$ well if $z$ is distant from $\M$. \YQ{When the sample size $N$ is sufficiently large and the sample points are blurred by Gaussian noise, there will always be several sample points that deviate far away from the latent manifold. If $x_i$ represents such a sample point, given Theorem \ref{thm:P_z}, $P_{x_i}$ cannot effectively capture the local structure of the latent manifold. However, the next section will explain how the error caused by $P_{x_i}$ can be eliminated}  when we calculate a weighted average over $\{P_{x_i}\}_{i \in I_{x,r}}$, denoted as $\YQ{\Psi_x^\alpha}$.

\subsection{Properties of $\YQ{\Psi_x^\alpha}$}\label{sec:phix}
This section evaluates how $\YQ{\Psi_x^\alpha}$ approximates $\YQ{\Pi_{x^*}}$ using the upper bound of $\|P_{x_i} - \YQ{\Pi_{x_i^*}}\|_F$ as derived in \YQ{Theorem} \ref{thm:P_z}. As the weighted average of $\{P_{x_i}\}_{i \in I_{x,r}}$, $\YQ{\Psi_x^\alpha}$ benefits from the mutual offset of the Gaussian noise. To mathematically clarify this phenomenon, Lemma \ref{lma:xi} below bounds the weighted averages regarding $\|\xi_i\|_2$ above  \YQ{ by the Berry-Esseen Theorem and the properties of the weights $\{\alpha_i(x)\}$ stated in Proposition \ref{prop:alpha_bound}}. 

\begin{prop}\label{prop:alpha_bound}
For a point $x$ satisfying $d(x,\M) \leq cr$, there exist constants $c_0$ and $c_0'$ such that
\begin{itemize}
    \item[(i)] $\tilde{\alpha}(x)$ is bounded below by $c_0|I_{x,r}|$, \YQ{with} probability $1-C_0/\sqrt{|I_{x,r}|}$
    \item[(ii)] $\tilde{\alpha}(x)$ is bounded below by a constant $c_0'$ \YQ{with} probability $1 - (1 - cr^d)^N = O(Nr^d)$.
\end{itemize}
\end{prop}

\begin{lemma}\label{lma:xi}
Suppose $d(x,\M) \leq cr$ with some constant $c<1$ and $r = O(\sqrt{\sigma})$. For any given $\delta$, there exist constants $C$, $c_0$ and $n_0$ such that if $N \geq n_0 r^{-d}$, then $\tilde{\alpha}(x) \geq c_0|I_{x,r}| $ \YQ{with} probability at least $(1-\delta)$ and
\begin{align} \label{prop:clt}
    \sum_{i \in I_{x,r}} \alpha_i(x) \|\xi_i\|_2^k \leq C \sigma^k 
    \quad {\rm and} \quad
     \frac{1}{|I_{x,r}|^2}\sum_{i,j \in I_{x,r}}  \|\xi_i\|_2^s\|\xi_i\|_2^t \leq C \sigma^{s+t}
\end{align}
hold for $k, s, t \leq 4$ \YQ{with} probability at least $(1-\delta)^2$.
\end{lemma}
\begin{lemma}\label{lma:Pi_ij}
Suppose $x$ and $y$ are two points on $\M$ 
, then
\begin{align}
   \| \Pi_x^* - \Pi_y^*\|_2 \leq  \| \Pi_x^* - \Pi_y^*\|_F \leq C\frac{\|x-y\|_2}{\tau}.
\end{align}
\end{lemma}

\YQ{The proof of Lemma \ref{lma:xi} is shown in Appendix \ref{proof:xi_e_var}. To evaluate how $\YQ{\Psi_x^\alpha}$ approximates $\Pi_{x^*}$, we first evaluate how the tangent space changes when the point of tangency changes as Lemma \ref{lma:Pi_ij}, which is also proved in Appendix \ref{proof:xi_e_var}.} Based on the theorem above and lemmas, we obtain the following theorem to evaluate $\YQ{\Psi_x^\alpha}$:

\begin{theorem}\label{thm:bound_Pix}
Suppose $d(x,\M) \leq cr$ with some constant $c<1$ and $r = O(\sqrt{\sigma})$. For any given $\delta$, there exist constants $C$ and $n_0$ such that if $N \geq n_0r^{-d}$, then 
\begin{align}\label{equ:bound_Pix}
 \| \YQ{\Psi_x^\alpha} - \YQ{\Pi_{x^*}}\|_2 \leq \| \YQ{\Psi_x^\alpha} - \YQ{\Pi_{x^*}}\|_F \leq Cr 
\end{align}
holds \YQ{with} probability $\delta_0(1-\delta)^2$.
\end{theorem} 

\begin{proof}
By definition of $A_x$,
\begin{align} 
\|A_x - \YQ{\Pi_{x^*}} \|_F
& = \bigg\| \sum_{i \in I_{x,r}} \alpha_{i}(x) (P_{x_i}-\YQ{\Pi_{x_i^*}}) + \sum_i \alpha_i(x) (\YQ{\Pi_{x_i^*}}-\YQ{\Pi_{x^*}}) \bigg\|_F \nonumber \\ 
& \leq \sum_{i \in I_{x,r}} \alpha_{i}(x)\|P_{x_i}-\YQ{\Pi_{x_i^*}}\|_F + \sum_{i \in I_{x,r}} \alpha_{i}(x)\| \YQ{\Pi_{x_i^*}}-\YQ{\Pi_{x^*}} \|_F
\label{equ:Ax}
\end{align}
Setting $z$ in Theorem \ref{thm:P_z} to be $x_i$ and replacing $r'$ by $r'=2r$, we obtain the upper bound of $\|P_{x_i}-\YQ{\Pi_{x_i^*}}\|_F$ \YQ{with probability $(1-\delta)^2$}. Plugging the upper bound into the first term on the right-hand side of (\ref{equ:Ax}), we obtain
\begin{align*}
\sum_{i \in I_{x,r}} \alpha_{i}(x)\|P_{x_i}-\YQ{\Pi_{x_i^*}}\|_F 
& \leq \frac{C}{r^2} \sum_{i \in I_{x,r}} \sum_{j \in I_{x_i,2r}}  \frac{\alpha_{i}(x)}{|I_{x_i,2r}|} \bigg( \|\xi_j\|_2^4 + \|\xi_j\|_2^3 + \|\xi_j\|_2^2 + r\|\xi_j\|_2 \bigg) \\
& +  C  \bigg(r + 
	\frac{\sum_{i \in I_{x,r}} \alpha_{i}(x) \|x_i - x_i^* \|_2}{r} + \frac{\sum_{i \in I_{x,r}} \alpha_{i}(x) \|x_i-x_i^*\|_2^2}{r^2} \bigg) \\
& \leq \frac{C}{r^2} \sum_{i \in I_{x,r}} \sum_{j \in I_{x_i,2r}}  \frac{\alpha_{i}(x)}{|I_{x_i,2r}|} \bigg( \|\xi_j\|_2^4 + \|\xi_j\|_2^3 + \|\xi_j\|_2^2 + r\|\xi_j\|_2 \bigg) \\
& +  C  \bigg(r + 
	\frac{\sum_{i \in I_{x,r}} \alpha_{i}(x) \|\xi_i  \|_2}{r} + \frac{\sum_{i \in I_{x,r}} \alpha_{i}(x) \|\xi_i\|_2^2}{r^2} \bigg) 
\end{align*}
Plugging the upper bound of $\sum_{i \in I_{x,r}} \alpha_i(x)\|\xi_i\|_2^k$ into the last formula leads to
\begin{align*}
\sum_{i \in I_{x,r}} \alpha_{i}(x)\|P_{x_i}-\YQ{\Pi_{x_i^*}}\|_F \leq C\Big(\frac{\sigma}{r}+\frac{\sigma^2}{r^2} + r \Big) \leq C r,
\end{align*}
\YQ{with probability $\delta_0(1-\delta)^2$, where} the last inequality holds given $r = O(\sqrt{\sigma})$.
As for the second term \YQ{on the right-hand side of (\ref{equ:Ax})},
\begin{align*}
	\sum_{i \in I_{x,r}} \alpha_{i}(x)\|\YQ{ \Pi_{x_i^*}-\Pi_{x^*}} \|_F
	& \leq \frac{C}{\tau}\sum_{i \in I_{x,r}} \alpha_{i}(x)\| x_i^*-x^*\|_2 \\
	& \leq \frac{C}{\tau}\sum_{i \in I_{x,r}} \alpha_{i}(x) \Big(\| x_i^*-x_i\|_2 + \| x_i-x\|_2 + \|x-x^*\|_2 \Big) \\
	& \leq C\frac{r}{\tau} + \frac{C}{\tau}\sum_{i \in I_{x,r}} \alpha_{i}(x) \| x_i^*-x_i\|_2 \\
	& \leq C\frac{r}{\tau} + C\frac{\sigma}{\tau} \leq C \frac{r}{\tau}.
\end{align*}
where the first inequality is by Lemma \ref{lma:Pi_ij}, \YQ{the second-to-last inequality holds given $r=O(\sqrt{\sigma})$, and the last inequality holds by (\ref{ineq:bound_r})}. Since $\YQ{\Psi_x^\alpha}$ is the closest $(D-d)$-rank projection matrix to $A_x$, we have
\begin{align} \label{bound:AxPix}
 \|\YQ{\Psi_x^\alpha} - A_x\|_{F} \leq \|A_x - \YQ{\Pi_{x^*}}\|_{F}  \leq Cr, \quad {\rm with \ probability \ } \delta_0(1-\delta)^2.
\end{align}
Hence,
$\| \YQ{\Psi_x^\alpha} - \YQ{\Pi_{x^*}}\|_F \leq \|\YQ{\Psi_x^\alpha} - A_x\|_F +  \|A_x - \YQ{\Pi_{x^*}} \|_F \leq Cr$ \YQ{with probability $ \delta_0(1-\delta)^2$}.
\end{proof}

\subsection{A bound on $f(x)$}\label{sec:f}
This section examines how $f(x)$ approximates the bias from $x$ to $\M$, which is achieved by calculating $\|f(x)\|_2$ for $x \in \M$. If $f$ approximates the bias well, such $\|f(x)\|_2$ should be bounded above by a small value with $x \in \M$. 
\begin{theorem} \label{thm:bound_fM}
Suppose $x \in \M$ and $r = O(\sqrt{\sigma})$. For any given $\delta$, there exist constants $C$ and $n_0$ such that if $N \geq n_0r^{-d}$, then $\|f(x)\|_2 \leq Cr^2$ \YQ{with} probability $\delta_0(1-\delta)^2$.
\end{theorem}
\begin{proof}
It is clear that $x = x^*$ when $x \in \M$. Accordingly, we use $x$ instead of $x^*$ in the following discussion for convenience. First, we bound the distance between $\sum_{i \in I_{x,r}} \alpha_i(x) x_i$ and $T_{x}\M$. By definition,
\begin{align*}
	d\Big(\sum_{i \in I_{x,r}} \alpha_i(x) x_i, T_{x}\M \Big)
	& = \bigg\| \Pi_{x}^* \Big(\sum_{i \in I_{x,r}} \alpha_i(x) x_i - x \Big) \bigg\|_2 \\
	& \leq \sum_{i \in I_{x,r}} \alpha_i(x) \| \Pi_x^*(x_i-x)\|_2 \\
	& \leq \sum_{i \in I_{x,r}} \alpha_i(x) \|x_i-x_i^*\|_2 + \sum_{i \in I_{x,r}} \alpha_i(x) \|\Pi_x^*(x_i^*-x)\|_2 \\
	& \leq \sum_{i \in I_{x,r}} \alpha_i(x) \|\xi_i\|_2 + \sum_{i \in I_{x,r}} \alpha_i(x) \frac{\|x_i^*-x\|_2^2}{\tau} \\
	& \leq C_1 \sigma + C_2 \sum_{i \in I_{x,r}} \alpha_i(x) \frac{(\|x_i^*-x_i\|_2+ \|x_i-x\|_2)^2}{\tau}  \\
	& \leq C_1 \sigma + C_2 \frac{(\sigma+r)^2}{\tau},
\end{align*}
\YQ{where the second-to-last inequality holds by Lemma \ref{lma:xi} with probability $(1-\delta)^2$.}
The parameter $r$ is selected in the order of $\sqrt{\sigma}$, namely$C_3$ exists such that $r = C_3\sqrt{\sigma} > C_3 \sigma$ since $\sigma < 1$. So $(\sigma+r)^2 < (\frac{1}{C_3}+1)^2r^2$ and
\[
C_1 \sigma + \frac{C_2}{\tau} (\sigma+r)^2 \leq \frac{C_1r^2}{C_3^2}+\frac{C_2}{\tau}(\frac{1}{C_3}+1)^2r^2 = Cr^2.
\]
Hence, we obtain $ d(\sum_{i \in I_{x,r}} \alpha_i(x) x_i, T_{x}\M) \leq Cr^2$.

We let $a = \sum_{i \in I_{x,r}} \alpha_i(x)x_i$ and $b$ be the projection of $a$ onto $T_{x}\M$.  Then, we have
\begin{align*}
	\| a- b\|_2 = \|\Pi_{x}^*(a-b)\|_2 =  d \bigg(\sum_{i \in I_{x,r}} \alpha_i(x) x_i, T_{x}\M \bigg) \leq Cr^2.
\end{align*}
According to the definition of $f(x)$,
\begin{align*}
f(x) = \YQ{\Psi_x^\alpha}(x - a) = \Pi_{x}^*(x-b) + (\YQ{\Psi_x^\alpha}-\Pi_{x}^*)(x-b) + \YQ{\Psi_x^\alpha}(b-a),
\end{align*}
where $\Pi_{x}^*(x-b) = \bf{0}$, since $x = x^* \in T_{x}\M$ and $b \in T_{x} \M$.  Hence, we obtain
\begin{align*}
\|f(x)\|_2 & \leq \|\YQ{\Psi_x^\alpha}-\Pi_{x}^*\|_F\big(\|x-a\|_2 + \|a-b\|_2\big)+ \|\YQ{\Psi_x^\alpha}(a-b)\|_2 \\
& \leq \|\YQ{\Psi_x^\alpha}-\Pi_{x}^*\|_F\big(\|x-a\|_2 + \|a-b\|_2\big)+ \|a-b\|_2 \\
& \leq C_1 r \times (r+ r^2) + C_2 r^2 \leq Cr^2,
\end{align*}
\YQ{where the second-to-last inequality holds by Theorem \ref{thm:P_z} with probability $\delta_0$ and the last inequality holds by (\ref{ineq:bound_r}). In summary, $\|f(x)\|_2 \leq Cr^2$ with probability $\delta_0(1-\delta)^2$.}
\end{proof}

\subsection{A bound on the first and second derivative of $f(x)$}\label{sec:df}
We now proceed to obtain an upper bound on $\|\partial_v f(x)\|_2$ with $\|v\|_2 = 1$, where
\[
    \partial_v f(x) = \lim_{t \to 0} \frac{f(x+tv)-f(x)}{t},
\]
for any $v \in \R^D$. 

\begin{theorem} \label{thm:first_der_f}
Suppose $d(x,\M) \leq cr$ and $r = O(\sqrt{\sigma})$. For any given $\delta$, there exist constants $C$ and $n_0$ such that if $N \geq n_0 r^{-d}$, \begin{align}\label{bound:first_der}
\| \partial_v f(x) - \YQ{\Psi_x^\alpha}  v\|_2 \leq C r,
\end{align}
\YQ{with} probability $\delta_0(1-\delta)^2$. 
\end{theorem}

\YQ{The proof of Theorem \ref{thm:first_der_f} refers to Appendix \ref{app:first_der_f}.
This theorem claims the first derivative of $f(x)$ approximates $\YQ{\Psi_x^\alpha} v$ in the order of $O(r)$. Taking $v$ in Theorem \ref{thm:first_der_f} as $e_1, \cdots, e_D$, we achieve the following Corollary \ref{coro:Jf_Phix}.}

\begin{corollary}\label{coro:Jf_Phix}
\YQ{
Suppose $d(x,\M) \leq cr$ and $r = O(\sqrt{\sigma})$. For any given $\delta$, there exist constants $C$ and $n_0$ such that if $N \geq n_0 r^{-d}$, 
\begin{align}\label{equ:bound_J_Pix}
\|J_f(x)-\Psi_x^\alpha\|_F \leq Cr
\end{align}
\YQ{with} probability $\delta_0(1-\delta)^2$. 
}
\end{corollary}

\begin{proof}
\YQ{
Let $e_i$ represent a $D$-dimensional vector where the $i$-th component is $1$, and the other components are $0$. The Jacobian matrix of function $f$ can be represented as 
\[
J_f(x) = \big( \partial_{e_1} f(x), \cdots, \partial_{e_D} f(x) \big)
\] 
and $\Psi_x^\alpha = \big( \Psi_x^\alpha e_1, \cdots, \Psi_x^\alpha e_D \big)$. Hence,
\begin{align*}
\| J_f(x) - \Psi_x^\alpha\|_F
 & = \| \big( \partial_{e_1} f(x), \cdots, \partial_{e_D} f(x) \big) - \big( \Psi_x^\alpha e_1, \cdots, \Psi_x^\alpha e_D \big)\|_F \\
& = \sqrt{\sum_{i=1}^D \| \partial_{e_i} f(x) - \Psi_x^\alpha e_i \|_2^2} \leq \sqrt{\sum_{i=1}^D (C_1^2r^2)} = C_1\sqrt{D}r = Cr.
\end{align*}
The last inequality holds by Theorem \ref{thm:first_der_f}, which concludes $\|J_f(x)-\Psi_x^\alpha\|_F \leq Cr$ with probability $\delta_0(1-\delta)^2$.
}
\end{proof}

We now proceed to obtain an upper bound on $\|\partial_v\big(\partial_u f(x)\big)\|_2$ with $\|v\|_2 = \|u\|_2=1$ in Theorem \ref{thm:second_der_f}. This theorem proves that the second derivative of $f(x)$ is bounded above by a certain constant, which indicates the smoothness of $f(x)$.
The proof of Theorem \ref{thm:second_der_f} refers to Appendix \ref{app:second_der_f}.

\begin{theorem} \label{thm:second_der_f}
Suppose $d(x,\M) \leq cr$ with some constant $c<1$ and $r = O(\sqrt{\sigma})$. For any given $\delta$, there exist constants $C$ and $n_0$ such that if $N \geq n_0r^{-d}$, 
then $\|\partial_v \partial_u f(x)\|_2 \leq C$ \YQ{with} probability $\delta_0(1-\delta)^2$.
\end{theorem}

\section{Proofs of Theorem \ref{thm:manifold}, Theorem \ref{thm:Hdist} and Theorem \ref{thm:reach_out}} \label{sec:bounds_M}
\YQ{Theorem \ref{thm:manifold} claims that the intersection of $\M_{\rm out}$ and the neighborhood of $x \in \M_{\rm out}$ is a $d$-dimensional manifold. To prove this conclusion, we first need to discuss the properties of the neighborhood of $x$, as stated in Proposition \ref{prop:delta_Pixz}. below}

\begin{prop} \label{prop:delta_Pixz}
Let $\epsilon = \min \{ \sqrt{ \frac{\alpha(x)}{|I_{x,2r}|^2} \frac{r^3}{\beta} }, r\}$ for given $x$, then 
\begin{align*}
\| \YQ{\Psi_x^\alpha} - \YQ{\Psi_z^\alpha} \|_2 \leq Cr, \quad \forall z \in B_D(x,\epsilon)
\end{align*}
\YQ{with} probability $\delta_0(1-\delta)^2\big(1-(1-cr^d)^N\big)$.
\end{prop}

\YQ{The proof of Proposition \ref{prop:delta_Pixz} can be found in Appendix \ref{app:D}. Based on Proposition \ref{prop:delta_Pixz}, we construct an auxiliary function $h$ to further characterize the neighborhood of $x$ and obtain the proof of Theorem \ref{thm:manifold}, as shown below:}


\begin{proof}{\bf of Theorem \ref{thm:manifold}}
\YQ{
 Let $h(z): B_D(x, \epsilon) \subset \R^{D} \to \R^{D-d}$, per
\begin{align}\label{def:hz}
	h(z) = V_x^T f(z),
\end{align}  
where  $V_x$ is the factor of $\YQ{\Psi_x^\alpha}$ such that $\YQ{\Psi_x^\alpha} = V_xV_x^T$. Then $h(z) = \bf{0}$ if $f(z)=\bf{0}$.  Assuming there exists $z$ such that $h(z) = \bf{0}$ but $f(z)\neq \bf{0}$, we obtain
\begin{align*}
\|\YQ{\Psi_x^\alpha} - \YQ{\Psi_z^\alpha}\|_2 & = \max_{v \neq 0} \frac{\big\| (\YQ{\Psi_x^\alpha}-\YQ{\Psi_z^\alpha}) v \big\|_2}{\|v\|_2} \geq  \frac{\big\| (\YQ{\Psi_x^\alpha}-\YQ{\Psi_z^\alpha}) f(z) \big\|_2}{\|f(z)\|_2} \\
& =  \frac{\big\| \YQ{\Psi_x^\alpha}f(z)-\YQ{\Psi_z^\alpha} f(z) \big\|_2}{\|f(z)\|_2}
 = \frac{\big\|V_xh(z)- f(z) \big\|_2}{\|f(z)\|_2} = \frac{\big\|0- f(z) \big\|_2}{\|f(z)\|_2}= 1
\end{align*}
However, $\|\Psi_x^\alpha -\YQ{\Psi_z^\alpha}\|_2 \leq Cr$ \YQ{with} probability $\delta_0(1-\delta)^2\big(1-(1-cr^d)^N \big)$ via Proposition \ref{prop:delta_Pixz}, which is contradictory to $\|\YQ{\Psi_x^\alpha}-\YQ{\Psi_z^\alpha}\|_2\geq 1$. Hence, $f(z) = \bf{0}$ if and only if $h(z) = \bf{0}$, equivalently $h^{-1}(\bf{0}) = f^{-1}(\bf{0})$ in $B_D(x,\epsilon)$, with probability $\delta_0(1-\delta)^2\big(1-(1-cr^d)^N\big)$.}

\YQ{
For $z \in B_D(x, \epsilon)$, 
\begin{align*}
    J_h(z) & = V_x^T J_f(z) \\
        & = V_x^T(J_f(z)-J_f(x)) + V_x^T(J_f(x)-\YQ{\Psi_x^\alpha}) + V_x^T\YQ{\Psi_x^\alpha}.
\end{align*}
On the right hand side of the above equality, we have $\|J_f(x)-\YQ{\Psi_x^\alpha}\|_F\leq Cr$ by Corollary \ref{coro:Jf_Phix}, $V_x^T\YQ{\Psi_x^\alpha} = V_x^T$ and
$\|J_f(z) - J_f(x)\|_F \leq C \max_{i=1}^D \|J_{e_i}f(z)-J_{e_i}f(x)\|_2 \leq C\|x-z\|_2 \leq C\epsilon \leq Cr$,
where the second inequality holds by Theorem \ref{thm:second_der_f}, which implies
\[
	\|J_h(z)-V_x^T\|_2 \leq \|J_h(z)-V_x^T\|_F = \| V_x^T(J_f(z)-J_f(x)) + V_x^T(J_f(x)-\YQ{\Psi_x^\alpha}) \|_F \leq Cr.
\]
}
Hence, the maximal difference between the singular values of $J_h(z)$ and $V_x^T$ is bounded by $Cr$. Let $\sigma_1 \geq \cdots \geq \sigma_{D-d}$ be the singular values of $J_h(z)$. We obtain $|\sigma_{D-d} - 1| \leq Cr$ since the singular values of $V_x^T$ are $1$, which implies $\sigma_{D-d} \geq 1-Cr$ and ${\rm rank}\big( J_h(z) \big) = D-d$ for any $z \in B_D(x, \epsilon)$. This means the rank of $h$ at $z$ equals $D-d$ for any $z \in B_D(x, \epsilon)$, and thus $h^{-1}(\bf{0})$ is a $d$-dimensional submanifold of $B_D(x, \epsilon) \subset \R^D$. The equivalence between $h^{-1}(\bf{0})$ and $f^{-1}(\bf{0})$ in $B_D(x,\epsilon)$ guarantees that $f^{-1}(\bf{0})$ is also a $d$-dimensional submanifold of $B_D(x, \epsilon) \subset \R^D$. 

\YQ{The above proof is based on Proposition \ref{prop:delta_Pixz}, Corollary \ref{coro:Jf_Phix} and Theorem  \ref{thm:second_der_f}, where Proposition \ref{prop:delta_Pixz} is proved based on Theorem \ref{thm:bound_Pix} and Proposition \ref{prop:alpha_bound}(ii), and Corollary \ref{coro:Jf_Phix} is proved based on Theorem \ref{thm:first_der_f}. Noting Theorem \ref{thm:bound_Pix}, Theorem \ref{thm:first_der_f} and Theorem  \ref{thm:second_der_f} are valid when Lemma \ref{lma:xi} and Theorem \ref{thm:P_z} hold, we obtain
\begin{align*}
& \mathbb{P}\big( \M_{\rm out}\cap B_D(x,\epsilon) {\rm \  is \ a \ } d{\rm -dimensional \ manifold} \big) \\
& \geq   \mathbb{P}\big({\rm Proposition \ \ref{prop:alpha_bound}(ii),  \  Lemma \ \ref{lma:xi}  \ and \ Theorem \ \ref{thm:P_z} \  hold \ for \ } x \big) \\
& \geq    \delta_0(1-\delta)^2\big(1-(1-cr^d)^N \big).
\end{align*}
}
\end{proof}

\begin{proof}{\bf of Theorem \ref{thm:Hdist}} 
For any fixed $x \in \M_{\rm out}$, we let $V_x \in \R^{D\times(D-d)}$ denote the orthonormal matrix such that $\YQ{\Psi_x^\alpha} = V_xV_x^T$, and let $U_x$ denote the orthogonal complement of $V_x$. Then, we define
\[
    F(z) = f(z) + U_xU_x^Tz.
\]
Let $x^*$ be the projection of $x$ onto $\M$, as done previously, $\YQ{\Pi_{x^*}} = V_*V_*^T$, and $U_*$ be the orthogonal complement of $V_*$. The difference $\|F(x^*)-F(x)\|_2$ can be evaluated as
\begin{align*}
    & \|F(x^*)-F(x)\|_2  \\
 =  & \|f(x^*)+U_xU_x^Tx^*-f(x)-U_xU_x^Tx \|_2 \\
 =  & \|f(x^*)+U_xU_x^Tx^*-U_xU_x^Tx \|_2 \\
 \leq & \|f(x^*)\|_2 + \|(U_xU_x^T-U_*U_*^T)(x-x^*)\|_2 + \|U_*U_*^T(x-x^*)\|_2 \\
= & \|f(x^*)\|_2 + \|(\YQ{\Psi_x^\alpha}-\YQ{\Pi_{x^*}})(x-x^*)\|_2 + \|U_*U_*^T(x-x^*)\|_2 \\
\leq &  \|f(x^*)\|_2 + \|\YQ{\Psi_x^\alpha}-\YQ{\Pi_{x^*}}\|_F\|x-x^*\|_2 + \|U_*U_*^T(x-x^*)\|_2 \\
\leq & Cr^2
\end{align*}
The second equality holds because $f(x) = \bf{0}$ for $x \in \M_{\rm out}$ while the last inequality holds because $\|f(x^*)\|_2\leq Cr^2$ via Theorem \ref{thm:bound_fM}, $\|\YQ{\Psi_x^\alpha} - \YQ{\Pi_{x^*}}\|_F\leq Cr$ via Theorem \ref{thm:bound_Pix}, $\|x-x^*\| = d(x, \M) \leq cr$ via the definition of $\M_{\rm out}$, and $U_*U_*^T x= U_*U_*^Tx^*$, since $x^*$ is the projection of $x$ onto $T_{x^*}\M$.

The Jacobian matrix of $F$ at $z = x$, denoted by $J_F(x)$ for simplicity, is
\begin{align*}
        J_F(x) = J_f(x) + U_xU_x^T = I_D + \big(J_f(x) - \YQ{\Psi_x^\alpha}\big).
\end{align*}
\YQ{By Corollary \ref{coro:Jf_Phix}, each entry of the matrix $J_f(x) - \Psi_x^\alpha$ is bounded above by $Cr$. Hence, we obtain
\[
 J_F(x) = I_D + O(r),
\]
}
which means that $J_F(x)$ approximates $I_D$ \YQ{with precision $O(r)$} and $J_F(x)$ is invertible. Moreover, $\|J_F(x)\|_F \leq C(1+ r)$ and its inversion is $\|J_F^{-1}(x)\|_F \leq C(1+r)$.

The changing rate of $J_F$ can also be bounded as follows: supposing $x'$ and $x''$ are two arbitrary points, we have
\begin{align} \label{bound:Jfx_Jfz}
\|J_F(x') - J_F(x'')\|_F = \|J_f(x')-J_f(x'')\|_F \leq C\|x'-x''\|_2
\end{align}
by the upper bound on the second derivative of $f(x)$ in Theorem \ref{thm:second_der_f}.

Based on the conclusions that $\|F(x)-F(x^*)\|_2\leq Cr^2$, $J_F(x) = I_D+O(r)$, and $\|J_F(x')-J_F(x'')\|_F \leq C\|x'-x''\|_2$, we could bound $\|x-x^*\|_2$ via Theorem 2.9.4 (the inverse function theorem) in \cite{hubbard2001vector}. Specifically,
\[
    \|x - x^*\|_2 \leq Cr^2.
\]

\YQ{The above proof is based on Theorem \ref{thm:bound_Pix},  Theorem \ref{thm:bound_fM}, Corollary \ref{coro:Jf_Phix} and Theorem \ref{thm:second_der_f}, which are valid when Lemma \ref{lma:xi} and Theorem \ref{thm:P_z} hold. Hence, the conclusion $\|x - x^*\|_2 \leq Cr^2$ is drawn with probability $\delta_0(1-\delta)^2$.}


\end{proof}

\YQ{
To prove the smoothness of the estimated manifold $\M_{\rm out}$, as stated in Theorem \ref{thm:reach_out}, we construct the following two auxiliary functions and clarify their properties.
}

\begin{prop}\label{prop:eq_fg}
\YQ{
For any fixed point $x \in \M_{\rm out}$, set $W_x$ to be the basis of the spanning space of $J_f(x)^T$ and
\begin{align}\label{def:gz}
    g(z) = W_x^Tf(z).
\end{align}
The following two statements
\begin{itemize}
\item[(i)] $g$ is a function from $\R^D$ to $\R^{D \times (D-d)} $
\item[(ii)] Given $z \in B_D(x, r\tau)$, $g(z) = \bf{0}$ if and only if $f(z) = \bf{0}$
\end{itemize}
hold simultaneously \YQ{with} probability at least $\delta_0^2(1-\delta)^4\big(1-(1-cr^d)^N \big)$.
}
\end{prop}



\YQ{Proposition \ref{prop:eq_fg} claims that  $f^{-1}(\bf{0})$ and $g^{-1}(\bf{0})$ describe the same set in the neighborhood of $x$ whose proof can be found in Appendix \ref{app:D}.} By $W_x$, we reset the coordinate system \YQ{for $g$}. Specifically, \YQ{the rows of $J_f(x)$ are orthogonal to the contour surface at $x$, and $W_x$ is also the basis of the normal space of $\M_{\rm out}$ at $x$.} Accordingly, we set the first $d$ coordinates as the basis of $T_x \M_{\rm out}$ and the last $D-d$ coordinates as the columns of $W_x$. In this coordinate system, we define an implicit function $\phi : \R^d \to \R^{D-d}$ based on $g(\cdot)$ using the implicit function theorem, such that $\big(\zeta; \phi(\zeta) \big)$ maps $\zeta \in T_x \M_{\rm out}$ to a point on the manifold $\M_{\rm out}$.  Here, we let $(\eta; \zeta)$ denote the concatenation of column vectors $\eta$ and $\zeta$. The upper bound on the first and second derivatives of $\phi$ is given in Lemma \ref{bound:der_phi}, with its proof appearing in Appendix \ref{app:D}.

\begin{lemma} \label{bound:der_phi}
  Suppose function $g$ is defined as (\ref{def:gz}). The implicit function $\phi: \R^{d} \to \R^{D-d}$ satisfying $g\big(\cdot, \phi(\cdot)\big)  = \bf{0}$ exists, and its first and second derivatives are bounded above by
    \[
        \partial_s \phi(\zeta) \leq C\| \big(\zeta ; \phi(\zeta) \big) - x\|_2, \quad
        \partial_t \partial_s \phi(\zeta) \leq C,
    \]
 \YQ{with} probability at least $\delta_0(1-\delta)^2\big(1-(1-cr^d)^N \big)$, for any $\|s\|_2 = \|t\|_2 = 1$.
\end{lemma}

\begin{proof}{\bf of Theorem \ref{thm:reach_out}}
Let $x$ and $z$ be two points on $\M_{\rm out}$, and $T_x \M_{\rm out}$ be the tangent space to $\M_{\rm out}$ at $x$. 
The proof is conducted with $\|z-x\|_2 > r\tau$ and $\|z-x\|_2 \leq r\tau$, respectively. First, when $\|z-x\|_2 > r\tau$, 
\begin{align}\label{equ:bound_reach}
\frac{\|z-x\|_2^2}{d(z, T_x\M_{\rm out})} \geq cr\tau
\end{align}
 holds because $\|z-x\|_2 \geq d(z,T_x \M_{\rm out})$. 
 Second, when $\|z-x\|_2 \leq r\tau$, we have $g(z) = f(z) = g(x) = f(x) = \bf{0}$ by Proposition \ref{prop:eq_fg} \YQ{with} probability $\delta_0^2(1-\delta)^4\big(1-(1-cr^d)^N \big)$, since $x$ and $z$ are on $\M_{\rm out}$. Let $\zeta_x$ and $\zeta_z$ denote the first $d$ coordinates of $x$ and $z$, respectively. We have $z = \big(\zeta_z; \phi(\zeta_z) \big)$, $x = \big(\zeta_x; \phi(\zeta_x)\big)$, $\partial_s \phi(\zeta_z) \leq C\|z-x\|_2$ and $\partial_t\partial_s \phi(\zeta_z) \leq C$ \YQ{with} probability at least $\delta_0(1-\delta)^2$ by Lemma \ref{bound:der_phi}. So,
\begin{align*}
d(z, T_x \M_{\rm out})
& =  \|\phi(\zeta_z)-\phi(\zeta_x)\|_2 \\
& \leq C\|z-x\|_2\|\zeta_z-\zeta_x\|_2+ C\|\zeta_z-\zeta_x\|_2^2 \\
& \leq C\|z-x\|_2^2 + C\|z-x\|_2^2 \leq C\|z-x\|_2^2 .
\end{align*}
As a result, 
\[
\frac{\|z-x\|_2^2}{d(z,T_x\M_{\rm out})} \geq \frac{\|z-x\|_2^2}{C\|z-x\|_2^2} = c.
\]
Combined with (\ref{equ:bound_reach}), we complete this proof.

\YQ{The above proof requires Proposition \ref{prop:eq_fg} and Lemma \ref{bound:der_phi}, which are valid when Theorem \ref{thm:manifold} and Theorem \ref{thm:Hdist} hold simultaneously for $x$ and when Theorem \ref{thm:bound_Pix} holds for $z$.  Given the dependency between the theorems as shown in Figure \ref{fig:dependency}, we establish that the above theorems hold when Lemma \ref{lma:xi} and Theorem \ref{thm:P_z} simultaneously hold for $x, z$, and when Proposition \ref{prop:alpha_bound}(ii) holds for $x$. Hence, we have
\begin{align*}
& \mathbb{P}\Big(\frac{\|z-x\|_2^2}{d(z, T_x\M_{\rm out})} \geq cr \Big) \\
& \geq  \mathbb{P} \Big( \big({\rm Proposition \ \ref{prop:alpha_bound}(ii) \  holds \ for \ } x \big) \cap  \big({\rm Lemma \ \ref{lma:xi}  \ and \ Theorem \ \ref{thm:P_z} \  hold \ for \ } x, z \big) \Big) \\
& \geq   \delta_0^2(1-\delta)^4\big(1-(1-cr^d)^N\big).
\end{align*}
}
\end{proof}

\section{Experimental Results} \label{sec:experiment}
This section comprises two parts. The first part provides numerical comparisons of the methods of \cite{mohammed2017manifold}, \cite{pmlr-v75-fefferman18a}, and \cite{Aizenbud2021}. Further, we apply relevant methods on several known manifolds, illustrate the output manifolds, and calculate the Hausdorff distances between the output and latent manifolds. In the second part, we focus on real applications, and use our method to denoise facial images sampled from a lengthy video recording. The results of our method are then compared to the findings of each of the other aforementioned methods.

{\bf Implementation:} the MATLAB codes, together with all numerical examples used in this paper, are available at \url{https://zhigang-yao.github.io/research.html} which contains a GitHub link under the code tap. We have also implemented the related methods from \cite{mohammed2017manifold} and \cite{pmlr-v75-fefferman18a}, since the authors of both papers have not provided implementation due to the nature of their work having been purely abstract.

\subsection{Simulation} \label{sec:simulations}

\begin{algorithm}[t]
\caption{Project $x$ onto $\M_{\rm out}$}\label{alg:GD}
Input: a point $x$, noisy data $X=[x_1, \cdots, x_N]$, \YQ{bandwidth parameters $r$ and $r^\prime$}, a step length parameter $a$, a tolerance $\epsilon$, and the maximal number of iteration $T$. \\
Output: projection $\tilde{x}$ of $x$ onto $\M_{\rm out}$.
\begin{itemize}
\item[1.] Calculate $P_{x_i} = I - VV^T$ for each $x_i \in X$, where $V$ is the $D \times d$ matrix whose columns are the eigenvectors corresponding to the largest $d$ eigenvectors of $\sum_{j \in I_{x_i,\YQ{r^\prime}}}(x_j-x_i)(x_j-x_i)^T$.
\item[2.] Set $t=1$.
	\begin{itemize}
		\item[(1).] Calculate $\tilde{\alpha}_i(x)$ and $\alpha_i(x)$ for $i \in I_{x,r}$ by (\ref{def:alpha}).
		\item[(2).] Plug $\{ \tilde{\alpha}_i(x), \alpha_i(x), P_{x_i} \}_{i \in I_{x,r}}$  into (\ref{eq:gradF}) to obtain the gradient ${\rm grad}(x)$ of  $\|f(x)\|_2^2$.
		\item[(3).] Update $t = t+1$ and $x = x - a \cdot {\rm grad}(x)$.
		\item[(4).] Repeat (1) to (3) until the tolerance condition $\|f(x)\|_2^2 \leq \epsilon$ or the maximal iteration $T$ is met.
	\end{itemize}
\item[3.] Output $\tilde{x} = x$.
\end{itemize}
\end{algorithm}

As explained in Subsection \ref{sec:motivation}, by removing the unreliable discs which centered at the sample points as in \cite{mohammed2017manifold} and \cite{pmlr-v75-fefferman18a}, one would expect an improved performance compared to these two methods. Assuming the data points are sampled from a tubular neighborhood, \cite{Aizenbud2021} denoises the sample points iteratively using a local polynomial regression.  As the degree increases, polynomial regression fits a manifold better when the noise is limited but on the other hand  a polynomial regression exhibits sensitivity once noise increases. As a method designed for Gaussian noise, our method is expected to be more robust as noise increases. To support this claim, we test methods in \cite{mohammed2017manifold} (marked by km17), \cite{pmlr-v75-fefferman18a} (marked by cf18), and \cite{Aizenbud2021} with polynomial degree 1 and 2 (marked by ya21(deg=1) and ya21(deg=2)) on manifolds with both constant and inconstant curvature, namely:
a circle embedded in $\mathbb{R}^2$, a sphere embedded in $\mathbb{R}^3$, and a torus embedded in $\mathbb{R}^3$.
To ensure a traceable comparison, all the tests are conducted in the following way, similar to that of \cite{mohammed2017manifold}:
\begin{itemize}
    \item Sample $N$ points from the latent manifold, blur the points with Gaussian noise defined in (\ref{func:noise}) with given $\sigma$, and use the noisy data $X=[x_1, \cdots, x_N]$ to implicitly construct output manifolds.        
    \item Initialize a collection of points $P= [p_1, \cdots, p_{N_0}]$ around the latent manifold.
    \item Project each $p_i$ to the constructed output manifolds via km17, cf18, ya21(deg=1),  ya21(deg=2)) and our method, respectively. We will then obtain $\tilde{P}$ as the projection of $P$ for each method.
    \item Calculate the Hausdorff distance between each $\tilde{P}$ and $\M$ to estimate the Hausdorff distance between the corresponding $\M_{\rm out}$ and $\M$.
\end{itemize}

As projections, points in $\tilde{P}$ lie on the corresponding $\M_{\rm out}$, and the Hausdorff distance $H(\tilde{P}, \M)$ could estimate $H(\M_{\rm out}, \M)$ when $\tilde{P}$ are dense enough. This motivates us to evaluate the approximation error of $\M_{\rm out}$ to $\M$ by $H(\tilde{P}, \M)$. To project a point $p$ onto a manifold defined by (\ref{out:ours}), we design algorithm \ref{alg:GD}. Taking $x=p$ and $f$ in algorithm \ref{alg:GD} as (\ref{fun:out_ours}), we could project $p$ onto our output manifold. It should be noted that the difficulty of calculating such a gradient lies in calculating a gradient of orthogonal projection, which can be addressed, according to \cite{Shapiro1995}. Detailed formula refers to Appendix \ref{app:gradient}. \cite{mohammed2017manifold} suggested a subspace-constrained gradient descent algorithm to project a point onto $\M_{\rm out}$ constructed by km17. Thus, we adopt this algorithm to implement km17 in this simulation. Although \cite{pmlr-v75-fefferman18a} did not consider the issue, we nevertheless implement their method too via algorithm \ref{alg:GD}, treating $f(x)$ as the approximated bias at $x$ defined by \cite{pmlr-v75-fefferman18a}.

The details of this simulation are as follows: we uniformly sample $N$ points denoted by $y_1, \cdots, y_{N}$ from each target manifold and i.i.d. sample $\xi_1, \cdots, \xi_{N}$ from a Gaussian distribution (\ref{func:noise}) with a given standard derivation $\sigma$. Then, the noisy data $X = \{x_i\}_{i=1}^{N}$  is constructed by $x_i = y_i + \xi_i$. The initial points $P$ are sampled from the tube centered at $\M$ with radius $\frac{1}{2}\sqrt{\frac{\sigma}{D}}$, so that $d(p_i,\M) \leq \sqrt{\sigma}$ for each $p_i$. According to Theorem \ref{thm:Hdist}, $d(\tilde{p}_i,\M) \leq O(\sigma)$, which means the output points should be much closer to the latent manifold than the initial points. Again, we take $N_0 = N$ initial points for each test in the simulation.

\begin{figure}[th]
\centering
\includegraphics[width=0.19\textwidth]{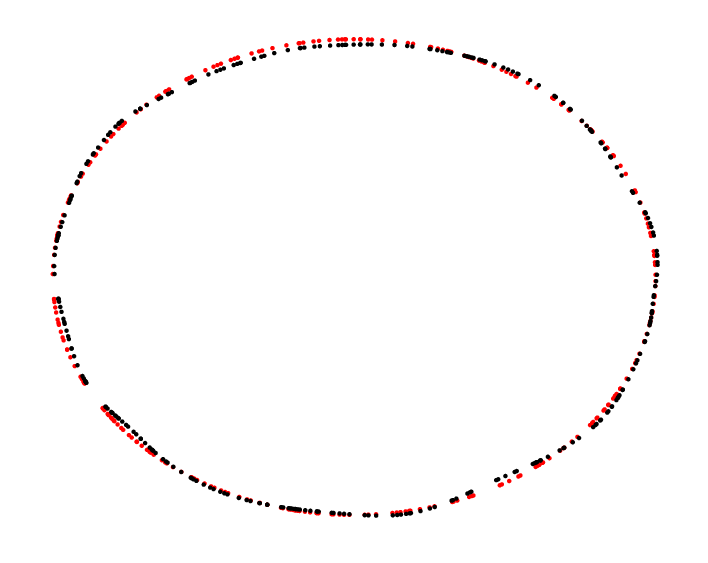}
\includegraphics[width=0.19\textwidth]{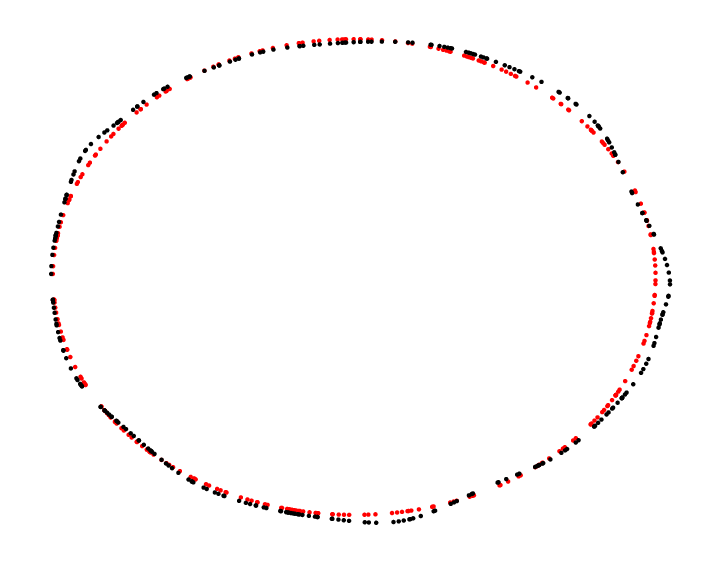}
\includegraphics[width=0.19\textwidth]{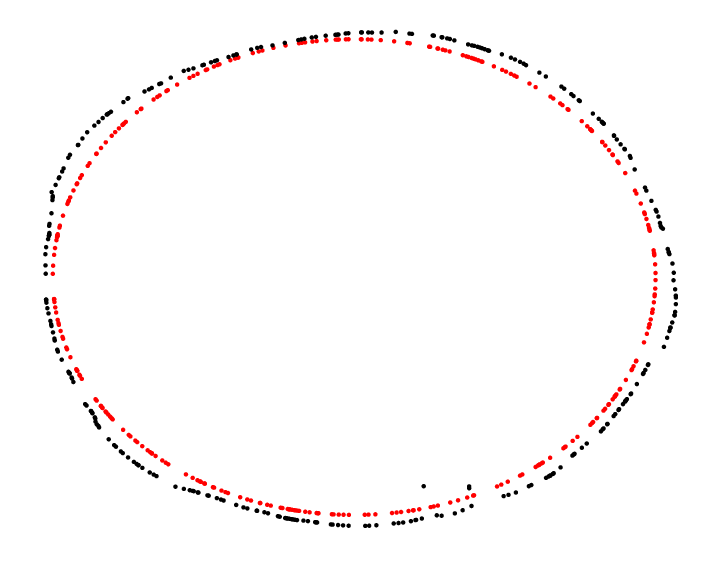}
\includegraphics[width=0.19\textwidth]{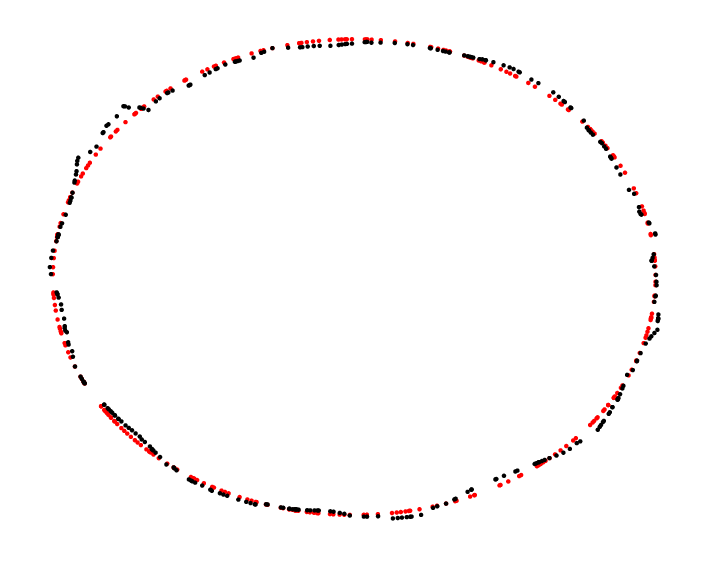}
\includegraphics[width=0.19\textwidth]{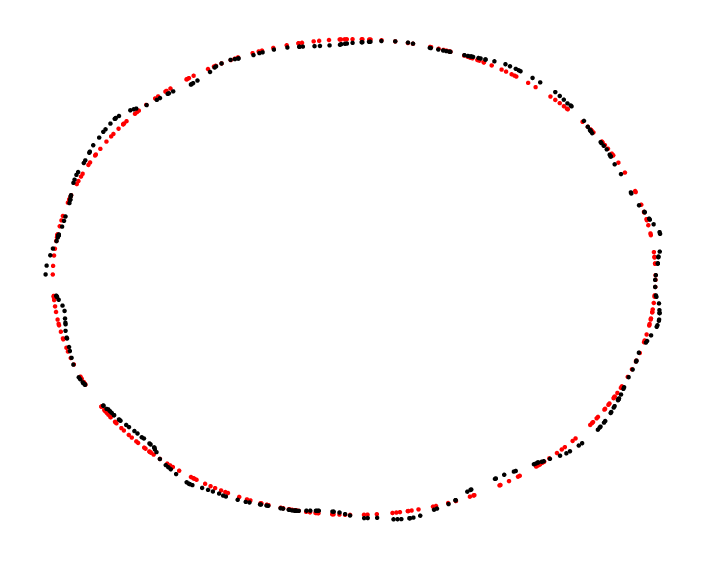}
\centering
\includegraphics[width=0.19\textwidth]{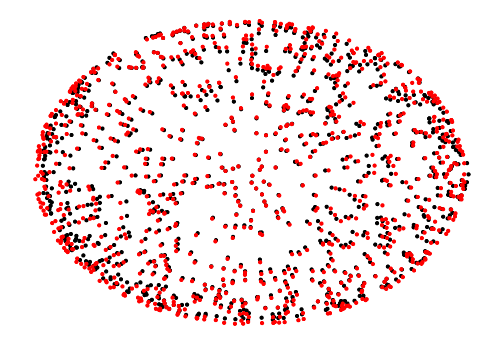}
\includegraphics[width=0.19\textwidth]{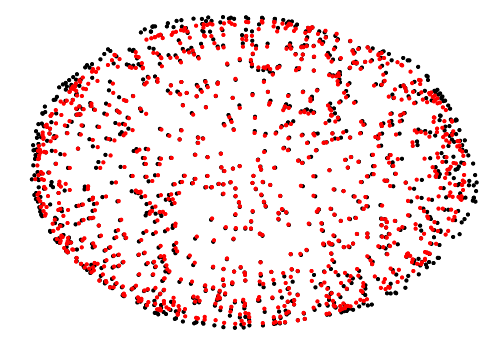}
\includegraphics[width=0.19\textwidth]{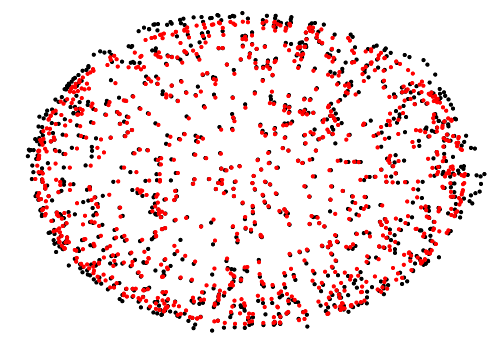}
\includegraphics[width=0.19\textwidth]{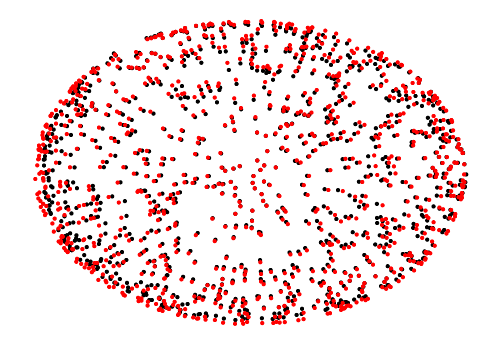}
\includegraphics[width=0.19\textwidth]{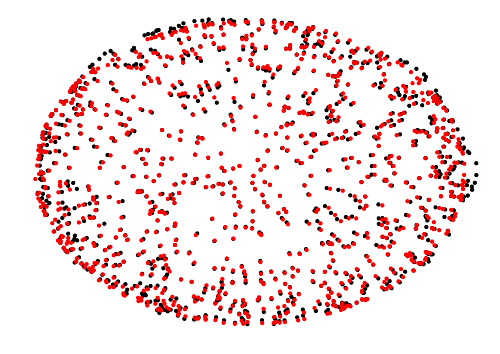}
\caption{The performance of our method, km17, cf18, ya21(deg=1) and ya21(deg=2)
 when fitting a circle (top row) and a sphere (bottom row), where black points represent points in $\tilde{P}$(black dots) and red points represents their projections onto $\M$.} \label{fig:circle_cf_xy}
\end{figure}

To implicitly construct the output manifolds, the methods--km17, cf18, and our method,each require a bandwidth parameter $r$. According to the theoretical analysis, $r = O(\sqrt{\sigma})$. So we take $r = \lambda \sqrt{\sigma}$ in this simulation, where $\lambda$ is tuned in a large range for each method and each $\sigma$. All the results reported in this section are the ones using the best $\lambda$. The method ya21 also requires a bandwidth parameter $h$, which is again selected as the best one tuned from a large range. In constructing $\tilde{\alpha}_i(x)$, our method requires $\beta \geq 2$. We take $\beta = d+2$ in the simulation, as \cite{pmlr-v75-fefferman18a} did.

\subsubsection{Manifold with constant curvature}
This part tests the manifold fitting methods for the circle in $\mathbb{R}^2$ and the sphere embedded in $\mathbb{R}^3$. For the circle, we set $N = N_0 = 300$, while for the sphere, we set $N = N_0 = 1000$. The different sample-size settings guarantee comparable density in each case, as Figure \ref{fig:circle_cf_xy}    illustrates that the $\tilde{P}$ (black dots) and their projection onto $\M$ (red dots) obtained by our method, cf18, cf18, ya21(deg=1) and ya21(deg=2), from left to right. The black dots and red dots can be treated as the discretized versions of $\M_{\rm out}$ and $\M$, respectively. Thus, a larger overlap of the two sets of dots means the manifold is better fitted. For the circle embedded in $\mathbb{R}^2$, we show the entire space in the left column, while for the sphere embedded in $\mathbb{R}^3$, we show the view from the positive $z$ axis. Figure \ref{fig:circle_cf_xy} shows that km17 clearly performs worse than the other methods in terms of fitting error. From the two estimated circles by ya21(deg=1) and ya21(deg=2), we observe that there are sharp corners – both at the top left and at the bottom right – an observation that confirms that the estimator by ya21 is not smooth. From the right edge of the circle and the sphere, we can also observe that our method preforms slightly better than cf18 in this experiment. 

\begin{figure}[htbp]
\centering
\includegraphics[height=1.5in,width=0.32\textwidth]{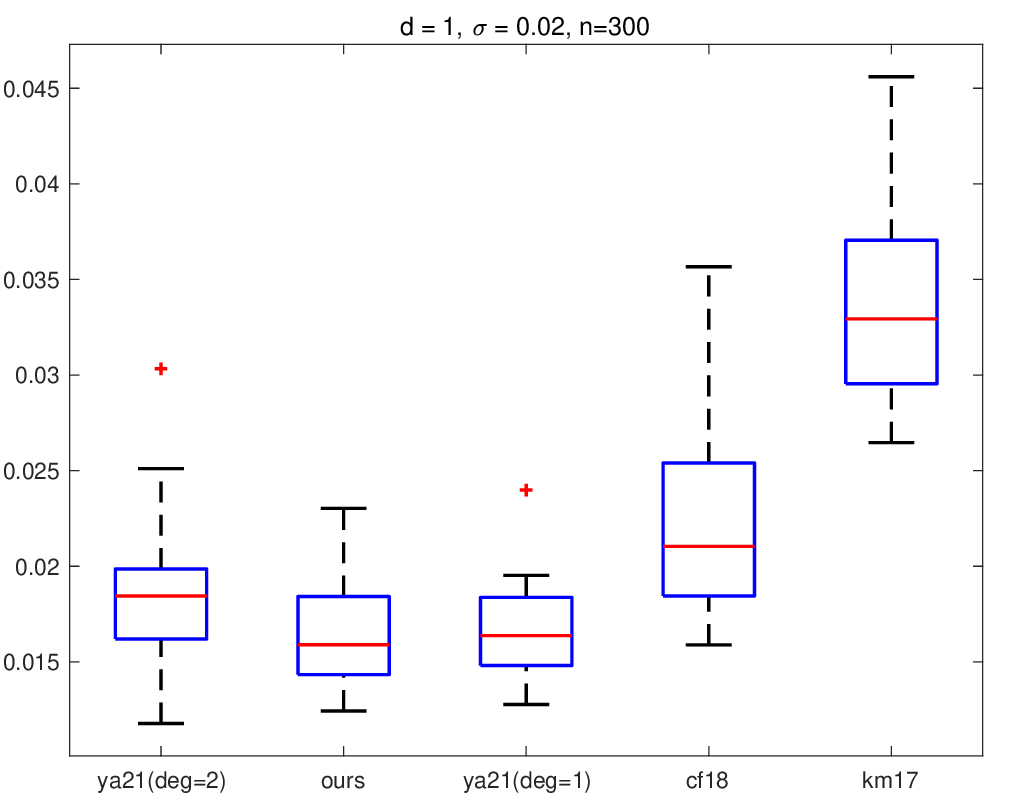}
\includegraphics[height=1.5in,width=0.32\textwidth]{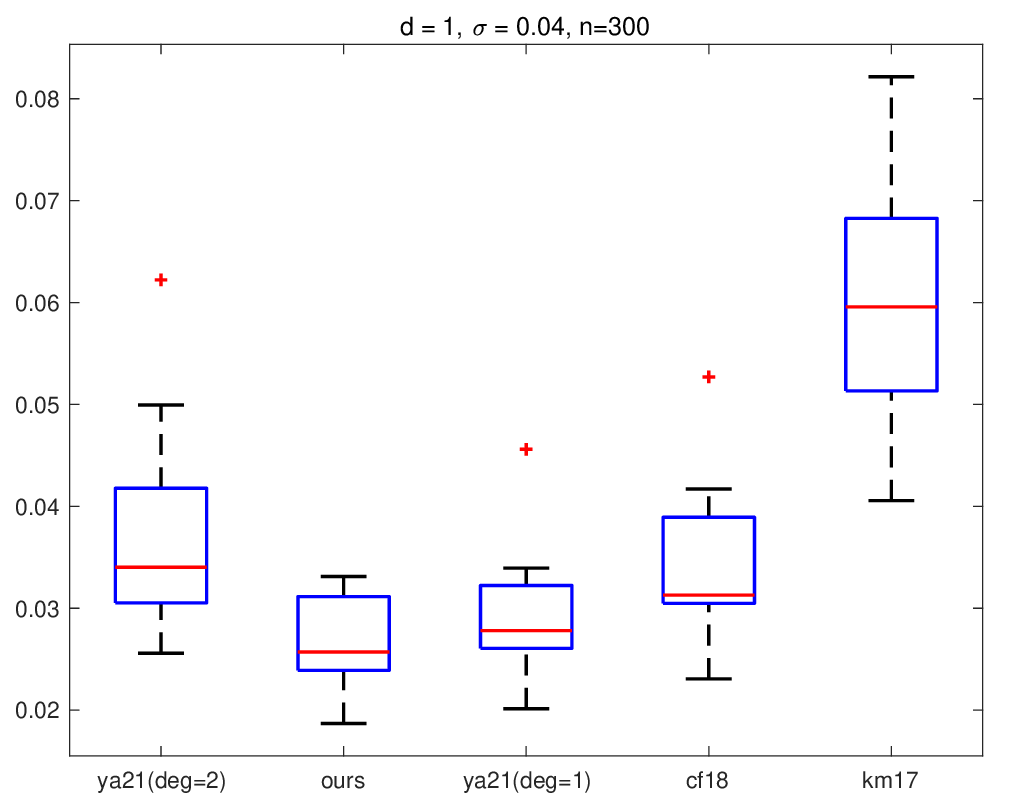}
\includegraphics[height=1.5in,width=0.32\textwidth]{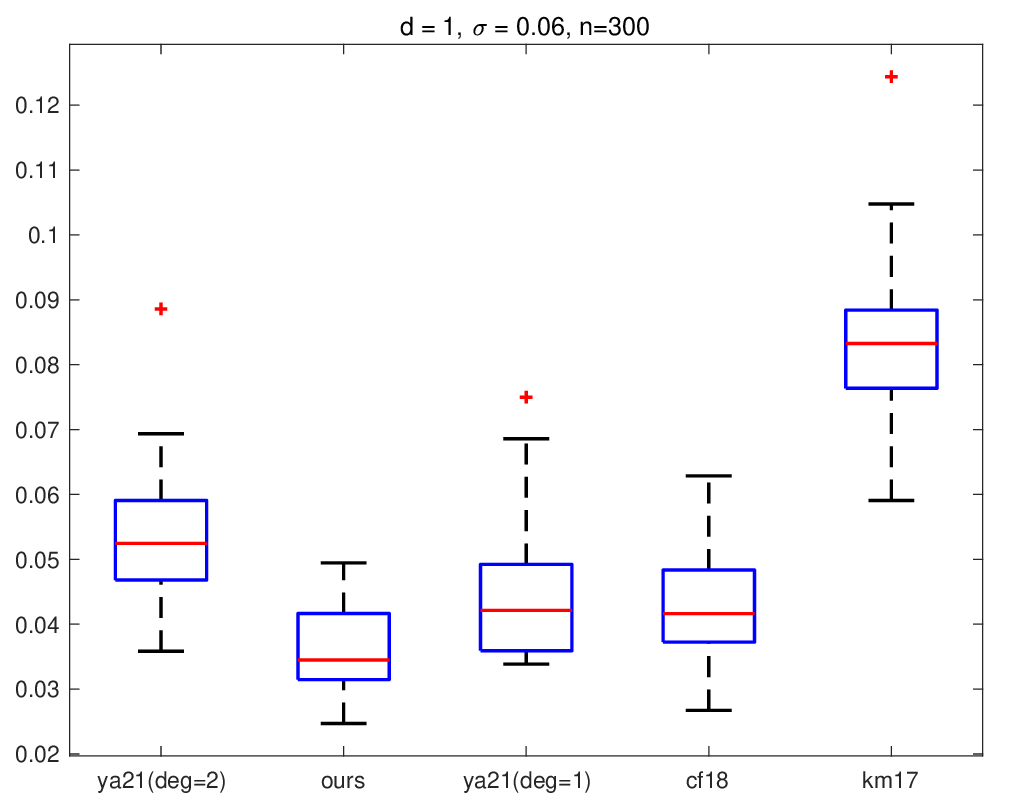}
\centering
\includegraphics[height=1.5in,width=0.32\textwidth]{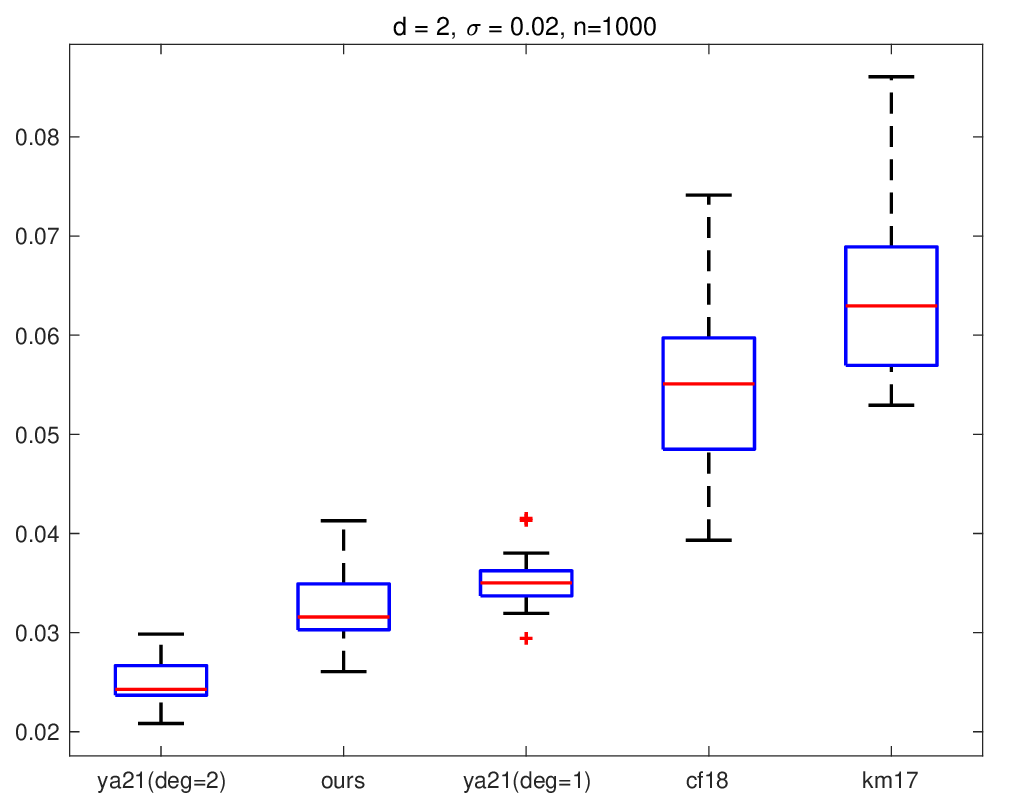}
\includegraphics[height=1.5in,width=0.32\textwidth]{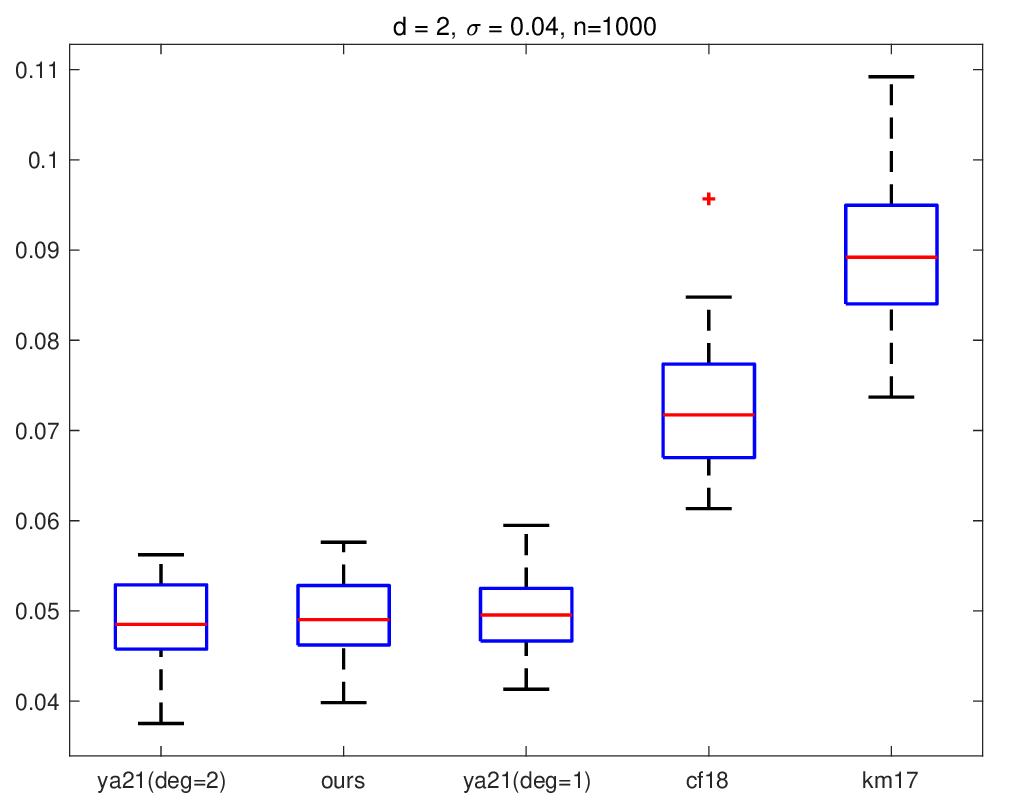}
\includegraphics[height=1.5in,width=0.32\textwidth]{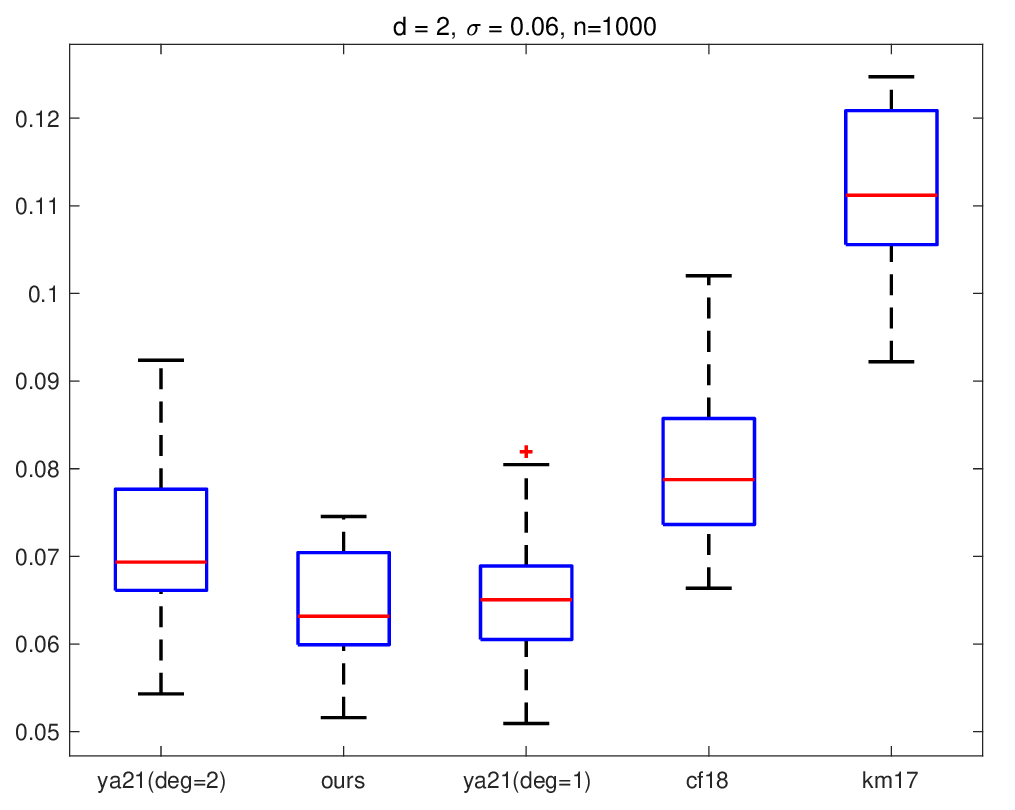}
\caption{The Hausdorff distance of fitting a circle (top row) and a sphere (bottom row) with $\sigma=0.02$ (left column),  $\sigma=0.04$ (middle column) and $\sigma=0.06$ (right column) using ya21(deg=2), our method, ya21(deg=1), cf18 and km17 respectively.} \label{fig:circle_sphere_max}
\end{figure}

To confirm the superiority of our method, we repeat each test for 20 trials, and list the results of $H(\tilde{P}, \M)$ using the different methods in Figure \ref{fig:circle_sphere_max}. Generally speaking, our method outperforms cf18, km17 and ya21(deg=1) in the compared cases and although ya21(deg=2) performs slightly better than our method in instances of very low noise, it is much more sensitive than our method. As the $\sigma$ increases, ya21(deg=2) fails to outperform other methods. From Figure \ref{fig:circle_sphere_max}, $H(\M, \M_{\rm out}) = O(\sigma)$ for our method, which supports Theorem \ref{thm:Hdist}.

\subsubsection{Manifold with inconstant curvature}

We also implement the compared methods in the torus case, which is a type of manifold with inconstant curvature. Figure \ref{fig:torus_cf_xy} illustrates the case with $N=N_0=800$ and $\sigma=0.04$, and the torus embedded in $\mathbb{R}^3$ is shown from the positive z axis.  Here, the sample points in $\tilde{P}$ are marked by black dots and their projection onto $\M$ are marked by red dots. The five subfigures are obtained from our method, cf18, km17, ya21(deg=1) and ya21(deg=2), from left to right. From the top and right edges of the torus, we can observe that our method performs better than both cf18 and km17.  From the fourth subfigure, we can identify a clear gap between the red and black dots around the edge of the torus, which means ya21(deg=1) failed to fit these points but using a second degree polynomial, ya21(deg=2) achieves a better fitting as the right subfigure shows.

\begin{figure}[th]
\centering
\includegraphics[width=0.19\textwidth]{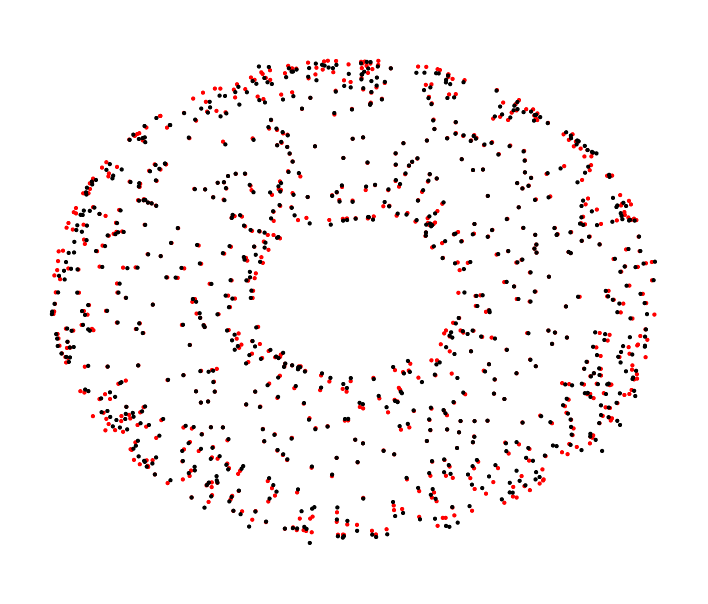}
\includegraphics[width=0.19\textwidth]{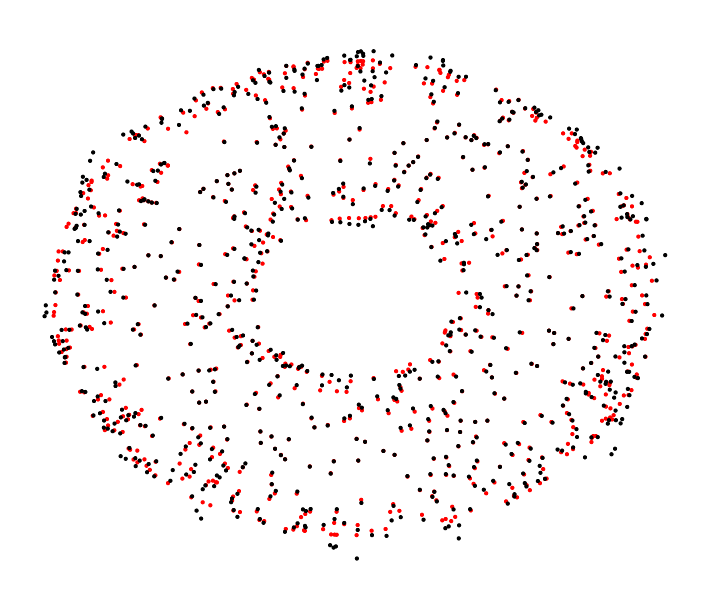}
\includegraphics[width=0.19\textwidth]{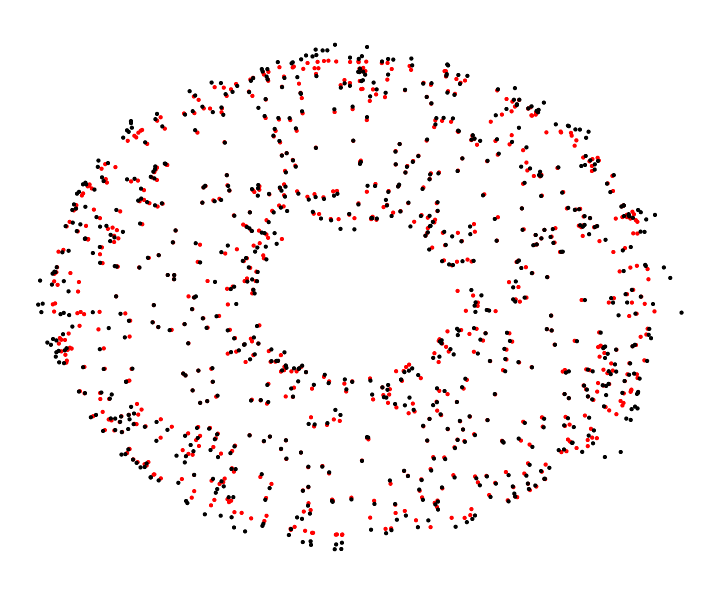}
\includegraphics[width=0.19\textwidth]{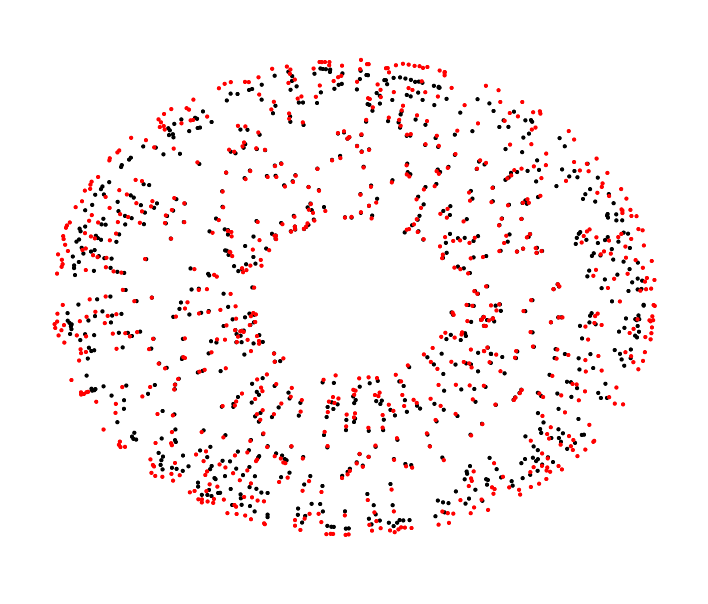}
\includegraphics[width=0.19\textwidth]{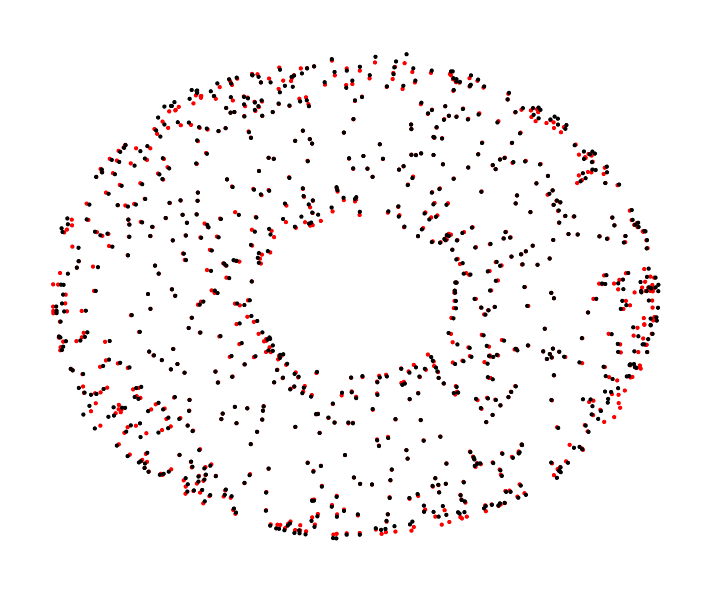}
\caption{The performance of our method, km17, cf18, ya21(deg=1) and ya21(deg=2)
 when fitting a torus with $N = N_0 = 800$ and $\sigma=0.04$, where black points represent points in $\tilde{P}$(black dots) and red points represent their projections onto $\M$.}\label{fig:torus_cf_xy}
\end{figure}

We also repeat each test for 20 trials and list the results of $H(\tilde{P}, \M)$ using the different methods shown in Figure \ref{fig:torus_max}. When $\sigma = 0.02$ and $\sigma = 0.04$, our method performs better than cf18, km17 and ya21(deg=1) but as $\sigma$ increases to $0.06$, the fitting problem becomes more difficult and the performance of km17, cf18 and our method are similar, which further demonstrates the sensitivity of ya21(deg=2). When $\sigma$ is small and the sample size is adequate, ya21(deg=2) outperforms the other methods but when the sample size decreases and $\sigma$ increases, the performance of ya21(deg=2) deteriorates rapidly.
\begin{figure}[htbp]
\centering
\includegraphics[height=1.5in,width=0.32\textwidth]{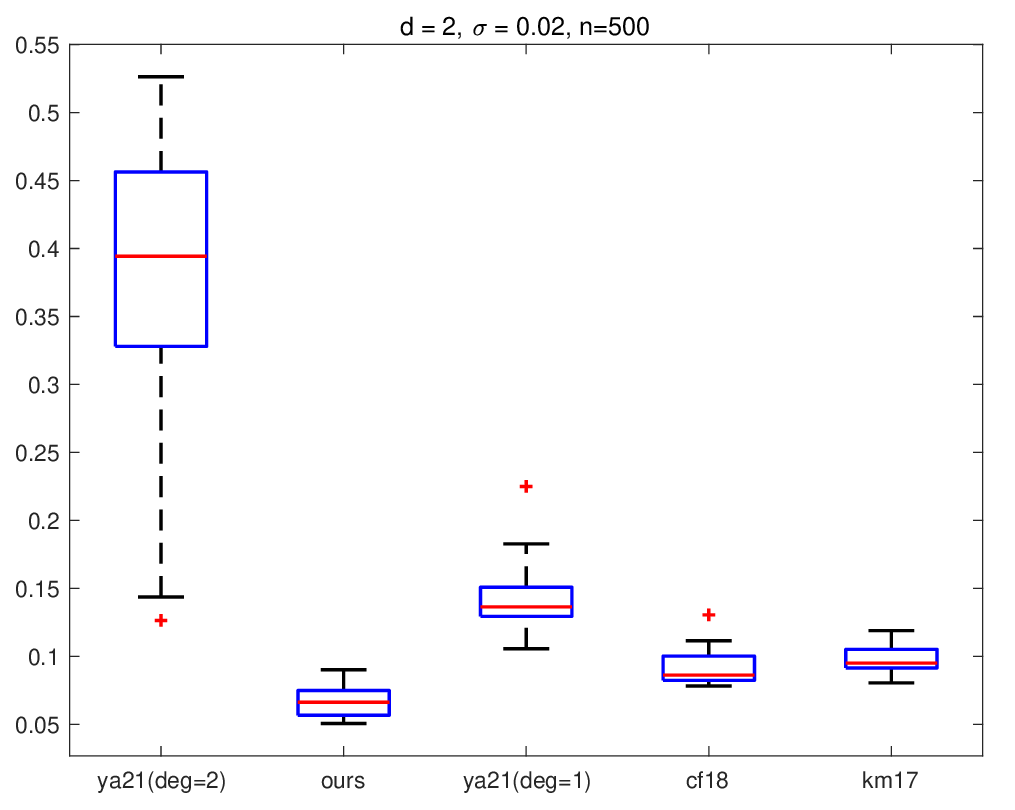}
\includegraphics[height=1.5in,width=0.32\textwidth]{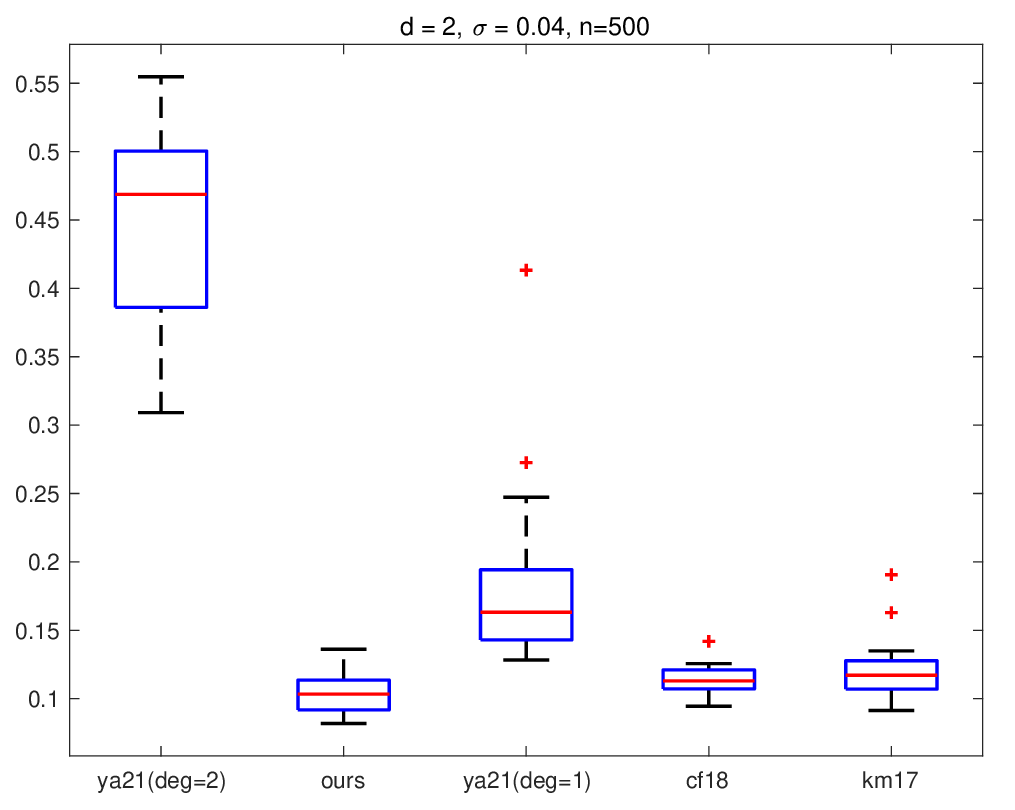}
\includegraphics[height=1.5in,width=0.32\textwidth]{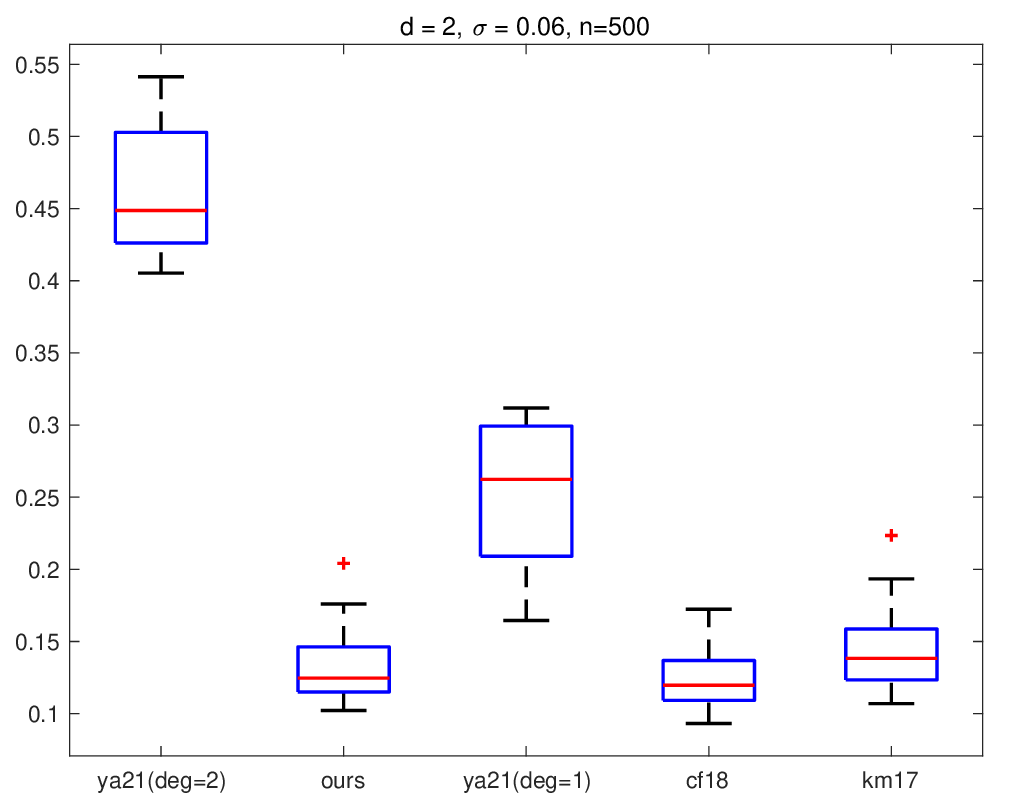}
\centering
\includegraphics[height=1.5in,width=0.32\textwidth]{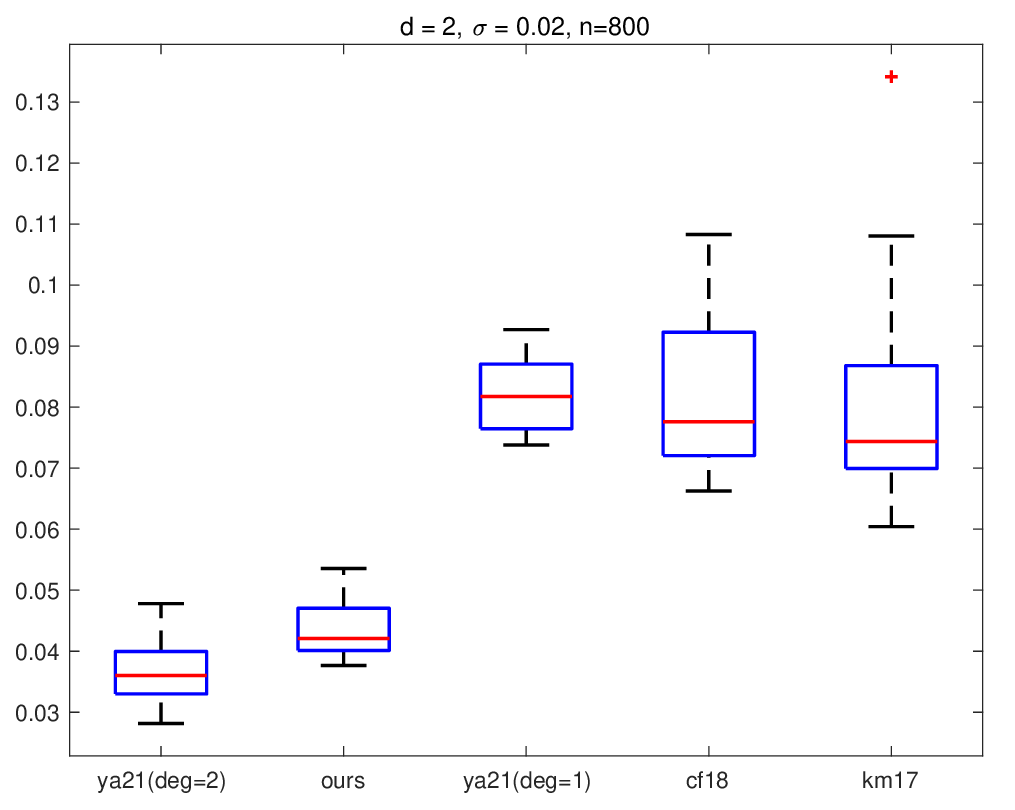}
\includegraphics[height=1.5in,width=0.32\textwidth]{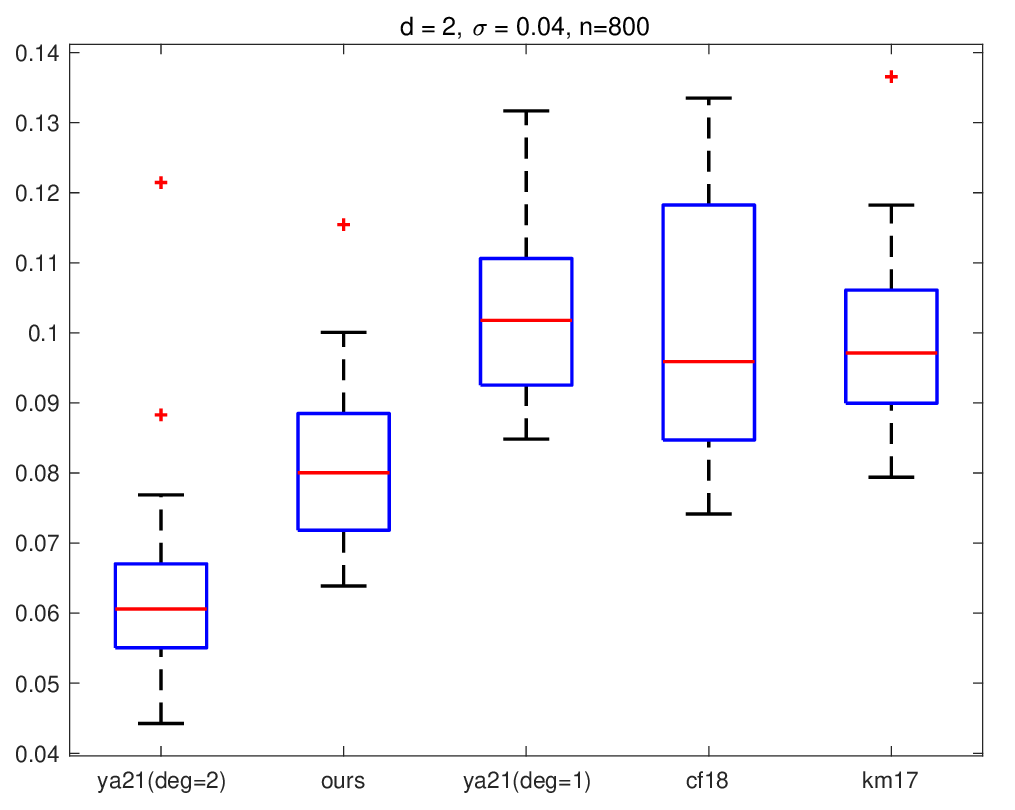}
\includegraphics[height=1.5in,width=0.32\textwidth]{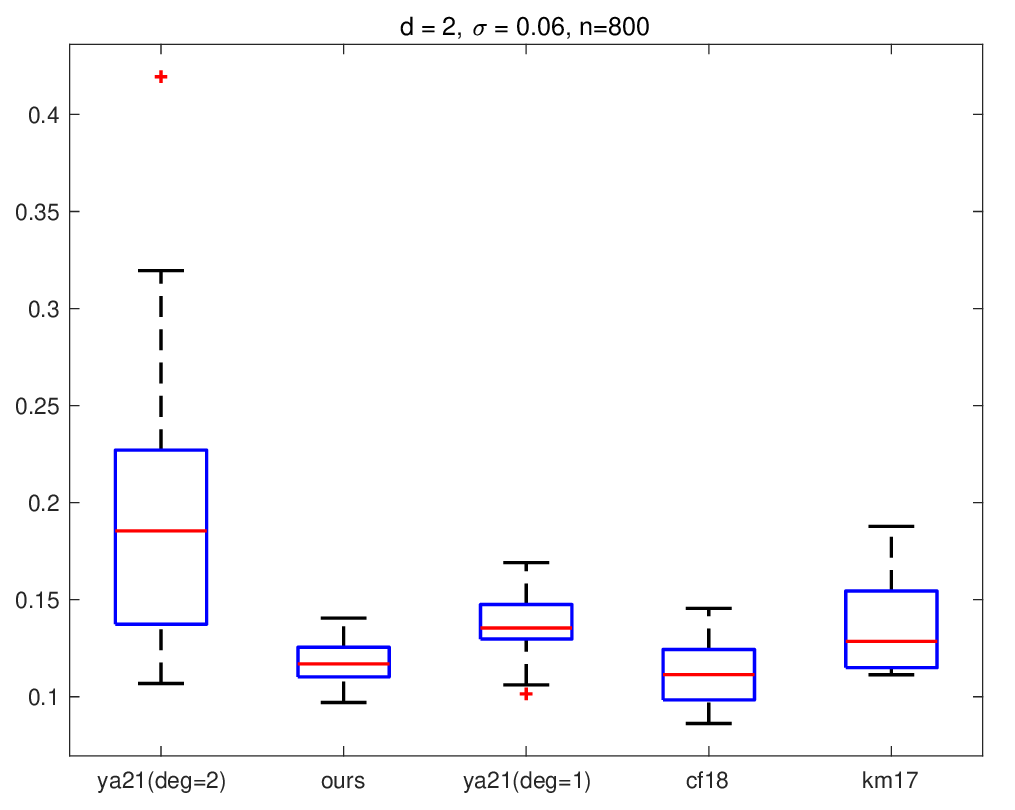}
\caption{The Hausdorff distance of fitting a torus given 500 (top row) and 1000 (bottom row) samples with $\sigma=0.02$ (left column),  $\sigma=0.04$ (middle column) and $\sigma=0.06$ (right column) using ya21(deg=2), our method, ya21(deg=1), cf18 and km17 respectively.} \label{fig:torus_max}
\end{figure}

\subsection{Facial image denoising}

This subsection considers a concrete case - denoising facial images selected from the video database in \cite{happy2012video}. We select 1,000 images of an individual turning his head around, then blurring those images via a Gaussian distribution with a different standard derivation $\sigma$. In this experiment, $\sigma$ is set to be the average of all pixels in 1,000 images multiplied by $\rho = 0.2, 0.3$, or $0.4$. The size of each facial image is $80 \times 80$, which means $D = 6400$. The dimension $d$ of the latent manifold is tuned from $\{1, 5, 10, 15, 20, 50, 75, 100\}$ for each method and we choose $d=10$ because of its outperformance.

\begin{figure}[htbp]
\centering
\includegraphics[width=0.18\textwidth]{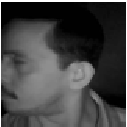}
\includegraphics[width=0.18\textwidth]{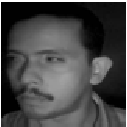}
\includegraphics[width=0.18\textwidth]{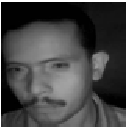}
\includegraphics[width=0.18\textwidth]{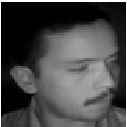}
\includegraphics[width=0.18\textwidth]{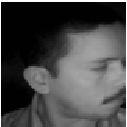}
\centering
\includegraphics[width=0.18\textwidth]{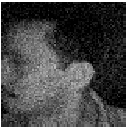}
\includegraphics[width=0.18\textwidth]{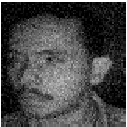}
\includegraphics[width=0.18\textwidth]{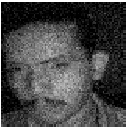}
\includegraphics[width=0.18\textwidth]{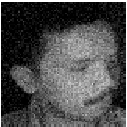}
\includegraphics[width=0.18\textwidth]{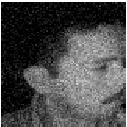}
\centering
\includegraphics[width=0.18\textwidth]{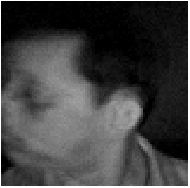}
\includegraphics[width=0.18\textwidth]{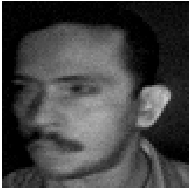}
\includegraphics[width=0.18\textwidth]{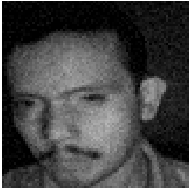}
\includegraphics[width=0.18\textwidth]{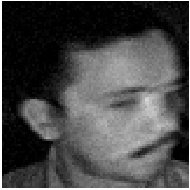}
\includegraphics[width=0.18\textwidth]{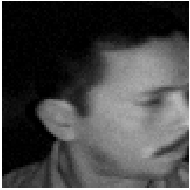}
\centering
\includegraphics[width=0.18\textwidth]{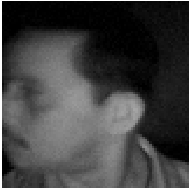}
\includegraphics[width=0.18\textwidth]{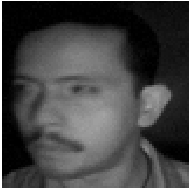}
\includegraphics[width=0.18\textwidth]{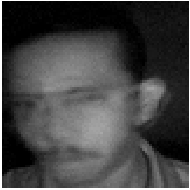}
\includegraphics[width=0.18\textwidth]{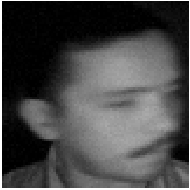}
\includegraphics[width=0.18\textwidth]{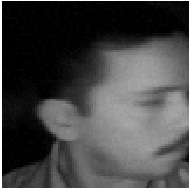}
\centering
\includegraphics[width=0.18\textwidth]{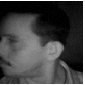}
\includegraphics[width=0.18\textwidth]{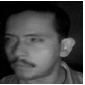}
\includegraphics[width=0.18\textwidth]{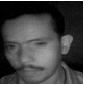}
\includegraphics[width=0.18\textwidth]{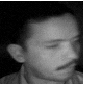}
\includegraphics[width=0.18\textwidth]{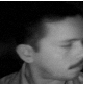}
\centering
\includegraphics[width=0.18\textwidth]{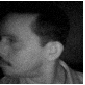}
\includegraphics[width=0.18\textwidth]{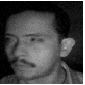}
\includegraphics[width=0.18\textwidth]{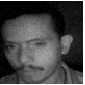}
\includegraphics[width=0.18\textwidth]{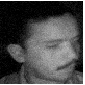}
\includegraphics[width=0.18\textwidth]{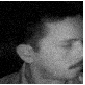}
\centering
\includegraphics[width=0.18\textwidth]{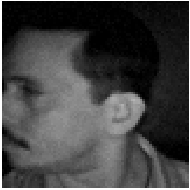}
\includegraphics[width=0.18\textwidth]{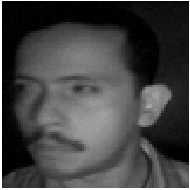}
\includegraphics[width=0.18\textwidth]{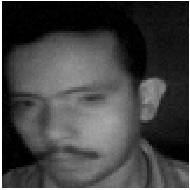}
\includegraphics[width=0.18\textwidth]{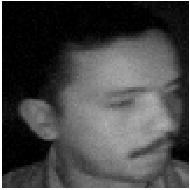}
\includegraphics[width=0.18\textwidth]{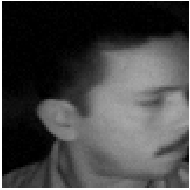}
\caption{Performance of facial image denoising with $\rho = 0.3$. The first row consists of original images while the second row consists of blurred images. The third to seventh rows contain deblurred images using km17, cf18, ya21(deg=1), ya21(deg=2) and our method, respectively. }\label{fig:face_0.3}
\end{figure}

From the 1,000 facial images, we select 5 with different head orientations. The top row of Figure \ref{fig:face_0.3} displays these five original images, while the second row of Figure \ref{fig:face_0.3} shows these five images blurred, with $\rho = 0.3$. The goal of this experiment is to denoise these five blurred images by projecting them to the manifold learnt by the remaining 995 blurred images, which are treated as the noisy samples. To achieve the denoising, we use km17, cf18, ya21(deg=1), ya21(deg=2) and our method to construct the output manifold with the 995 noisy samples, and project the five tested images to each output manifold. When the output manifold correctly fits the latent one, projecting blurred images to the output manifold denoises these facial images. In this experiment, we take $\beta=2$ for our method to construct $\tilde{\alpha}_i(x)$. If cf18 uses $\tilde{\alpha}_i(x)$ as \cite{pmlr-v75-fefferman18a} has suggested, it would not work quite satisfactorily, because of the over-large power $d+2$ rather than $\beta$. Therefore, we take the same $\tilde{\alpha}_i(x)$ for cf18 and our method to make the results comparable.

The last three rows of Figure \ref{fig:face_0.3} show the denoised images obtained by km17, cf18, ya21(deg=1), ya21(deg=2) and our method, respectively. The first and third facial images were not recovered by km17. Although the faces in the other three images obtained by km17 can be distinguished, they are still very noisy. Cf18 could not recover the third image either, although the other four images obtained by cf18 are better than the ones obtained by km17. Both ya21(deg=1) and ya21(deg=2) can recover these five faces. However, the faces obtained by ya21(deg=2) are still somewhat fuzzy, compared with the ones obtained by ya21(deg=1) and our method. Our method recovered all the five faces, with the third face of much better quality than the faces from km17 and cf18.

The results with the settings $\rho = 0.2$ or $\rho = 0.4$ are listed in Figure \ref{fig:face_0.2} and Figure \ref{fig:face_0.4} (Appendix \ref{app:face}). When we take $\rho = 0.2$, the results of all three methods provide fairly good results. However, the results from km17 are somewhat noisy, with the obtained faces darker than the original ones. When $\rho = 0.4$, km17 is barely able to recover the faces, cf18 fails at the first and third ones, but our method can still provide acceptable faces.

\section{Discussion}
We have proposed a new output manifold $\M_{\rm out}$ to fit data collection with Gaussian noise. The theoretical analysis of $\M_{\rm out}$ has two main components: (1) the upper bound on $d(x, \M)$ for arbitrary $x$, which guarantees $\M_{\rm out}$ approximates $\M$ well, and (2) the upper bound on the second-order difference of $\M_{\rm out}$, which guarantees the smoothness of $\M_{\rm out}$.

To \YQ{demonstrate} the contribution of this paper, we compared our theoretical results 
to relevant works presented in \cite{mohammed2017manifold} and \cite{pmlr-v75-fefferman18a}. All three of these works aim to fit data collection by a smooth manifold, while the difference among these works lies in the assumptions about noise.  \cite{mohammed2017manifold} requires the data to be noiseless, which is the most strict assumption of the three. As mentioned in the Introduction, \cite{pmlr-v75-fefferman18a} essentially requires the noise of data to be bounded, that is, the data collection $X$ satisfying $H(X, \M) \leq O(r^2)$\YQ{, where} $H(\cdot, \cdot)$ denotes the Hausdorff distance. If the noise of data obeys a Gaussian distribution, the researchers would select a subset from the entire dataset, assume the noise of the subset is bounded, and implement their proof on this subset of data. However, their sample selection step imposes a lower bound on $r$, meaning that the upper bound of $H(\M, \M_{\rm out})$ cannot tend to $0$. This paper, therefore, proposes a method to address the problem of Gaussian noise, which is commonly assumed but unsolved in relevant works. Unlike the bounded noise, $X$ with Gaussian noise are not required to satisfy $H(X, \M) \leq O(r^2)$, which increases the difficulty of manifold fitting.

According to the discussion in Subsection \ref{sec:motivation} and the experiment results, our method could achieve a smaller approximating error than the methods presented in \cite{mohammed2017manifold} and \cite{pmlr-v75-fefferman18a}. One possible reason is that we use the weighted average $\sum_{i \in I_{x,r}} \alpha_i(x) P_{x_i}$ to estimate $\YQ{\Pi_{x^*}}$ rather than using each $P_{x_i}$ separately. To explain this claim, we consider the following expression:
\begin{align} \label{ineq:discussion}
\sum_{i \in I_{x,r}} \alpha_i(x) P_{x_i} - \YQ{\Pi_{x^*}}
\!=\! \sum_{i \in I_{x,r}} \alpha_i(x) (P_{x_i} - \YQ{\Pi_{x_i^*}}) \! +\! \big(\sum_{i \in I_{x,r}} \alpha_i(x) \YQ{\Pi_{x_i^*}} - \YQ{\Pi_{x^*}} \big).
\end{align}
For certain “symmetric'' manifolds, the second term in the right hand side of (\ref{ineq:discussion}) might be much closer to zero matrix than \YQ{$(\Pi_{x_i^*} - \Pi_{x^*})$}.

A circle may be considered as an example. Suppose $x$, $x_1$, and $x_2$ are points on the circle satisfying $x_1-x = x-x_2$; then, the average of orthogonal projections onto the normal spaces at $x_1$ and $x_2$ equals the orthogonal projection onto the normal space at $x$, while the projection onto the normal space at $x_1$ (or $x_2$) differs from that at $x$ with an error in the order of $\|x-x_1\|_2$ (or $\|x-x_2\|_2$) by Lemma \ref{lma:Pi_ij}.

This phenomenon illustrates that the average of $\{\YQ{P_{x_i}}\}_{i \in I_{x,r}}$ approximates $\YQ{\Psi_x^\alpha}$ better than each \YQ{$P_{x_i}$} for certain manifolds. We benefit from this fact by using $\sum_{i \in I_{x,r}} \alpha_i(x) P_{x_i}$ to construct our output manifold, while \cite{mohammed2017manifold} and \cite{pmlr-v75-fefferman18a} use each $P_{x_i}$ separately instead. Characterizing the “symmetric'' property mentioned above and using this property in the methodology of manifold fitting is an attractive and promising topic, and our work on it will continue.



\acks{ZY and YX were supported by the MOE Tier 1 A-0004809-00-00 and Tier 2 R-155-000-184-112 at the National University of Singapore. ZY is also supported by Tier 2 A-0008520-00-00. ZY thanks Professor Charles Fefferman and Professor Hariharan Narayanan for their helpful discussions on some details of \cite{mohammed2017manifold} and \cite{pmlr-v75-fefferman18a} which we found very useful. ZY thanks Professor Shing-Tung Yau for his sagacious comments and the support from the Center of Mathematical Sciences and Applications at Harvard University.}


\newpage

\appendix
\section{Proofs} \label{app:A}
\subsection{Proof of Proposition \ref{prop:samplesize}} \label{proof:samplesize}

\begin{lemma} \label{lma:prob_in_I}
If $d(x, \M) \leq cr$ with some $c < 1$ and $c_1$ satisfies $c < c_1 \leq 1$, then there exists a constant $c'$ such that $\mathbb{P}(i \in I_{x, c_1r}) \geq c'r^d$.
\end{lemma}
\begin{proof}
 Setting $c_2$ be a constant satisfying $c < c_2 < c_1$, then
\begin{align*}
\mathbb{P}(i \in I_{x, c_1r}) & \geq \mathbb{P}\big(y_i \in \M \cap B_D(x,c_2r), \|\xi_i\|_2 \leq (c_1-c_2)r \big) \\
& = \mathbb{P}\big(y_i \in \M \cap B_D(x,c_2r) \big) \mathbb{P}\big(\|\xi_i\|_2 \leq (c_1-c_2)r \big). 
\end{align*}
In order to bound $\mathbb{P}(i \in I_{x, c_1r})$ below, we bound the two probability $\mathbb{P}(y_i \in \M \cap B_D(x,c_2r))$ and $\mathbb{P}(\|\xi_i\|_2 \leq (c_1-c_2)r)$, respectively. Since $d(x,\M) \leq cr < c_2r$, there exists $c_3$ such that 
\begin{align*}
\mathbb{P}\big(y_i \in \M \cap B_D(x,c_2r) \big) = \frac{{\rm Vol}\big(\M \cap B_D(x, c_3r)\big)}{{\rm Vol}(\M)} = c_3r^d.
\end{align*}
Since $\|\xi_i\|_2^2/\sigma^2$ obeys Chi-square distribution and $r = O(\sqrt{\sigma}) > C\sigma$ with $\sigma<1$, 
\begin{align*}
\mathbb{P}\big(\|\xi_i\|_2 \leq (c_1-c_2)r \big)
& = \mathbb{P}\left(\frac{\|\xi_i\|_2^2}{\sigma^2} \leq \frac{(c_1-c_2)^2r^2}{\sigma^2} \right) \\
& \leq 1 - \left( \frac{(c_1-c_2)^2r^2}{\sigma^2} e^{1-\frac{(c_1-c_2)^2r^2}{\sigma^2}}\right)^{D/2} \\
& \leq 1 - \left( (c_1-c_2)^2C^2e^{1-(c_1-c_2)^2C^2}\right)^{D/2} := c_4,
\end{align*}
where the second inequality holds by the Chernoff bound. Calculating the product of  $\mathbb{P}\big(y_i \in \M \cap B_D(x,c_2r)\big)$ and $\mathbb{P}\big(\|\xi_i\|_2 \leq (c_1-c_2)r \big)$ completes this proof.
\end{proof}

\begin{proof} {\bf of Proposition \ref{prop:samplesize}}
Setting $c_1 = 1$ in Lemma \ref{lma:prob_in_I}, we obtain $\mathbb{P}(i \in I_{x, r}) \geq c'r^d$. Hence, whether $i \in I_{x,r}$ or not, can be treated as a Bernoulli distribution with the expectation of $c'r^d$. Applying the Berry-Esseen theorem to the $N$ Bernoulli trials, there exists $c'<1$ such that $|I_{x,r}| \geq c'r^dN$ \YQ{with} probability $1-C/\sqrt{N}$. 
\end{proof}

\subsection{Proof of Lemma \ref{lma:covariance} and Lemma \ref{lma:lambda}} \label{app:B}
The following proof is derived from the notations illustrated in Figure \ref{fig:Pi} and the settings $\sigma<1$, $r' = 2r$ and $r = O(\sqrt{\sigma})$, which imply that there exist constants $C$ and $C'$ independent on $\sigma$ such that $r < C$ and $r' < C'$ by (\ref{ineq:bound_r}).

\begin{proof}{\bf of Lemma \ref{lma:covariance}}
Let $p_i - z^* = q_i$; then, $p_i' - z = p_i - z^* = q_i$. Considering $z_i - z = z_i - p_i' + p_i'-z = z_i - p_i' + q_i := \delta_i + q_i$, we can rewrite $ \sum_{i \in I_{z,r'}} (z_i - z)(z_i-z)^T - \sum_{i \in I_{z,r'}} (p_i - z^*)(p_i-z^*)^T$ as
\begin{align}\label{equ:err_lma_cov}
    \sum_{i \in I_{z,r'}} q_i \delta_i^T + \delta_i q_i^T + \delta_i \delta_i^T.
\end{align}

To begin with, we bound $\|\delta_i\|_2$. Recalling that the projection onto the normal space at $z^*$ is $\YQ{\Pi_{z^*}}$,
\begin{align*}
    \|\delta_i\|_2 & = \|\YQ{\Pi_{z^*}}(z_i - z)\|_2 \leq \|\YQ{\Pi_{z^*}}\big( (z_i-z_i^*) + (z_i^*-z^*) + (z^*-z) \big)\|_2 \\
    & \leq \|\YQ{\Pi_{z^*}}(z_i^* - z^*)\|_2 + \|z_i - z_i^*\|_2 + \|z^*-z\|_2 \\
    & \leq \frac{\|z_i^*-z^*\|_2^2}{\tau} + \|z_i-z_i^*\|_2 + \|z-z^*\|_2 \\
    & \leq \frac{(\|z_i^*-z_i\|_2 + \|z_i-z\|_2 + \|z-z^*\|_2)^2}{\tau} + \|z_i-z_i^*\|_2 + \|z-z^*\|_2.
\end{align*}
The last but one inequality holds in accordance with Proposition \ref{prop:reach}. As established previously, each $z_i$ is generated as $y_i+\xi_i$ with $y_i \in \M$ and $\xi_i \sim N(0,\sigma^2I_D)$. Then, $\|\xi_i\|_2 = \|z_i-y_i\|_2 \geq \|z_i - z_i^*\|_2$ since $z_i^*$ is the projection of $z_i$ onto $\M$. Thus, $\|\delta_i\|_2$ can be bounded by
\[
\|\delta_i\|_2 \leq \frac{(\|\xi_i\|_2 + r' + \|z-z^*\|_2)^2}{\tau} + \|\xi_i\|_2 + \|z-z^*\|_2 \leq C_1 \big( \|\xi_i\|_2^2 + \|\xi_i\|_2 + {r'}^2 + \|z-z^*\|_2 \big).
\]
The last inequality is achieved by replacing certain $\|z-z^*\|_2$ by its upper bound $r'$ and replacing certain $r'$ by a constant independent on $\sigma$ , since $r'<C'$ \YQ{by $r'=2r$ and (\ref{ineq:bound_r})}.
Considering the average over $I_{z,r'}$, we obtain 
\begin{align*}
\frac{1}{|I_{z,r'}|} \sum_{i \in I_{z,r'}} \|\delta_i\|_2 \leq  C_1 \big( \psi_2 + \psi_1 + r^2 + \|z-z^*\|_2 \big),  
\end{align*}
and
\begin{align*} 
\frac{1}{|I_{z,r'}|} \sum_{i \in I_{z,r'}} \|\delta_i\|_2^2 \leq  C_2\big( \psi_4 + \psi_3+\psi_2 + r\psi_1 + r^4 + r^2\|z-z^*\|_2 + \|z-z^*\|_2^2 \big).
\end{align*}
where $\psi_k :=  \frac{1}{|I_{z,r'}|} \sum_{i \in I_{z,r'}} \|\xi_i\|_2^k$, 
the above bounds are then plugged into the bound of (\ref{equ:err_lma_cov}) as follows:
\begin{align*}
    & \| \frac{1}{|I_{z,r'}|} \big( \sum_{i \in I_{z,r'}} (z_i - z)(z_i-z)^T - \sum_{i \in I_{z,r'}} (p_i - z^*)(p_i-z^*)^T \big)\|_F \\
    \leq & \|\frac{1}{|I_{z,r'}|}\sum_{i \in I_{z,r'}} \big( q_i \delta_i^T + \delta_i q_i^T + \delta_i \delta_i^T \big)\|_F \\
    \leq & \frac{1}{|I_{z,r'}|} \sum_{i \in I_{z,r'}} \big(2 \|q_i\|_2\|\delta_i\|_2 + \|\delta_i\|_2^2 \big) \\
     \leq & \frac{1}{|I_{z,r'}|} \sum_{i \in I_{z,r'}} \big(2 r'\|\delta_i\|_2 + \|\delta_i\|_2^2 \big) \\
    \leq & C\big( \psi_4+\psi_3+\psi_2+r'\psi_1 + {r'}^3 + r'\|z-z^*\|_2 + \|z-z^*\|_2^2\big)
\end{align*}
The last but one inequality holds since $\|q_i\|_2 \leq \|z_i-z\|_2 \leq r'$. Replacing $\psi_k$ by corresponding summation completes the proof.
\end{proof}

\begin{lemma}[Theorem 21 in \cite{mohammed2017manifold}] \label{lma:clt_lambda}
Let $\Lambda_1, \cdots, \Lambda_k$ be i.i.d. random positive semidefinite $D\times D$ matrices with expected value $\E[\Lambda_i] = M \succeq \mu I$ and $\Lambda_i \preceq I$. Then for all $\epsilon \in [0,1/2]$,
\[
\mathbb{P} \Big[ \frac{1}{k} \sum_{i=1}^k \Lambda_i \notin [(1-\epsilon)M, (1+\epsilon)M] \Big] \leq 2D \exp \big\{ \frac{-\epsilon^2\mu k}{2\ln2}\big\}.
\]
Here, the matrix interval $ A \in [B, C]$ means $a_{ij} \in [b_{ij}, c_{ij}]$
 holds for any $i,j$ and the matrix ordering $A  \succeq B$ means $A-B$ is a positive semidefinite.
 \end{lemma}

\begin{proof} {\bf of Lemma \ref{lma:lambda}}
Before the proof of Lemma \ref{lma:lambda}, we provide the useful notations and contents. For convenience, $z^*$ is set to be the origin of the local coordinate system, and the coordinates in $T_{z^*}\M$ are set to be the first $d$ coordinates of the $D$ coordinates. We let $\mP_d : \R^D \to \R^D$ be an operator, setting the last $(D-d)$ entries of a vector to be zeros, that is, $\mP_d(v) = [v_1, \cdots, v_d, 0,\cdots, 0]^T$. We also let $\bar{\mP}_d$ be the operator, setting the first $d$ entries of a vector to be zeros, that is, $\bar{\mP}_d = \I - \mP_d$, with $\I$ being the identity operator. Notations $\bar v := \mP_d(v)$ and $\hat{v} = \bar{\mP}_d(v)$ are also used without confusion.

Based on these notations, we calculate the useful bound on $\|\hat{\eta}\|_2$ for $\eta \in \M \cap B_D(z,r')$. Using the definition of $\bar{\eta}$, we obtain $\langle z^*-\bar{\eta}, z-z^*\rangle = 0$, $\langle z^*-\bar{\eta},\hat{\eta}\rangle = 0$, and, therefore
\begin{align*}
    {r^\prime}^2 & \geq \|z-\eta\|_2^2 = \|(z-z^*)+(z^*- \bar{\eta})    -  \hat{\eta} \|_2^2 \\
    & \geq \|z-z^*\|_2^2-2\|z-z^*\|_2\|\hat{\eta}\|_2+\|z^*-\bar{\eta}\|_2^2 + \|\hat{\eta}\|_2^2 \\
    & = \|z-z^*\|_2^2-2\|z-z^*\|_2\|\hat{\eta}\|_2+\|z^*-\eta\|_2^2.
\end{align*}
Moreover, in accordance with Proposition \ref{prop:reach}, $\|z^*-\eta\|_2^2 \geq 2 \tau \|\hat{\eta}\|_2$. Combining these two inequalities, we obtain
\[
{r'}^2 - \|z-z^*\|_2^2 + 2\|z-z^*\|_2\|\hat{\eta}\|_2 \geq \|z^*-\eta\|_2^2 \geq 2\tau\|\hat{\eta}\|_2
\]
and, hence,
\begin{align} \label{neq:eta-Pi}
    \|\hat{\eta}\| \leq \frac{{r'}^2-\|z-z^*\|^2}{2(\tau-\|z-z^*\|)}.
\end{align}

We are now ready to prove Lemma \ref{lma:lambda}.
Let $\lambda_1 \geq \cdots \geq \lambda_D$ be the eigenvalues of  matrix
$\frac{1}{|I_{z,r'}|}\sum_{i \in I_{z,r'}} (p_i - z^*)(p_i-z^*)^T$ and $\mu_1 \cdots \geq \mu_D$ be the eigenvalues of the population covariance matrix $M$, that is,
\[
M:=\E[\frac{1}{|I_{z,r'}|}\sum_{i \in I_{z,r'}} (p_i - z^*)(p_i-z^*)^T].
\]
We see that $\lambda_{d+1} = \cdots = \lambda_D = \mu_{d+1} = \cdots = \mu_D =0$. Therefore, we need only a lower bound for $\lambda_d$, which can be obtained by relating its value to $\mu_d$ through a concentration inequality given in Lemma \ref{lma:clt_lambda}. Assuming the first $d$ coordinates are aligned with the eigenvectors corresponding to the $d$ largest eigenvalues of $M$, $\mu_d$ is the variance in the $d$-th direction. Clearly, the first $d$ coordinates are located in $T_{z^*}\M$. Let $\mathbb{P}$ be the probability measure on $T_{z^*} \M \cap B_D(z, r')$. For any $q \in T_{z^*}\M \cap B_D(z,r')$, we first bound $\mathbb{P}(q)$ above.

We set $S(q) = \{\zeta': \bar{\zeta}' = q\} \cap B_D(z,r')$, and $\hat{S}(q) = \cup_{\zeta' \in S(q)} \{\eta':|\eta'(i)-\zeta'(i)| \leq 3\sigma, \forall i=1,\cdots,D \}$, where $\eta(i)$ and $\zeta(i)$ represent the $i$-th element of $\eta$ and $\zeta$, respectively. Then, we have $\cup_{q \in T_{z^*} \M \cap B_D(z,r')} S(q) \subset B_D(z,r')$ and
\begin{align*}
\cup_{q \in T_{z^*} \M \cap B_D(z,r')} \hat{S}(q)
& \subset \cup_{\zeta' \in B_D(z,r')} \{\eta':|\eta'(i)-\zeta'(i)| \leq 3\sigma, \forall i=1,\cdots,D \} \\
& \subset B_D(z,r'+3\sigma \sqrt{D}).
\end{align*}
The probability at $q$ is
\begin{align}
    \mathbb{P}(q) &= \frac{(2\pi\sigma)^{-D/2}}{\Vol(\M)} \int_{S(q)} d\zeta' \int_{\M} e^{-\|\eta'-\zeta'\|_2^2/2\sigma^2} d\mu_{\M}(\eta') \nonumber\\
    & = \frac{(2\pi\sigma)^{-D/2}}{\Vol(\M)} \int_{S(q)} d\zeta' \int_{\M \cap \hat{S}(q)} e^{-\|\eta'-\zeta'\|_2^2/2\sigma^2} d\mu_{\M}(\eta') \label{eq:dom1} \\
    & + \frac{(2\pi\sigma)^{-D/2}}{\Vol(\M)} \int_{S(q)} d\zeta' \int_{\M \setminus \hat{S}(q)} e^{-\|\eta'-\zeta'\|_2^2/2\sigma^2} d\mu_{\M}(\eta') \label{eq:dom2}.
\end{align}
We bound $\mathbb{P}(q)$ above by bounding (\ref{eq:dom1}) and (\ref{eq:dom2}).
\begin{align*}
    (\ref{eq:dom1}) & = \frac{(2\pi\sigma)^{-D/2}}{\Vol(\M)}\int_{S(q)} d\zeta' \int_{\M \cap \hat{S}(q)} e^{-\|\bar{\eta}'-q\|_2^2/2\sigma^2} e^{-\|\hat{\eta}'-\hat{\zeta}'\|_2^2/2\sigma^2} d\mu_\M(\eta') \\
    & \leq \frac{(2\pi\sigma)^{-D/2}}{\Vol(\M)} \int_{\M \cap \hat{S}(q)} e^{-\|\bar{\eta}'-q\|_2^2/2\sigma^2} \Big( \int_{0_d \times \R^{D-d}} e^{-\|\hat{\eta}'-\hat{\zeta}'\|_2^2/2\sigma^2} d \hat{\zeta}' \Big) d\mu_\M(\eta') \\
    & = \frac{(2\pi\sigma)^{-d/2}}{\Vol(\M)} \int_{\M \cap \hat{S}(q)} e^{-\|\bar{\eta}'-q\|_2^2/2\sigma^2} d\mu_\M(\eta') \\
    & = \frac{(2\pi\sigma)^{-d/2}}{\Vol(\M)} \int_{\mP_d(\M\cap \hat{S}(q))} e^{-\|\bar{\eta}'-q\|_2^2/2\sigma^2}\sqrt{\det \big(I+J(\bar{\eta}')^TJ(\bar{\eta}' )\big)} d \bar{\eta}' \\
    & \leq \frac{(2\pi\sigma)^{-d/2}}{\Vol(\M)} \Big( 1+ \frac{C^2(r'+3\sigma\sqrt{D})^2}{\tau^2}\Big)^{d/2} \int_{\R^d \times 0_{D-d}} e^{-\|\bar{\eta}'-q\|_2^2/2\sigma^2} d\bar{\eta}' \\
    & = \frac{1}{\Vol(\M)} \Big( 1+ \frac{C^2(r'+3\sigma\sqrt{D})^2}{\tau^2}\Big)^{d/2}.
\end{align*}
The last inequality holds since $\|J(\bar{\eta}')\|_F \leq (C(r'+3\sigma\sqrt{D}))/\tau$ with $\eta' \in B_D(z, r'+3\sigma \sqrt{D})$.
According to the definition of $S(q)$ and $\hat{S}(q)$, we have for any $\eta' \in \M \setminus \hat{S}(q)$ and $\zeta' \in S(q)$ the formula $|\eta(i)'-\zeta(i)'| \leq 3 \sigma$, which implies
\[
(2\pi\sigma^2)^{-D/2} \int_{S(q)} e^{-\|\eta'-\zeta'\|/2\sigma^2} d\zeta' \leq (0.01)^D \quad \forall \eta' \in \M \setminus \hat{S}(q).
\]
Hence,
\begin{align*}
    (\ref{eq:dom2}) &= \frac{(2\pi\sigma)^{-D/2}}{\Vol(\M)} \int_{\M \setminus \hat{S}(q)} \Big( \int_{S(q)} e^{-\|\eta'-\zeta'\|_2^2/2\sigma^2} d\zeta' \Big) d \mu_{\M}(\eta') \\
    & \leq \frac{\Vol(\M \setminus \hat{S}(q))}{\Vol(\M)} (0.01)^D \leq (0.01)^D.
\end{align*}
In summary, we have
\begin{align}
    \mathbb{P}(q) \leq \frac{1}{\Vol(\M)} \Big( 1+ \frac{C^2(r'+3\sigma\sqrt{D})^2}{\tau^2}\Big)^{d/2} + (0.01)^D
\end{align}
for any $ q \in T_{z^*}\M \cap B_D(z,r')$.

We consider only the lower bound for $q$ in a subset of $T_{z^*} \M \cap B_D(z,r')$, namely $T_{z^*} \M \cap B_D(z^*,r_0)$, where $r_0$ is set as
\begin{align}
    \begin{split}
    r_0 = \min \{
    & \sqrt{{r'}^2-\big( \frac{{r'}^2-\|z-z^*\|_2^2}{2(\tau-\|z-z^*\|_2)} + \|z-z^*\|_2 + 3\sigma\sqrt{D-d} \big)^2}, \\
    & \sqrt{{r'}^2 - \big( \frac{{r'}^2-\|z-z^*\|_2^2}{2(\tau - \|z-z^*\|_2)} + \|z-z^*\|_2 \big)^2} - 3\sigma\sqrt{d}\}
    \end{split}
\end{align}
For any $q \in T_{z^*} \M \cap B_D(z^*,r_0)$ and $\eta \in \M \cap B_D(z,r')$, we can verify the following conclusions via (\ref{neq:eta-Pi}):
\begin{itemize}
    \item[(i)] The $d$-dimensional cube
    \begin{align*}
     \{ q': q'(i)=0 &  \  \forall i \geq d+1, \  |q'(j)-q(j)| \leq 3 \sigma \ \forall j \leq d \} \\
    & \subset   B_D(q, 3\sigma\sqrt{d}) \cap T_{z^*} \M
        \subset \{\bar{\eta}': \eta' \in \M \cap B_D(z,r')\},
    \end{align*}
    \item[(ii)] The $(D-d)$-dimensional cube
    \begin{align*}
        \{\eta': \eta'(i)=0 \ \forall i \leq d, \ |\eta'(j)-\eta(j)|\leq 3\sigma \ \forall j \geq d+1\}
        \subset \{\hat{\zeta}':\zeta' \in S(q)\}.
    \end{align*}
\end{itemize}
Now, we are ready to bound $\mathbb{P}(q)$ below for any $q \in B_D(z^*,r_0) \cap T_{z^*} \M$.
\begin{align*}
    \mathbb{P}(q) &= \frac{(2\pi\sigma)^{-D/2}}{\Vol(\M)} \int_{S(q)} d\zeta' \int_{\M} e^{-\|\eta'-\zeta'\|_2^2/2\sigma^2} d\mu_{\M}(\eta') \\
    & \geq \frac{(2\pi\sigma)^{-D/2}}{\Vol(\M)} \int_{S(q)} d\zeta' \int_{\M \cap B_D(z,r')}  e^{-\|\bar{\eta}'-q\|_2^2/2\sigma^2} e^{-\|\hat{\eta}'-\hat{\zeta}'\|_2^2/2\sigma^2} d\mu_\M(\eta') \\
    & \geq \frac{(0.99)^{D-d}(2\pi\sigma^2)^{-d/2}}{\Vol(\M)} \int_{\M \cap B_D(z,r')}  e^{-\|\bar{\eta}'-q\|_2^2/2\sigma^2} d\mu_\M(\eta') \\
    & = \frac{(0.99)^{D-d}(2\pi\sigma^2)^{-d/2}}{\Vol(\M)} \int_{\mP_d(\M \cap B_D(z,r'))}  e^{-\|\bar{\eta}'-q\|_2^2/2\sigma^2} \sqrt{\det\big(I+J(\bar{\eta}')^TJ(\bar{\eta}') \big)}d\bar{\eta'} \\
    & = \frac{(0.99)^{D-d}(2\pi\sigma^2)^{-d/2}}{\Vol(\M)} \int_{\mP_d(\M \cap B_D(z,r'))}  e^{-\|\bar{\eta}'-q\|_2^2/2\sigma^2} d\bar{\eta'} \\
    & \geq \frac{(0.99)^D}{\Vol(\M)}
\end{align*}
The last but one inequality holds since $\sqrt{\det\big(I+J(\bar{\eta}')^TJ(\bar{\eta}') \big)} \geq 1$.

Since $\mu_d$ is the variance in the $d$-th direction, we have
\begin{align*}
    \mu_d & = \frac{1}{\int_{T_{z^*} \M \cap B_D(z,r')} \mathbb{P}(q') d\mathcal{L}_d(q')} \int_{ T_{z^*} \M \cap B_D(z,r')} q_d^2 \mathbb{P}(q) d \mathcal{L}_d(q)  \\
    & \geq \frac{\alpha}{{\rm Vol}(B_d(r'))}
     \int_{T_{z^*} \M \cap B_D(z^*,r_0) } q_d^2  d \mathcal{L}_d(q) \\
    & \geq  \frac{\alpha}{{\rm Vol}(B_d(r'))} \int_0^{r_0} \int_0^\pi \cdots \int_0^\pi \int_0^{2\pi} \Big( \ell \Pi_{j=1}^{d-1} \phi_j \Big)^2 dV \\
    & = \frac{\Gamma(d/2+1) \alpha}{\pi^{d/2}(r')^d} \int_0^{r_0} \int_0^\pi \cdots \int_0^\pi \int_0^{2\pi} \ell^{d+1} \prod_{j=1}^{d-1}\sin^{d-j+1} \phi_j d\ell \prod_{j=1}^{d-1}d\phi_j.
\end{align*}
where $\alpha$ is the ratio between the lower bound and upper bound of $\mathbb{P}(q)$, namely,
\[
\alpha =  \frac{(0.99)^D}{(1+\frac{C^2(r'+3\sigma\sqrt{D})^2}{\tau^2})^{d/2}+(0.01)^D \Vol(\M)},
\]
and the third line follows with a change of coordinates. Substitute
\begin{align*}
\Big\{ q_1 \to \ell\cos \phi_1, q_{2\leq i \leq d-1} \to \ell \cos \phi_i \Pi_{j=1}^T{i-1}\sin \phi_j, q_d \to \ell \sin \Pi_{j=1}^{d-1} \phi_{j} \Big\}
\end{align*}
with $\phi_{d-1}\in [0, 2\pi], \phi_{i\leq d-2} \in [0, \pi], \ell \in [0, r_0]$, and let
\[
dV := \ell^{d-1}\Pi_{j=1}^{d-2} \sin^{d-j-i} \phi_j d\ell d\phi_1 \cdots d\phi_{d-1}.
\]
The integral in the fourth line can be evaluated by noting that $\int_{0}^{r_0} \ell^{d+1}d\ell = {r_0}^{d+2}/(d+2)$, $\int_{0}^{2\pi} \sin^2 \phi_{d-1} d\phi_{d-1} = \pi$ and $\int_{0}^{\pi} \sin^{d-j+1}\phi_j d\phi_j = \frac{\sqrt{\pi}\Gamma((d-j+2)/2)}{\Gamma(1+(d-j+1)/2}$ for $1 \leq j \leq d-2$. Simplifying as \cite{mohammed2017manifold} did, we get
\[
\mu_d \geq \frac{\alpha}{d+2}\frac{{r_0}^{d+2}}{(r')^{d}} = c_0\frac{(0.99)^D}{d+2} \frac{(r_0)^{d+2}}{(r')^d}.
\]

According to Lemma \ref{lma:clt_lambda}, for any $\epsilon \in [0, 1/2]$, $\lambda_d \geq (1-\epsilon) \mu_d$ \YQ{with} probability $1-d\exp\{\frac{-\epsilon^2\mu_d |I_{z,r'}|}{2\ln 2}\}$. 
Taking $\epsilon = 1/2$, we have
\[
\lambda_d  \geq c_0\frac{(0.99)^D}{d+2} \frac{(r_0)^{d+2}}{(r')^d}.
\]
\YQ{with} probability $1-d\exp\{\frac{-\epsilon^2\mu_d |I_{z,r'}|}{2\ln 2}\}$. Using $r = O(\sqrt{\sigma})$ and $\|z-z^*\|\leq (1+c)r$, we can simplify $r_0$ and find $c_0$ satisfying $r_0 \geq c_0r$. Hence, there exists a constant $c$ independent on $r$ such that $\lambda_d \geq c r^2$, which completes this proof. 
\end{proof}

\subsection{Proof of Proposition \ref{prop:alpha_bound}, Lemma \ref{lma:xi} and Lemma \ref{lma:Pi_ij} }\label{proof:xi_e_var}

\begin{proof} {\bf of Proposition \ref{prop:alpha_bound}}
To show that $\tilde{\alpha}(x)$ is bounded below by $c_0|I_{x,r}|$ is equivalent to showing that there exists constant $c_1>c$ and $c_2$ such that among the $|I_{x,r}|$ samples there are $c_2|I_{x,r}|$ ones lying in $B_D(x,c_1r)$, where $c_0$ in the lower bound is $c_0 = c_2(1-c_1^2)^{\beta}$. To quantify the number of samples $\{x_i\}_{i \in I_{x,r}}$ lying in $B_D(x,c_1r)$, we bound the conditional probability $\mathbb{P}(\|x_i-x\|_2\leq c_1r | i \in I_{x,r})$ below by calculating the lower bound of $\mathbb{P}(\|x_i-x\|_2\leq c_1r)$ and the upper bound of $\mathbb{P}(i \in I_{x,r})$, respectively.

By Lemma \ref{lma:prob_in_I}, we have $\mathbb{P}(i \in I_{x, c_1r}) \geq c_3r^d$. For the probability $ \mathbb{P}(i \in I_{x,r})$, we have
\begin{align} \label{probability:I_x}
\begin{split}
 \mathbb{P}(i \in I_{x,r})
 & = \mathbb{P}(\|x_i-x\|_2 \leq r, \ y_i \in \M \setminus B_D(x, Cr)) \\
 & + \mathbb{P}(\|x_i-x\|_2 \leq r, \ y_i \in \M \cap B_D(x, Cr)),
 \end{split}
\end{align}
where
\begin{align*}
\mathbb{P}(\|x_i-x\|_2 \leq r, \ y_i \in \M \cap B_D(x, Cr))
& \leq  \mathbb{P}(y_i \in \M \cap B_D(x, Cr)) \\
& = \frac{{\rm Vol}(\M \cap B_D(x, C_0r))}{{\rm Vol}(\M)} = C
r^d,
\end{align*}
and
\begin{align*}
\mathbb{P}(\|x_i-x\|_2 \leq r, \ y_i \in \M \setminus B_D(x, Cr))
& \leq \mathbb{P}(\|\xi_i\|_2\geq (C-1)r) \\
& \leq \frac{C_1}{r}e^{-\frac{C_2}{r^2}} \leq Cr^d,
\end{align*}
where the second-last inequality holds by Chernoff bound, and the last inequality holds since $r = O(\sqrt{\sigma})$ is sufficiently small. Plugging the above bounds into (\ref{probability:I_x}), we obtain
\[
\mathbb{P}(i \in I_{x,r}) \leq Cr^d.
\]
Hence, for any $i \in I_{x,r}$, we have $\|x_i-x\|_1 \leq c_1r$ \YQ{with} probability $\rho = (cr^d)/(Cr^d) < 1$ for being a constant independent on $r$.

Applying the Berry-Esseen theorem to the $|I_{x,r}|$ Bernoulli trials, we conclude that there exists $c_2|I_{x,r}|$ $i'$ in $I_{x,r}$ such that $\|x_i-x\|_2\leq c_1 r$ \YQ{with} probability $1 - C/\sqrt{|I_{x,r}|}$, which proves (i).

To show (ii), we recall Lemma \ref{lma:prob_in_I} that $\mathbb{P}(i \in I_{x,c_1r}) \geq cr^d$. Thus there is a sample among $N$ samples lying in $B_D(x,c_1r)$ \YQ{with} probability
\[
1 - (1 - cr^d)^N = O(Nr^d).
\]
Then, $\tilde{\alpha}(x) \geq (1-c_1^2)^{\beta} := c_0'$ with the same probability.
\end{proof}

\begin{lemma}\label{lma:xi_e_var}
Suppose $\xi \sim N(0, \sigma^2 I_D)$; then we have, for any positive integer $k$:
\begin{itemize}
    \item[(i)] $\mathbb{E}(\|\xi\|_2^k) = C_1\sigma^k$
    \item[(ii)] ${\rm Var}(\|\xi\|_2^k) = C_2 \sigma^{2k} $
    \item[(iii)] $\mathbb{E}[\big(\|\xi\|_2^k- \mathbb{E}(\|\xi\|_2^k) \big)^3] =C_3\sigma^{3k}$
    \item[(iv)] $\|\xi_i\|_2^k$ and $\|\xi_j\|_2^k$ are independent if $\xi_i$ and $\xi_j$ are independent,
    \item[(v)] $\mathbb{E}(\|\xi_i\|_2^s\|\xi_j\|_2^t) = C_4\sigma^{s+t}$
     \item[(vi)] ${\rm Var}(\|\xi_i\|_2^s\|\xi_j\|_2^t) = C_5\sigma^{2(s+t)}$
     \item[(vii)] $\mathbb{E}[\big(\|\xi_i\|_2^s\|\xi_j\|_2^t- \mathbb{E}(\|\xi_i\|_2^s\|\xi_j\|_2^t) \big)^3] = C_6\sigma^{3(s+t)}$
\end{itemize}
where $C_n$, $n = 1, \cdots, 6$ are constants depending on $D$ and $k$.
\end{lemma}
\begin{proof}
Letting the $i$-th element of $\xi$ be denoted by $\xi^{(i)}$, we have the following qualities:
\begin{align*}
 \E(\|\xi\|_2^k) & = \frac{1}{(2\pi \sigma^2)^{\frac{D}{2}}}\int_{-\infty}^{+\infty} \! \cdots \! \int_{-\infty}^{+\infty}
    (\sum_{i=1}^D {\xi^{(i)}}^2)^{k/2} e^{-\frac{\sum_{i=1}^D {\xi^{(i)}}^2}{2\sigma^2}}
    d\xi^{(1)} \! \cdots \! d\xi^{(D)} \\
    & = \frac{1}{(2\pi \sigma^2)^{\frac{D}{2}}} \int_{r=0}^{+\infty}
    r^k e^{-r^2/(2\sigma^2)} S_D(r) dr \\
    & = \frac{2\pi^{\frac{D}{2}}}{(2\pi \sigma^2)^{\frac{D}{2}} \Gamma(\frac{D}{2})} \int_{r=0}^{+\infty}
    r^{D+k-1} e^{-r^2/(2\sigma^2)} dr \\
    & = \frac{\pi^{\frac{D}{2}}}{(2\pi \sigma^2)^{\frac{D}{2}} \Gamma(\frac{D}{2})} \int_{r=0}^{+\infty}
    r^{D+k-2} e^{-r^2/(2\sigma^2)} dr^2 \\
    & = \frac{\pi^{\frac{D}{2}} (2\sigma^2)^{\frac{D+k}{2}}}{(2\pi \sigma^2)^{\frac{D}{2}} \Gamma(\frac{D}{2})} \int_{z=0}^{+\infty}
    (\frac{z}{2\sigma^2})^{\frac{D+k}{2}-1} e^{-z/(2\sigma^2)} d\frac{z}{2\sigma^2} \\
    & = \frac{2^{k/2}\sigma^k}{\Gamma(\frac{D}{2})} \int_{z=0}^{+\infty}
    z^{\frac{D+k}{2}-1} e^{-z} dz = \frac{2^{k/2}\Gamma(\frac{D+k}{2})}{\Gamma(\frac{D}{2})} \sigma^k
\end{align*}
where $\Gamma(t)=\int_{0}^{+\infty} s^{t-1}e^{-s}ds$ is the Gamma function. Plugging the above equality into ${\rm Var}(\|\xi\|_2^k) = \mathbb{E}(\|\xi\|_2^{2k}) - \mathbb{E}(\|\xi\|_2^k)^2$, and
\begin{align*}
\mathbb{E}[\big(\|\xi\|_2^k- \mathbb{E}(\|\xi\|_2^k) \big)^3] = & \mathbb{E}(\|\xi\|_2^{3k}) - 3  \mathbb{E}(\|\xi\|_2^{2k}) \mathbb{E}(\|\xi\|_2^k) \\
& + 3 \mathbb{E}(\|\xi\|_2^k) \mathbb{E}(\|\xi\|_2^k)^2 -  \mathbb{E}(\|\xi\|_2^k)^3,
\end{align*}
will yield the variance and third moment.

To show the independence, we set $F_X$ as the cumulative distribution function of $X$,  $S_t(\zeta) = \{\xi_t:\|\xi_t\|_2^k \leq \zeta \}$ and $\eta_t = \|\xi_t\|_2^k$ with $t = i,j$. Then
\begin{align*}
F_{\eta_i, \eta_j}(\zeta_i, \zeta_j)
& = P(\eta_i \leq \zeta_i, \eta_j \leq \zeta_j) \\
& = P(\xi_i \in S_i(\zeta_i), \xi_j \in S_j(\zeta_j)) \\
& = P(\xi_i \in S_i(\zeta_i))P(\xi_j \in S_j(\zeta_j)) \\
& = P(\eta_i \leq \zeta_i)P(\eta_j \leq \zeta_j) \\
& = F_{\eta_i}(\zeta_i)F_{\eta_j}(\zeta_j),
\end{align*}
which completes the proof of independence by definition. Based on the independence, we obtain
\begin{align*}
\mathbb{E}(\|\xi_i\|_2^s \|\xi_j\|_2^t) = \mathbb{E}(\|\xi_i\|_2^s)\mathbb{E}(\|\xi_i\|_2^s) = (C_1 \sigma^s) \times C_1 (\sigma^t) = C_4 \sigma^{s+t}.
\end{align*}
Plugging the above equality into 
\begin{align*}
{\rm Var}(\|\xi_i\|_2^s \|\xi_j\|_2^t) = \mathbb{E}(\|\xi_i\|_2^{2s} \|\xi_j\|_2^{2t}) - \mathbb{E}(\|\xi_i\|_2^s \|\xi_j\|_2^t)^2
\end{align*}
and
\begin{align*}
\mathbb{E}[\big(\|\xi_i\|_2^s \|\xi_j\|_2^t- \mathbb{E}(\|\xi_i\|_2^s \|\xi_j\|_2^t) \big)^3] = & \mathbb{E}(\|\xi_i\|_2^{3s} \|\xi_j\|_2^{3t}) - 3  \mathbb{E}(\|\xi_i\|_2^{2s} \|\xi_j\|_2^{2t}) \mathbb{E}(\|\xi_i\|_2^s \|\xi_j\|_2^t) \\
& + 3 \mathbb{E}(\|\xi_i\|_2^s \|\xi_j\|_2^t) \mathbb{E}(\|\xi_i\|_2^s \|\xi_j\|_2^t)^2 -  \mathbb{E}(\|\xi_i\|_2^s \|\xi_j\|_2^t)^3,
\end{align*}
will produce the variance and the third moment.
\end{proof}

\begin{prop}\label{prop:xi_e_var}
Suppose $\{\xi_i\}_{i=1}^n$ are i.i.d. drawn from $N(0, \sigma^2I_D)$, $\sum_{i=1}^n \alpha_i = 1$ and $\max_{i \in \{1, \cdots, n\}} \alpha_i \leq C_\alpha/n$ with certain constant $C_\alpha$. For any $\delta$, there exist constants $C$, depending on  $D$, $k$, $\delta$, and $n_1$, depending on $\delta$ and $C_\alpha$ such that if $n \geq n_1$, then
\begin{align*}
    \sum_{i=1}^n \alpha_i \|\xi_i\|_2^k \leq C \sigma^k 
    \quad {\rm and} \quad
     \frac{1}{n^2} \sum_{i=1}^n \sum_{j=1}^n  \|\xi_i\|_2^s\|\xi_j\|_2^t \leq C \sigma^{s+t}
\end{align*}
holds for $k, s, t \leq 4$ \YQ{with} probability at least $1-\delta$.
\end{prop}
\begin{proof} 
By Lemma \ref{lma:xi_e_var}, $\|\xi_1\|_2^k, \cdots, \|\xi_n\|_2^k$ are i.i.d. random variables drawn from a distribution whose expectation is $\mathbb{E}(\|\xi\|_2^k)$ and variance is ${\rm Var}(\|\xi\|_2^k)$.
Using the Berry-Esseen Theorem, the cumulative distribution function of the variable
\[
 (\sum_{i=1}^n \alpha_i \|\xi_i\|_2^k-\mathbb{E}(\|\xi\|_2^k)) \ / \ (C_2^{1/2}\sigma^k\sqrt{\sum_{i=1}^n \alpha_i^2})
\]
denoted by $F_n$ satisfies
\[
|F_n(t) - \Phi(t)| \leq \frac{C_0 \rho \sum_{i=1}^n \alpha_i^3}{\sigma^{3k} (\sum_{i=1}^n \alpha_i^2)^{3/2}} = C_0' \frac{\sum_{i=1}^n \alpha_i^3}{(\sum_{i=1}^n \alpha_i^2)^{3/2}} 
\leq C_0' \frac{\sum_{i=1}^n \alpha_i^3}{\Big(\frac{1}{n}(\sum_{i=1}^n \alpha_i )^2\Big)^{3/2}} =  C_0' n^{\frac{3}{2}} \sum_{i=1}^n \alpha_i^3
\]
where $\Phi$ is the cumulative distribution function of standard normal distribution, $\rho$ is the third moment of $\|\xi\|_2^k$, which is in the order of $\sigma^{3k}$ according to Lemma \ref{lma:xi_e_var}(iii), and the last inequality holds in accordance with Cauchy's inequality.

Since $\alpha_i \leq C_\alpha/n$, we obtain
$\sum_{i=1}^n \alpha_i^3 \leq n (\frac{C_\alpha}{n})^3 = C_\alpha^3 n^{-2}$ and therefore $|F_n(t)-\Phi(t)| \leq C'/\sqrt{n}$. So there exists a constant $C$ depending on $D, k$, and $\delta$ such that
\begin{align*}
    \sum_{i=1}^n \alpha_i \|\xi_i\|_2^k \leq C \sigma^k,
\end{align*}
\YQ{with} probability $1-\frac{\delta}{2}-C'/\sqrt{n}$. Taking $n_1 = \frac{4C'}{\delta^2}$, $\sum_{i=1}^n \alpha_i \|\xi_i\|_2^k \leq C \sigma^k$ \YQ{with} probability at least $1-\delta$ when $n \geq n_1$. Analogously, there exist $C$ and $n_0$ such that
\begin{align*}
    \frac{1}{n^2}\sum_{i=1}^n \sum_{j=1}^n \|\xi_i\|_2^s\|\xi_i\|_2^t \leq C \sigma^{s+t}.
\end{align*}
\YQ{with} probability at least $1-\delta$ when $n  \geq n_1$.
\end{proof}

\begin{proof}{\bf of Lemma \ref{lma:xi}}
\YQ{By (\ref{ineq:bound_r}), we have $r < C_1$.}
For any given $\delta$, let
\[
n_0 = \max \Big\{ \frac{4C^2C_1^{2d}}{\delta^2}, \frac{\max\{n_1, \frac{4C_0^2}{\delta^2}\}}{c'} \Big\},
\]
where $C$ and $c'$ are the two constants in Proposition \ref{prop:samplesize}, $n_1$ is the constant in Proposition \ref{prop:xi_e_var}, and $C_0$ is the constant in Proposition \ref{prop:alpha_bound}. Plugging $N \geq n_0r^{-d}$ into Proposition \ref{prop:samplesize}, we obtain $|I_{x,r}| \geq \max \{n_1, \frac{4C_0^2}{\delta^2}\}$ \YQ{with} probability at least $1-\frac{\delta}{2}$.
Recalling Proposition \ref{prop:alpha_bound} (i) and the definition of $\alpha_i$ in (\ref{def:alpha}), $\alpha_i \leq \frac{1}{\tilde{\alpha}} \leq \frac{C_\alpha}{|I_{x,r}|}$ \YQ{with} probability at least $1-\frac{\delta}{2}$ and $C_\alpha = \frac{1}{c_0}$ since $1 - \frac{C_0}{\sqrt{|I_{x,r}|}} \geq 1-\frac{\delta}{2}$ by $|I_{x,r}| \geq \frac{4C_0^2}{\delta^2}$. As a result, conditions of Proposition \ref{prop:xi_e_var} hold \YQ{with} probability at least $(1-\frac{\delta}{2})^2 \geq 1-\delta$. Using Proposition \ref{prop:samplesize}, we are able to complete the proof.
\end{proof}

\begin{proof}{\bf of Lemma \ref{lma:Pi_ij}}
Corollary 12 of \cite{boissonnat2018reach} shows that
\[
\Big\| \sin \frac{\theta(U_x,U_y)}{2} \Big\|_2 \leq \frac{\|x-y\|_2}{2 {\rm reach}(\M)},
\]
where $U_x$ and $U_y$ are the basis of $T_x\M$ and $T_y\M$, respectively. Letting the orthogonal complements of $U_x$ and $U_y$ be denoted by $V_x$ and $V_y$, respectively, we obtain $\YQ{\Pi_x} = V_xV_x^T$ and $\YQ{\Pi_y} = V_yV_y^T$. Then, in accordance with (ii) of Lemma \ref{lma:Matrix},
\begin{align*}
\|\YQ{\Pi_x} - \YQ{\Pi_y}\|_F
& = \|V_xV_x^T - V_yV_y^T\|_F = \|U_xU_x - U_yU_y\|_F \leq C \Big\| \sin \theta(U_x,U_y) \Big\|_2 \\
& \leq 2C \Big\| \sin \frac{\theta(U_x,U_y)}{2} \Big\|_2 \leq C \frac{\|x-y\|_2}{\tau}.
\end{align*}
\end{proof}

\subsection{Proof of Theorem \ref{thm:first_der_f}} \label{app:first_der_f}

To prove Theorem \ref{thm:first_der_f}, we first introduce two lemmas, namely Lemma \ref{lma:Pi_x_2norm} and Lemma \ref{lma:first_der_alpha}.
\begin{lemma} \label{lma:Pi_x_2norm}
Suppose $d(x,\M) \leq cr$ with some constant $c<1$ and $r = O(\sqrt{\sigma})$. For any given $\delta$, there exist constants  $C$ and $n_0$ such that if $N \geq n_0r^{-d}$, then the following inequalities hold:
\begin{itemize}
    \item[(i)] $\|\big(\|P_{x_i}-\Pi_{x_i^*}\|_2\big)_{i \in I_{x,r}}\|_2 \leq Cr|I_{x,r}|^{\frac{1}{2}}$ \YQ{with} probability $\delta_0(1-\delta)^2$,
    \item[(ii)] $\| \big( \|x_i - x_i^*\|_2 \big)_{I_{x,r}} \|_2 \leq Cr^2|I_{x,r}|^{\frac{1}{2}}$ \YQ{with} probability $1 - \delta$,
    \item[(iii)] $\| \big( \|x_i^* - x^*\|_2 \big)_{I_{x,r}} \|_2 \leq Cr|I_{x,r}|^{\frac{1}{2}}$ \YQ{with} probability $1 - \delta$.
\end{itemize}
\end{lemma}
\begin{proof}
We begin with (i). \YQ{Plugging $z = x_i$ into Theorem \ref{thm:P_z}, we obtain}
\begin{align*}
\Big\|\big(\|P_{x_i}-\Pi_{x_i^*}\|_2\big)_{i \in I_{x,r}} \Big\|_2^2 \leq \sum_{i \in I_{x,r}} (A^2 + 2*AB + B^2 ) {\mbox \ \YQ{with} \ probability} \ \delta_0, 
\end{align*}
where $A =  \frac{C}{r^2} \frac{1}{|I_{x_i,2r}|}  \sum_{j \in I_{x_i,2r}}   \big( \|\xi_j\|_2^4 + \|\xi_j\|_2^3 + \|\xi_j\|_2^2 + r\|\xi_j\|_2 \big) $ and $B =  C(r + \frac{ \|\xi_i \|_2}{r} + \frac{ \|\xi_i\|_2^2}{r^2} $). In accordance with Lemma \ref{lma:xi}, there exist $C$ and $n_0$ such that 
\begin{align*}
\sum_{i \in I_{x,r}} AB 
& = \frac{C|I_{x,r}|}{r} \frac{1}{|I_{x,r}|\times|I_{x_i,2r}|} \sum_{i \in I_{x,r}}\sum_{j \in I_{x_i, 2r}} \big( \|\xi_j\|^4+\|\xi_j\|^3+\|\xi_j\|^2+r\|\xi_j\| \big) \\
& + \frac{C|I_{x,r}|}{r^3} \frac{1}{|I_{x,r}|\times|I_{x_i,2r}|}\sum_{i \in I_{x,r}} \sum_{j \in I_{x_i, 2r}}\big( \|\xi_j\|^4\|\xi_i\|+\|\xi_j\|^3\|\xi_i\|+\|\xi_j\|^2\|\xi_i\|+r\|\xi_j\|\|\xi_i\|  \big) \\
& + \frac{C|I_{x,r}|}{r^4} \frac{1}{|I_{x,r}|\times|I_{x_i,2r}|}\sum_{i \in I_{x,r}} \sum_{j \in I_{x_i, 2r}} \big( \|\xi_j\|^4\|\xi_i\|^2+\|\xi_j\|^3\|\xi_i\|^2+\|\xi_j\|^2\|\xi_i\|^2+r\|\xi_j\|\|\xi_i\|^2  \big)  \\
& \leq C |I_{x,r}| (r^2+r^2+r^3) \leq Cr^2|I_{x,r}|, 
\end{align*}
\begin{align*}
A^2 & = \frac{C}{r^4} \frac{1}{|I_{x_i,2r}|^2} \sum_{j, k \in I_{x_i,2r}}  \big(\|\xi_j\|^4\|\xi_k\|^4+\|\xi_j\|^4\|\xi_k\|^3+  \cdots + r^2\|\xi_j\|\|\xi_k\| \big) \\
& \leq \frac{C}{r^4} \big( \sum_{k=4}^8 \sigma^k + r\sum_{k=3}^5 \sigma^k + r^2 \sigma^2 \big) \leq Cr^2,
\end{align*}
\begin{align*}
\sum_{i \in I_{x,r}} B^2 = \sum_{i \in I_{x,r}} \Big( r^2+ \frac{\|\xi_i\|_2^2}{r^2}+\frac{\|\xi_i\|_2^4}{r^4} + 2\|\xi_i\|_2+2\frac{\|\xi_i\|_2^2}{r}+2\frac{\|\xi_i\|_2^3}{r^3} \Big) \leq Cr^2|I_{x,r}|,\end{align*}
\YQ{with} probability $1-\delta/3$ respectively. The above bounds amount to $\Big\|\big(\|P_{x_i}-\Pi_{x_i^*}\|_2\big)_{i \in I_{x,r}}\Big\|_2^2 \leq Cr^2|I_{x,r}|$, which leads to
\[
\Big\|\big(\|P_{x_i}-\Pi_{x_i^*}\|_2\big)_{i \in I_{x,r}} \Big\|_2 \leq Cr|I_{x,r}|^{\frac{1}{2}},  {\mbox \ with \ probability} \ \delta_0(1-\delta)^2.
\]
As for (ii),
\begin{align*}
\Big\| \big( \|x_i - x_i^*\|_2 \big)_{I_{x,r}} \Big\|_2^2 = \sum_{i \in I_{x,r}} \|x_i - x_i^*\|_2^2 \leq \sum_{i \in I_{x,r}} \|\xi_i \|_2^2 \leq |C |I_{x,r}| \sigma^2,
\end{align*}
\YQ{with} probability $1-\delta$, which implies $\big\| \big( \|x_i - x_i^*\|_2 \big)_{I_{x,r}} \big\|_2\leq C |I_{x,r}|^{\frac{1}{2}} \sigma = Cr^2 |I_{x,r}|^{\frac{1}{2}}$.
We derive (iii) based on
\[
\|x_i^* - x^*\|_2 \leq \|x_i^*-x_i\|_2 + \|x_i-x\|_2 + \|x - x^*\|_2 \leq \|\xi_i\|_2 + 2r.
\]
Thus we have
\begin{align*}
\Big\| \big( \|x_i^*-x^*\|_2 \big)_{I_{x,r}} \Big\|_2^2 \leq  \sum_{i \in I_{x,r}} \big( \|\xi_i \|_2^2 + 4r^2 + 2r \|\xi_i\|_2 \big) \leq C |I_{x,r}|\big( \sigma^2 + 4r^2 + 2 r \sigma\big) \leq Cr^2 |I_{x,r}|
\end{align*}
\YQ{with} probability $1-\delta$, which implies $\Big\| \big( \|x_i^* - x^*\|_2 \big)_{I_{x,r}} \Big\|_2\leq Cr |I_{x,r}|^{\frac{1}{2}}$.
\end{proof}

\begin{lemma}\label{lma:first_der_alpha}
Suppose $d(x,\M) \leq cr$ with some constant $c<1$, $r = O(\sqrt{\sigma})$ and $\beta \geq 2$. For any given $\delta$, there exist constants $C$ and $n_0$ such that if $N \geq n_0r^{-d}$, then
 \[
    \|\big(\partial_v \alpha_i(x)\big)_{i \in I_{x,r}}\|_2 \leq \frac{C}{r} |I_{x,r}|^{-\frac{1}{2}} \ {\mbox with \  probability} \ 1 - \delta.
 \]
\end{lemma}
\begin{proof}
By Lemma \ref{lma:xi}, $\tilde{\alpha}(x) \geq c_0|I_{x,r}|$ \YQ{with} probability at least $1-\delta$. Based on this, we obtain the following inequalities given $0 \leq \tilde{\alpha}_i(x) \leq 1$:
\begin{align*}
\Big\|\big(\partial_v \alpha_i(x)\big)_{i \in I_{x,r}} \Big\|_{2}
& \leq \Big\| \Big( \frac{\partial_v \tilde \alpha_i(x)}{\tilde \alpha(x)} \Big)_{i \in I_{x,r}} \Big\|_{2}
+ \Big\| \Big( \frac{(\partial_v \tilde{\alpha}(x))\tilde{\alpha}_i(x)}{\tilde{\alpha}^2(x)} \Big)_{i \in I_{x,r}} \Big\|_{2} \\
& \leq \frac{C}{r} \Big\|\Big( \frac{\tilde{\alpha}_i(x)^{\frac{\beta-1}{\beta}}}{\tilde{\alpha}(x)}\Big)_{i \in I_{x,r}} \Big\|_{2}
+ \Big|\frac{\partial_v \tilde{\alpha}(x)}{\tilde{\alpha}^2(x)} \Big|
\big\|\big(\tilde{\alpha}_i(x)\big)_{i \in I_{x,r}} \big\|_{2} \\
& \leq \frac{C}{r} \Big\|\Big( \frac{1}{\tilde{\alpha}(x)}\Big)_{i \in I_{x,r}} \Big\|_{2}
 + |\frac{\partial_v \tilde{\alpha}(x)}{\tilde{\alpha}^2(x)}| \|(1)_{i \in I_{x,r}}\|_2\\
& \leq \frac{C}{r}|I_{x,r}|^{-\frac{1}{2}} +
\frac{C}{r} |I_{x,r}|^{\frac{1}{2}} \frac{\sum_{i \in I_{x,r}} \tilde{\alpha}_i(x)^{\frac{\beta-1}{\beta}}}{\tilde{\alpha}^2(x)} \\
& \leq \frac{C}{r}|I_{x,r}|^{-\frac{1}{2}} + \frac{C}{r} |I_{x,r}|^{\frac{1}{2}} \frac{|I_{x,r}|}{|I_{x,r}|^2} \leq \frac{C}{r} |I_{x,r}|^{-\frac{1}{2}}.
\end{align*}
\end{proof}

\begin{proof}{\bf of Theorem\ref{thm:first_der_f}}

We rewrite (\ref{fun:out_ours}) as
\begin{align}\label{eq:equivalent_fx}
    f(x) = \YQ{\Psi_x^\alpha}\sum_{i \in I_{x,r}} \alpha_i(x)(x-x_i),
\end{align}
and calculate the first derivative of $f(x)$ as
\begin{align}
\begin{split} \label{eq:first_der}
    \partial_v f(x)
    & = \sum_{i \in I_{x,r}} \alpha_i(x)\YQ{\Psi_x^\alpha}\big(\partial_v(x-x_i)\big) \\
    & + \sum_{i \in I_{x,r}} \alpha_i(x)(\partial_v \YQ{\Psi_x^\alpha})(x-x_i) \\
    & + \sum_{i \in I_{x,r}}(\partial_v \alpha_i(x))\YQ{\Psi_x^\alpha}(x-x_i).
\end{split}
\end{align}
We deal with the three terms one by one. First,
\[
    \sum_{i \in I_{x,r}} \alpha_i(x)\YQ{\Psi_x^\alpha}\big((\partial_v(x-x_i)\big) = \sum_{i \in I_{x,r}}\alpha_i(x) \YQ{\Psi_x^\alpha} v = \YQ{\Psi_x^\alpha} v.
\]
To bound the second term of (\ref{eq:first_der}), we proceed to bound $\|\partial_v \YQ{\Psi_x^\alpha}\|_2$. In accordance with (26) of \cite{pmlr-v75-fefferman18a}, we establish the relationship between $\|\partial_v\YQ{\Psi_x^\alpha}\|_2$ and $\|\partial_v A_x\|_2$ as follows:
\begin{align*}
\| \partial_v \YQ{\Psi_x^\alpha}\|_2 & \leq 8 \|\partial_v A_x\|_2 \\
& = C \Big\|\sum_i \partial_v \alpha_i(x) \big( (P_{x_i}-\YQ{\Pi_{x_i^*}}) + (\YQ{\Pi_{x_i^*}}-\YQ{\Pi_{x^*}}) \big) + \YQ{\Pi_{x^*}}\big(\partial_v \sum_{i} \alpha_i(x)\big) \Big\|_2 \\
& \leq C\sum_{i} |\partial_v \alpha_i(x)| \|P_{x_i} - \YQ{\Pi_{x_i^*}}\|_2 + \frac{C}{\tau}\sum_{i} |\partial_v \alpha_i(x)|\|x_i^*-x^*\|_2+ 0\\
&
\leq C \Big\|\big(\partial_v \alpha_i(x)\big)_{i \in I_{x,r}} \Big\|_2
\Big\|\big(\|P_{x_i}-\Pi_{x_i^*}\|_2\big)_{i \in I_{x,r}} \Big\|_2 \\
&
 + \frac{C}{\tau} \Big\|\big(\partial_v \alpha_i(x)\big)_{i \in I_{x,r}} \Big\|_2  \Big\|\big(\|x_i^*-x^*\|_2\big)_{i \in I_{x,r}} \Big\|_2  \\
&
\leq C r \Big\|\big(\partial_v \alpha_i(x)\big)_{i \in I_{x,r}} \Big\|_2 |I_{x,r}|^{\frac{1}{2}},
\end{align*}
where 
the second to the last inequality holds by Cauchy-Schwarz inequality, and the last inequality holds by Lemma \ref{lma:Pi_x_2norm} and Lemma \ref{lma:first_der_alpha}. As a result,
\begin{align}\label{bound:partial_Pi_x}
\| \partial_v \YQ{\Psi_x^\alpha}\|_2 \leq 8\|\partial_v A_x\|_2 \leq C.
\end{align}
Therefore, the second term of (\ref{eq:first_der}) is bounded as
\[
\Big\|\sum_{i\in I_{x,r}} \alpha_i(x)(\partial_v \YQ{\Psi_x^\alpha})(x-x_i) \Big\|_2 \leq \sum_{i\in I_{x,r}} \alpha_i(x) \|\partial_v \YQ{\Psi_x^\alpha}\|_2 \|x-x_i\|_2 \leq \sum_{i\in I_{x,r}} \alpha_i(x) C r = Cr.
\]
As for the last term in (\ref{eq:first_der}), we have
\begin{align*}
\Big\|\sum_{i\in I_{x,r}} \partial_v \alpha_i(x) \YQ{\Psi_x^\alpha} (x - x_i) \Big\|_2
& \leq \Big\|\sum_{i\in I_{x,r}} \partial_v \alpha_i(x) \YQ{\Psi_x^\alpha} (x^*-x_i^*) \Big\|_2 \\
& +  \Big\| \big(\YQ{\Psi_x^\alpha} (x-x^*) \big)\sum_{i\in I_{x,r}} \partial_v \alpha_i(x) \Big\|_2 \\
& + \Big\| \sum_{i\in I_{x,r}} \partial_v \alpha_i(x) \YQ{\Psi_x^\alpha}(x_i^*-x_i) \Big\|_2  \\
& = \Big\|\sum_{i\in I_{x,r}} \partial_v \alpha_i(x) \YQ{\Psi_x^\alpha} (x^*-x_i^*) \Big\|_2+  0  \\
& + \Big\| \sum_{i\in I_{x,r}} \partial_v \alpha_i(x) \YQ{\Psi_x^\alpha}(x_i^*-x_i) \Big\|_2,
\end{align*}
where
\begin{align*}
\Big\| \sum_{i \in I_{x,r}} \partial_v \alpha_i(x) \YQ{\Psi_x^\alpha}(x^* - x_i^*) \Big\|_2
& \leq \sum_{i \in I_{x,r}} |\partial_v \alpha_i(x)| \|\YQ{\Psi_x^\alpha}(x_i^*-x^*)\|_2 \\
& \leq \Big\| \big( \partial_v \alpha_i(x) \big)_{i \in I_{x,r}}  \Big\|_{2} \Big\|\big(\|\YQ{\Psi_x^\alpha}(x_i^*-x^*)\|\big)_{i \in I_{x,r}} \Big\|_2
\leq Cr
\end{align*}
and
\begin{align*}
\Big\| \sum_{i\in I_{x,r}} \partial_v \alpha_i(x) \YQ{\Psi_x^\alpha}(x_i^*-x_i) \Big\|_2
& \leq \sum_{i \in I_{x,r}} |\partial_v \alpha_i(x)| \| \YQ{\Psi_x^\alpha}(x_i^*-x_i)\|_2\\
& \leq  \Big\| \big( \partial_v \alpha_i(x) \big)_{i \in I_{x,r}} \Big\|_{2}  \Big\|\big(\|\YQ{\Psi_x^\alpha}(x_i^*-x_i)\|\big)_{i \in I_{x,r}} \Big\|_2 \leq Cr
\end{align*}
based on
\begin{align*}
\|\YQ{\Psi_x^\alpha}(x_i^* - x^*)\|_2 & \leq \|\YQ{\Psi_x^\alpha} - \YQ{\Pi_{x^*}}\|_2\|x_i^*-x^*\|_2 + \|\YQ{\Pi_{x^*}}(x_i^*-x^*)\|_2 \\
& \leq Cr^2  + C \frac{\|x_i^*-x^*\|_2^2}{\tau}\leq C r^2,
\end{align*}
where the second inequality holds in probability via Theorem \ref{thm:bound_Pix} and Proposition \ref{prop:reach}. The above bounds amount to the bound on the first derivative, that is, $\| \partial_v f(x) - \YQ{\Psi_x^\alpha}  v\|_2 \leq C r$.

\YQ{The above proof is based on Lemma \ref{lma:Pi_x_2norm}, Lemma \ref{lma:first_der_alpha} and Theorem \ref{thm:bound_Pix}, which are valid when Lemma \ref{lma:xi} and Theorem \ref{thm:P_z} hold. Hence, the conclusion obtained from the above proof is valid when Lemma \ref{lma:xi} and Theorem \ref{thm:P_z} simultaneously hold, whose probability is at least $\delta_0(1-\delta)^2$.}
\end{proof}

\subsection{Proof of  Theorem \ref{thm:second_der_f}} \label{app:second_der_f}

To prove Theorem \ref{thm:second_der_f}, we first introduce Lemma \ref{lma:der2_alpha}.

\begin{lemma}\label{lma:der2_alpha}
Suppose $d(x,\M) \leq cr$ with some constant $c<1$, $r = O(\sqrt{\sigma})$ and $\beta \geq 2$. For any given $\delta$, there exist constants $C$ and $n_0$ such that if $N \geq n_0r^{-d}$, then
\begin{align*}
\Big\|\big(\partial_v \partial_u \alpha_i(x)\big)_{i \in I_{x,r}} \Big\|_2 \leq \frac{C}{r^2} |I_{x,r}|^{-\frac{1}{2}} \ {\mbox with \  probability} \ 1 - \delta.
\end{align*}
\end{lemma}
\begin{proof}
\begin{align*}
\Big\|\big(\partial_v \partial_u \alpha_i(x) \big)_{i \in I_{x,r}} \Big\|_{2}
& \leq \Big\|\Big(\frac{\partial_v \partial_u \tilde{\alpha}_i(x)}{\tilde{\alpha}(x)} \Big)_{i \in I_{x,r}} \Big\|_{2}
+  \Big\|\Big( \frac{\partial_v\partial_u \tilde{\alpha}(x)}{\tilde{\alpha}^2(x)} \tilde{\alpha}_i(x) \Big)_{i \in I_{x,r} }\Big\|_{2} \\
& + \Big\| \Big( \frac{(\partial_v\tilde{\alpha}_i(x))(\partial_u \tilde{\alpha}(x))}{\tilde{\alpha}^2(x)} \Big)_{i \in I_{x,r}} \Big\|_{2}
+ \Big\| \Big( \frac{(\partial_u\tilde{\alpha}_i(x))(\partial_v \tilde{\alpha}(x))}{\tilde{\alpha}^2(x)} \Big)_{i \in I_{x,r}} \Big\|_{2} \\
& + 2\Big\|\Big( \big(\frac{\partial_v \tilde{\alpha}(x)}{\tilde{\alpha}(x)}\big)\big(\frac{\partial_u \tilde{\alpha}(x)}{\tilde{\alpha}(x)}\big) \big( \frac{\tilde{\alpha}_i(x)}{\tilde{\alpha}(x)} \big) \Big)_{i \in I_{x,r}} \Big\|_{2}.
\end{align*}
We bound these five terms one-by-one using $\tilde{\alpha}(x) \geq c|I_{x,r}|$ which holds \YQ{with} probability $1-\delta$ by Lemma \ref{lma:xi} and $0 \leq \tilde{\alpha}_i(x) \leq 1$. For the first term,
\begin{align*}
    \Big\|\Big(\frac{\partial_v \partial_u \tilde{\alpha}_i(x)}{\tilde{\alpha}(x)} \Big)_{i \in I_{x,r}} \Big\|_{2}
   & \leq \frac{C}{\tilde{\alpha}(x)} \Big\| \big(\tilde{\alpha}_i(x)^{\frac{\beta-2}{\beta}} \frac{\|x-x_i\|_2^2}{r^4} + \tilde{\alpha}_i(x)^{\frac{\beta-1}{\beta}} \frac{|v^Tu|}{r^2}  \big)_{i \in I_{x,r}} \Big\|_2\\
   & \leq  \frac{C}{\tilde{\alpha}(x)} \Big\| \big( \frac{2}{r^2} \big)_{i \in I_{x,r}} \Big\|_2 \leq \frac{C}{r^2} |I_{x,r}|^{-\frac{1}{2}}.
\end{align*}
For the second term,
\begin{align*}
    \Big\|\Big( \frac{\partial_v \partial_u \tilde{\alpha}(x)}{\tilde{\alpha}^2(x)} \tilde{\alpha}_i(x) \Big)_{i \in I_{x,r} } \Big\|_{2}
    & \leq \Big|\frac{\partial_v\partial_u \tilde{\alpha}(x)}{\tilde{\alpha}^2(x)} \Big| \big\|\big(\tilde{\alpha}_i(x)\big)_{i \in I_{x,r}} \big\|_{2}  \\
    & \leq \Big|\frac{\partial_v\partial_u \tilde{\alpha}(x)}{\tilde{\alpha}^2(x)} \Big| \big\|(1)_{i \in I_{x,r}} \big\|_{2} \\
    & \leq  \frac{1}{\tilde{\alpha}(x)}\Big\|\Big(\frac{\partial_v\partial_u \tilde{\alpha}_i(x)}{\tilde{\alpha}(x)} \Big)_{i \in I_{x,r}} \Big\|_{2}  \|(1)_{i \in I_{x,r}}\|_{2}^2 \\
    & \leq \frac{C}{r^2} |I_{x,r}|^{-1} |I_{x,r}|^{-\frac{1}{2}} |I_{x,r}| = \frac{C}{r^2}|I_{x,r}|^{-\frac{1}{2}} .
\end{align*}
The third and fourth terms are similar, where the third term is bounded by
\begin{align*}
    \| \big( \frac{(\partial_v\tilde{\alpha}_i(x))(\partial_u \tilde{\alpha}(x))}{\tilde{\alpha}^2(x)} \big)_{i \in I_{x,r}} \|_{2}
    &\leq |\frac{\partial_u \tilde{\alpha}(x)}{\tilde{\alpha}(x)}| \| \big( \frac{\partial_v \tilde{\alpha}_i(x)}{\tilde{\alpha}(x)} \big)_{i \in I_{x,r}} \|_{2} \\
    & \leq \| \big( \frac{\partial_u \tilde{\alpha}_i(x)}{\tilde{\alpha}(x)} \big)_{i \in I_{x,r}} \|_{2} \|(1)_{i \in I_{x,r}}\|_{2} \| \big( \frac{\partial_v \tilde{\alpha}_i(x)}{\tilde{\alpha}(x)} \big)_{i \in I_{x,r}} \|_{2} \\
    & \leq \frac{C}{r} |I_{x,r}|^{-\frac{1}{2}} |I_{x,r}|^{\frac{1}{2}} \frac{C}{r} |I_{x,r}|^{-\frac{1}{2}} = \frac{C}{r^2}|I_{x,r}|^{-\frac{1}{2}},
\end{align*}
and analogically, the fourth is bounded by
\[
\Big\| \Big( \frac{(\partial_u\tilde{\alpha}_i(x))(\partial_v \tilde{\alpha}(x))}{\tilde{\alpha}^2(x)} \Big)_{i \in I_{x,r}} \Big\|_{2} \leq \frac{C}{r^2}|I_{x,r}|^{-\frac{1}{2}}
\]
Finally, the fifth term:
\begin{align*}
   \Big \|\Big( \big(\frac{\partial_v \tilde{\alpha}(x)}{\tilde{\alpha}(x)}\big)
    \big(\frac{\partial_u \tilde{\alpha}(x)}{\tilde{\alpha}(x)}\big) \big( \frac{\tilde{\alpha}_i(x)}{\tilde{\alpha}(x)} \big) \Big)_{i \in I_{x,r}} \Big\|_{2}
    &= \big(\frac{\partial_v \tilde{\alpha}(x)}{\tilde{\alpha}(x)}\big)
    \big(\frac{\partial_u \tilde{\alpha}(x)}{\tilde{\alpha}(x)}\big) \Big\|\big(\frac{\tilde{\alpha}_i(x)}{\tilde{\alpha}(x)}\big)_{i \in I_{x,r}} \Big\|_{2} \\
    &\leq \frac{C}{r} \times \frac{C}{r} \times |I_{x,r}|^{- \frac{1}{2}}
    = \frac{C}{r^2}|I_{x,r}|^{- \frac{1}{2}}
 \end{align*}
Summing the above five terms up amounts to the proof. 
\end{proof}

\begin{proof}{\bf of Theorem\ref{thm:second_der_f}}
Letting $G(x) = \sum_{i\in I_{x,r}}\alpha_i(x)(x-x_i)$, we obtain the following bound on the second derivative of $f(x)$
\begin{align}
\begin{split}\label{eq:second_dev}
    \big\|\partial_v\big(\partial_u f(x)\big) \big\|_2
    & \leq \big\|(\partial_v \partial_u \YQ{\Psi_x^\alpha})G(x)\big\|_2 + \big\|(\partial_v \YQ{\Psi_x^\alpha})(\partial_u G(x)) \big\|_2 \\
    & + \big\|(\partial_u \YQ{\Psi_x^\alpha})(\partial_v G(x)) \big\|_2+ \big\|\YQ{\Psi_x^\alpha} (\partial_v \partial_u G(x))\big\|_2.
\end{split}
\end{align}
For the first term, we have
\begin{align*}
\| \partial_v \partial_u \YQ{\Psi_x^\alpha} \|_2 & \leq C \big( \| \partial_v A_x\|_2\|\partial_u A_x\|_2 + \|\partial_v \partial_u A_x \|_2 \big) \\
& \leq C +  C\sum_i |\partial_v \partial_u \alpha_i(x)| \big( \|P_{x_i}-\YQ{\Pi_{x_i^*}}\|_2 + \|\YQ{\Pi_{x_i^*}}-\YQ{\Pi_{x^*}}\|_2 \big) \\
& + C\Big\|\YQ{\Pi_{x^*}}\big(\partial_v \partial_u \sum_{i} \alpha_i(x)\big)\Big\|_2 \\
& \leq C
+ C \Big\|\big(\partial_v \partial_u \alpha_i(x)\big)_{i \in I_{x,r}} \Big\|_{2}
\Big\|\big(\|P_{x_i}-\YQ{\Pi_{x_i^*}}\|_2\big)_{i \in I_{x,r}} \Big\|_{2} \\
& +C \Big\|\big(\partial_v \partial_u \alpha_i(x)\big)_{i \in I_{x,r}} \Big\|_{2}
\Big\|\big(\frac{\|x_i^*-x^*\|_2}{\tau}\big)_{i \in I_{x,r}} \Big\|_{2}  + 0 \\
& \leq C+ \Big(\frac{C}{r^2} |I_{x,r}|^{-\frac{1}{2}} \Big) \times \Big(Cr|I_{x,r}|^{\frac{1}{2}} \Big) \leq \frac{C}{r},
\end{align*}
where the second to the last inequality holds by Lemma \ref{lma:Pi_x_2norm} while the last inequality holds by Lemma \ref{lma:der2_alpha}, and therefore
\begin{align}
\|(\partial_v \partial_u \YQ{\Psi_x^\alpha}) G(x)\|_2 \leq \frac{C}{r} \times r = C.
\end{align}
For the second and third terms,
\begin{align*}
\|\partial_v G(x)\|_2
& = \Big\|v + \sum_{i} \partial_v \alpha_i(x) (x_i-x_1) + \big(\sum_{i} \partial_v \alpha_i(x) \big) x_1 \Big\|_2 \\
& \leq 1 + \big\|\big(\partial_v\alpha_i(x)\big)_{i \in I_{x,r}} \big\|_{2} \big\|(2r)_{i \in I_{x,r}}\big\|_{2} \\
& \leq 1 + \Big(\frac{C}{r} |I_{x,r}|^{-\frac{1}{2}} \Big) \times \Big(2r|I_{x,r}|^{\frac{1}{2}} \Big) = 1 + C,
\end{align*}
and  by (\ref{bound:partial_Pi_x}) we obtain
\[
\big\|(\partial_v \YQ{\Psi_x^\alpha}) \big(\partial_u G(x) \big) \big\|_2 \leq C
, \quad \big\| \big(\partial_u G(x) \big)(\partial_v \YQ{\Psi_x^\alpha}) \big\|_2 \leq C.
\]
For the fourth term, we have
\begin{align*}
     & \big\|\YQ{\Psi_x^\alpha} \big(\partial_v \partial_u G(x) \big) \big\|_2 \\
\leq & \Big\|\YQ{\Psi_x^\alpha} \sum_i (\partial_v \partial_u \alpha_i(x)) x_i \Big\|_2 \\
\leq & \|\YQ{\Psi_x^\alpha}\|_2 \sum_i |\partial_v \partial_u \alpha_i(x)| \|x_i-x_i^*\|_2  + \sum_i |\partial_v \partial_u \alpha_i(x)| \|\YQ{\Psi_x^\alpha} (x^*_i-x^*)\|_2  + 0 \\
\leq &
\Big\|\big(\partial_v \partial_u \alpha_i(x) \big)_{i \in I_{x,r}} \Big\|_{2}
\Big( \big\|(\|x_i-x_i^*\|)_{i \in I_{x,r}} \big\|_{2} + \big\|(\|\YQ{\Psi_x^\alpha}(x_i^*-x^*)\|)_{i \in I_{x,r}} \big\|_{2} \Big)  \\
\leq &
C\Big( \frac{1}{r^2}|I_{x,r}|^{-\frac{1}{2}}\Big) \times \Big((\sigma+r^2) |I_{x,r}|^{\frac{1}{2}}\Big) = C
\end{align*}
\YQ{The above proof is based on Lemma \ref{lma:Pi_x_2norm} and Lemma \ref{lma:der2_alpha}, which are valid when Lemma \ref{lma:xi} and Theorem \ref{thm:P_z} simultaneously hold. Hence, $\|\partial_v \partial_u f(x)\|_2 \leq  C$ when Lemma \ref{lma:xi} and Theorem \ref{thm:P_z} simultaneously hold, whose probability is at least $\delta_0(1-\delta)^2$.}
\end{proof}

\subsection{Proof of Proposition \ref{prop:delta_Pixz}, Proposition \ref{prop:eq_fg} and Lemma \ref{bound:der_phi}} \label{app:D}

\begin{proof} {\bf of Proposition \ref{prop:delta_Pixz}}
Considering the function $\phi(d) = (1-\frac{d}{r^2})^\beta$ for $t \geq 0$, whose derivative is $\phi^\prime(d) = \frac{\beta}{r^2}\big(1-\frac{d}{r^2}\big)^{\beta-1}$, we obtain $|\phi^{\prime}(d)| \leq \frac{\beta}{r^2}$. This implies 
\[
|\tilde{\alpha}_i(x) - \tilde{\alpha}_i(z)| \leq \frac{\beta}{r^2}\|x-z\|_2^2 \leq \frac{\beta}{r^2}\epsilon^2 \leq \frac{\alpha(x)}{|I_{x,2r}|^2}r \leq \frac{r}{|I_{x,2r}|},
\]
where the last inequality holds since $\alpha(x) = \sum_{i \in I_{x,r}} \tilde{\alpha}_i(x) \leq \sum_{i \in I_{x,r}} 1 = |I_{x,r}| \leq |I_{x,2r}|$.
For any $z \in B_D(x,\epsilon)$, we have $\|z-x\|_2 \leq r$ and $I_{z,r} \subset I_{x,2r}$. By the definition of $\tilde{\alpha}_i(z)$, we have $\tilde{\alpha}_i(z) = 0$ for $i \notin I_{z,r}$, and therefore
\begin{align*}
\alpha(z) & = \sum_{i \in I_{z,r}} \tilde{\alpha}_i(z) = \sum_{i \in I_{x,2r}} \tilde{\alpha}_i(z) \\
& = \sum_{i \in \I_{x,2r}} \big( \tilde{\alpha}_i(x) + \tilde{\alpha}_i(z) -\tilde{\alpha}_i(x) \big) = \alpha(x) + \sum_{i \in \I_{x,2r}} \big(\tilde{\alpha}_j(z) -\tilde{\alpha}_j(x) \big).
\end{align*}
Plug $\alpha(z)$ into the following denominator,
\begin{align*}
| \alpha_i(z) - \alpha_i(x) | & = \Big|\frac{\tilde{\alpha}_i(z)}{\alpha(z)} - \frac{\tilde{\alpha}_i(x)}{\alpha(x)} \Big| \\
& \leq \max \left\{ \frac{\tilde{\alpha}_i(x) \pm \big| \tilde{\alpha}_i(z) - \tilde{\alpha}_i(x)\big|}{\alpha(x) \mp \sum_{j \in I_{x,2r}} |\tilde{\alpha}_i(z)-\tilde{\alpha}_i(x)| } - \frac{\tilde{\alpha}_i(x)}{\alpha(x)} \right\} \\ 
& \leq \max \left\{ \frac{\tilde{\alpha}_i(x)+\frac{\alpha(x)}{|I_{x,2r}|^2}r}{\alpha(x)-|I_{x,2r}|\frac{\alpha(x)}{|I_{x,2r}|^2}r} - \frac{\tilde{\alpha}_i(x)}{\alpha(x)} ,
\frac{\tilde{\alpha}_i(x)}{\alpha(x)} - \frac{\tilde{\alpha}_i(x)-\frac{\alpha(x)}{|I_{x,2r}|^2}r}{\alpha(x)+|I_{x,2r}|\frac{\alpha(x)}{|I_{x,2r}|^2}r}  \right\} \\
& \leq \max \left\{ \frac{\tilde{\alpha}_i(x)+\frac{r}{|I_{x,2r}|}}{\alpha(x)-\alpha(x)\frac{r}{|I_{x,2r}|}} - \frac{\tilde{\alpha}_i(x)}{\alpha(x)} ,
\frac{\tilde{\alpha}_i(x)}{\alpha(x)} - \frac{\tilde{\alpha}_i(x)-\frac{r}{|I_{x,2r}|}}{\alpha(x)+\alpha(x)\frac{r}{|I_{x,2r}|}}  \right\} \\
& \leq \max \left\{ \frac{ \big(\tilde{\alpha}_i(x)+\frac{r}{|I_{x,2r}|} \big)(1+C\frac{r}{|I_{x,2r}|})-\tilde{\alpha}_i(x)}{\alpha(x)} ,
\frac{\tilde{\alpha}_i(x)-\big(\tilde{\alpha}_i(x)-\frac{r}{|I_{x,2r}|}\big)(1-C\frac{r}{|I_{x,2r}|})}{\alpha(x)}  \right\}  \\
& \leq \frac{\tilde{\alpha}_i(x)r + Cr + Cr^2}{\alpha(x)|I_{x,2r}|} \leq \frac{r+Cr+Cr^2}{\alpha(x)|I_{x,2r}|} \leq \frac{r+Cr+Cr^2}{c_0|I_{x,2r}|} = C'\frac{r}{|I_{x,2r}|},
\end{align*}
the second-to-last inequality holds since $\tilde{\alpha}_i(x) \leq 1$ while the last inequality holds \YQ{with} probability $1-(1-cr^d)^N$ by Proposition \ref{prop:alpha_bound}(ii). 

Based on the upper bound of $|\alpha_i(x) - \alpha_i(z)|$, we obtain
\begin{align*}
\|A_x - A_z\|_2 & = \|\sum_{i \in I_{x,r}} \alpha_i(x) P_{x_i} - \sum_{i \in I_{z,r}} \alpha_i(z) P_{x_i} \|_2 = \|\sum_{i \in I_{x,2r}} \alpha_i(x) P_{x_i} - \sum_{i \in I_{x,2r}} \alpha_i(z) P_{x_i} \|_2 \\
& \leq \sum_{i \in I_{x,2r}} |\alpha_i(x) - \alpha_i(z)| \|P_{x_i}\|_2 \leq \sum_{i \in I_{x,2r}} C'\frac{r}{|I_{x,2r}|} \cdot 1 = C'r.
\end{align*}
Noting $\|\YQ{\Psi_x^\alpha} - A_x\|_2 \leq Cr$ with probability $\delta_0(1-\delta)^2$ by (\ref{bound:AxPix}) in Theorem \ref{thm:bound_Pix}, we have
\[
\|\YQ{\Psi_z^\alpha} - A_z\|_2 \leq \|\YQ{\Psi_x^\alpha} - A_z\|_2 \leq \|\YQ{\Psi_x^\alpha} - A_x\|_2 + \|A_x - A_z\|_2 \leq Cr,
\]
and hence $\|\YQ{\Psi_x^\alpha} - \YQ{\Psi_z^\alpha}\| \leq \|\YQ{\Psi_x^\alpha} - A_z\|_2 + \|A_z - \YQ{\Psi_z^\alpha}\|_2 \leq Cr$ \YQ{with} probability $\delta_0(1-\delta)^2\big(1-(1-cr^d)^N \big)$, which completes this proof.
\end{proof}


\begin{proof}{\bf of Proposition \ref{prop:eq_fg}}
\YQ{We first prove (i). Since the rows of $J_f(x)$ are orthogonal to the contour surface at $x$, as the basis of the spanning space of $J_f(x)^T$, $W_x$ is also the basis of the normal space of $\M_{\rm out}$ at $x$ and thereby $W_x \in \R^{D \times (D-d)}$ by Theorem \ref{thm:manifold}. This implies 
\begin{align}\label{prob_g}
\mathbb{P}\big({\rm statement \ (i) \  holds} \big) =  \mathbb{P}\big( W_x \in  \R^{D \times (D-d)} \big) \geq  \mathbb{P}\big({\rm Theorem \   \ref{thm:manifold} \ holds} \big)
\end{align}
}


\YQ{Now we proceed to prove (ii).} It is clear that $g(z) = \bf{0}$ if $f(z) = \bf{0}$. Thus, we only need to prove that $g(z) = \bf{0}$ implies $f(z) = \bf{0}$. To do this, we first assume the reverse, $f(z)\neq 0$  and $g(z) = W_x^T f(z) = \bf{0}$. Since $W_x^T$ is the basis of ${\rm span}\big(J_f(x)^T \big)$, $J_f(x)$ can be rewritten as $J_f(x) = YW_x^T$ and $J_f(x) f(z) = Y \big(W_x^T f(z) \big) = Yg(z) = \bf{0}$. By the definition of $f(z)$ in equality (\ref{fun:out_ours}), $\YQ{\Psi_z^\alpha} f(z) = f(z)$. Hence, we obtain
\begin{align*}
\|J_f(x) - \YQ{\Psi_z^\alpha}\|_2  = \max_{v \neq 0} \frac{\big\| \big( J_f(x) - \YQ{\Psi_z^\alpha} \big) v \big\|_2 }{\|v\|_2} \geq \frac{\big\| \big( J_f(x) - \YQ{\Psi_z^\alpha} \big) f(z) \big\|_2 }{\|f(z)\|_2} = \frac{\| 0 - f(z)\|_2 }{\|f(z)\|_2} = 1.
\end{align*}
However,
\begin{align*}
\|J_f(x) - \YQ{\Psi_z^\alpha}\|_2 \leq \|J_f(x)-\YQ{\Psi_x^\alpha}\|_2 + \|\YQ{\Psi_x^\alpha} - \YQ{\Pi_{x^*}}\|_F + \|\YQ{\Pi_{x^*}}-\YQ{\Pi_{z^*}}\|_F + \|\YQ{\Pi_{z^*}}-\YQ{\Psi_z^\alpha}\|_F \leq Cr
\end{align*}
where the first term is bounded by (\ref{equ:bound_J_Pix}) \YQ{in Corollary \ref{coro:Jf_Phix}}, the second and fourth terms are bounded by \YQ{ applying Theorem \ref{thm:bound_Pix} for $x$ and $z$, respectively}, and the third term is bounded by Lemma \ref{lma:Pi_ij}. We conduct contradictory bounds of $\|J_f(x) - \YQ{\Psi_z^\alpha}\|_2$. 
Hence, $f(z) = \bf{0}$ if $g(z) = \bf{0}$. \YQ{The statement (ii) is proved when Corollary \ref{coro:Jf_Phix}, Theorem \ref{thm:bound_Pix} and Lemma \ref{lma:Pi_ij} hold simultaneously. Noticing that Corollary \ref{coro:Jf_Phix} holds when  Lemma \ref{lma:xi} and Theorem \ref{thm:P_z} hold, Theorem \ref{thm:manifold} holds when Theorem \ref{thm:bound_Pix} and Proposition \ref{prop:alpha_bound}(ii) hold, and Theorem \ref{thm:bound_Pix} holds when Lemma \ref{lma:xi} and Theorem \ref{thm:P_z} hold, we obtain
\begin{align*}
 \mathbb{P}
& \big({\rm statement \ (i) \ and \ (ii) \ hold} \big) \\
& \geq  \mathbb{P} \Big(\big({\rm Theorem \ \ref{thm:manifold} \  and \ Corollary \ \ref{coro:Jf_Phix} \  hold \ for \ } x \big)  \cap \big({\rm \ Theorem \ \ref{thm:bound_Pix} \ holds \ for \ } x, z \big) \Big) \\
& \geq  \mathbb{P} \big({\rm Proposition \ \ref{prop:alpha_bound}(ii)  \  holds \ for \ } x \big)  \mathbb{P}\big({\rm Lemma \ \ref{lma:xi}  \ and \ Theorem \ \ref{thm:P_z} \  hold \ for \ } x, z \big) \\
& \geq  \delta_0^2(1-\delta)^4\big(1-(1-cr^d)^N\big).
\end{align*}
}
The proof is, therefore, complete.
\end{proof}


\begin{prop} \label{prop:sig_Jf}
Letting $\sigma_1 \geq \cdots \geq \sigma_{D}$ be the singular values of $J_f(x)$, then \YQ{with} probability at least $\delta_0(1-\delta)^2$,
\[
1+O(r) \geq \sigma_1 \geq \sigma_{D-d} \geq 1- O(r).
\]
\end{prop}
\begin{proof}
Let $\YQ{\Psi_x^\alpha} = V_xV_x^T$ and $J_f(x) = U_x \Sigma_x W_x^T$ be the thin singular value decomposition of $J_f(x)$, where $U_x, W_x \in \R^{D \times (D-d)} $ and $\Sigma_x \in \R^{(D-d) \times (D-d)}$ by Theorem \ref{thm:manifold}.

To begin with, we bound $\sigma_{D-d}$ below. Let $S_1 = {\rm span}(V_x)$ and $S_2 = {\rm span}\{w_1, \cdots w_{D-d-1} \}$, where  $w_1, \cdots w_{D-d-1}$ are the first $(D-d-1)$ columns of $W_x$. Since ${\rm dim}(S_1) > {\rm dim}(S_2)$, there exists $\eta \neq 0 \in S_1 \cap S_2^{\bot}$, which implies $\YQ{\Psi_x^\alpha} \eta = \eta$ and $w_i^T\eta = 0$ for $i = 1, \cdots, D-d-1$. Hence,
\begin{align*}
\big( \YQ{\Psi_x^\alpha} - J_f(x) \big) \eta = \eta - U_x \Sigma_x W_x^T \eta = \eta - u_{D-d}\sigma_{D-d} w_{D-d}^T \eta,
\end{align*}
where $u_{D-d}$ is the $(D-d)$-th column of $U_x$. This leads to 
\begin{align*}
\| \big( \YQ{\Psi_x^\alpha} - J_f(x) \big) \eta \|_2 & = \| \eta - u_{D-d}\sigma_{D-d} w_{D-d}^T \eta \|_2 \\
& \geq \Big|  \| \eta \|_2 - \| u_{D-d}\sigma_{D-d} w_{D-d}^T \eta \|_2 \Big| = |1 - \sigma_{D-d}| \|\eta\|_2.
\end{align*}
We obtain
\begin{align*}
Cr \geq \| \YQ{\Psi_x^\alpha} - J_f(x) \|_2 \geq \frac{\Big\| \big( \YQ{\Psi_x^\alpha} - J_f(x) \big) \eta \Big\|_2}{ \|\eta\|_2 } =  |1 - \sigma_{D-d}|,
\end{align*}
\YQ{where the first inequality holds by (\ref{equ:bound_J_Pix}) in Corollary \ref{coro:Jf_Phix}.} So, $\sigma_{D-d} \geq 1 - O(r)$.

Now, we turn to the upper bound of $\sigma_1$. Let $\eta = w_1$, then $\|\eta\|_2 = 1$ and $w_i^T\eta = 0$ for any $i \geq 2$. Hence,
\begin{align*}
\big(\YQ{\Psi_x^\alpha} - J_f(x) \big) \eta = \YQ{\Psi_x^\alpha} \eta - U_x\Sigma_xW_x^T \eta = \YQ{\Psi_x^\alpha} \eta - \sigma_1 u_1.
\end{align*}
This leads to 
\begin{align*}
Cr \geq \|\YQ{\Psi_x^\alpha} - J_f(x)\|_2 \geq \big\| \big( \YQ{\Psi_x^\alpha} - J_f(x) \big) \eta \big\|_2 = \| \YQ{\Psi_x^\alpha} \eta - \sigma_1 u_1 \|_2 \geq \big | \|\YQ{\Psi_x^\alpha} \eta\|_2 - \sigma_1 \big|.
\end{align*}
So, $\sigma_1 \leq \|\YQ{\Psi_x^\alpha} \eta\|_2 + Cr \leq 1 + Cr$. \YQ{Note that the above proof relies on Theorem \ref{thm:manifold} and Corollary \ref{coro:Jf_Phix}, where Theorem \ref{thm:manifold} and Corollary \ref{coro:Jf_Phix} hold when Lemma \ref{lma:xi}, Theorem \ref{thm:P_z} and Proposition \ref{prop:alpha_bound}(ii) hold. This proof is completed with probability at least $\delta_0(1-\delta)^2\big(1-(1-cr^d)^N \big)$.}
\end{proof}

\begin{prop}\label{prop:eq_VW}
$\| W_x^T \YQ{\Psi_x^\alpha} W_x^T - I_{D-d} \|_2 \leq Cr$ \YQ{with} probability at least $\delta_0(1-\delta)^2\big(1-(1-cr^d)^N \big)$.
\end{prop}
\begin{proof}
\YQ{By Theorem \ref{thm:manifold}, we obtain $W_x \in \R^{D \times (D-d)} $.}
Let the singular value decomposition of $W_x^TV_x = \sum_{i=1}^{D-d} s_i a_i b_i^T$, where $s_i$ is the $i$-th singular value of $W_x^TV_x$ and $a_i$ and $b_i$ are the singular vectors corresponding to $s_i$. Let $\eta = V_x b_{D-d}$, then
\begin{align*}
Cr  \geq \|\YQ{\Psi_x^\alpha} - J_f(x)\|_2 
& \geq \big\| \big( \YQ{\Psi_x^\alpha} - J_f(x) \big)\eta\|_2 = \| V_x b_{D-d} - U_x \Sigma_x W_x^TV_x b_{D-d} \|_2 \\
& \big| 1 - \|U_x \Sigma_x (\sum_{i=1}^{D-d} s_ia_ib_i^T) b_{D-d} \|_2 \big| = \big| 1 - s_{D-d}\|\Sigma_x a_{D-d}\|_2 \big|,
\end{align*}
\YQ{where the first inequality holds by (\ref{equ:bound_J_Pix}) in Corollary \ref{coro:Jf_Phix}.} This leads to
\begin{align*}
\frac{1-Cr}{\|\Sigma_x a_{D-d}\|_2 } \leq s_{D-d} \leq \frac{1+Cr}{\|\Sigma_x a_{D-d} \|_2}.
\end{align*}
Noticing $1-O(r) \leq \|\sigma_x a_{D-d} \|_2 \leq 1+O(r)$ by Proposition \ref{prop:sig_Jf}, we conclude $1 - O(r) \leq  s_{D-d} \leq s_1 \leq 1$ since $\|W_x^TV_x\|_2 \leq 1$. So,
\begin{align*}
\|W_x^T\YQ{\Psi_x^\alpha} W_x - I_{D-d}\|_2 & = \|W_x^TV_x V_x^T W_x - I_{D-d}\|_2 \\
& = \| ASS^TA^T - AA^T\|_2 = \|A(SS^T - I_{D-d})A^T\|_2 \leq Cr,
\end{align*}
where $A = [a_1, \cdots, a_{D-d}]$ and $S$ is a diagonal matrix with $(s_1, \cdots, s_{D-d})$ as the diagonal entries. \YQ{Note that the above proof relies on Theorem \ref{thm:manifold} and Corollary \ref{coro:Jf_Phix}, where Theorem \ref{thm:manifold} and Corollary \ref{coro:Jf_Phix} hold when Lemma \ref{lma:xi}, Theorem \ref{thm:P_z} and Proposition \ref{prop:alpha_bound}(ii) hold. This proof is completed with probability at least $\delta_0(1-\delta)^2\big(1-(1-cr^d)^N \big)$.}
\end{proof}

\begin{proof}{\bf of Lemma \ref{bound:der_phi}}
Under the settings that the first $d$ coordinates are the basis of $T_x \M_{\rm out}$, and the last $D-d$ coordinates are the columns of $W_x$, $W_x$ can be rewritten as $W_x = ({\bf 0}, I_{D-d})^T$. Hence, we obtain
\begin{align*}
 J_g(z) & ({\bf 0}, I_{D-d})^T  = W_x^T J_f(z) W_x \\
= & W_x^T \big(J_f(z) - J_f(x) \big)W_x + W_x^T \big(J_f(x) - \YQ{\Psi_x^\alpha} \big)W_x + \big( W_x^T \YQ{\Psi_x^\alpha} W_x - I_{D-d} \big) + I_{D-d}.
\end{align*}
This leads to
\begin{align*}
\|J_g(z) & ({\bf 0}, I_{D-d})^T - I_{D-d}\|_2 \\
& \leq \|J_f(z) - J_f(x)\|_2 + \|J_f(x) - \YQ{\Psi_x^\alpha} \|_2 + \|W_x^T \YQ{\Psi_x^\alpha} W_x - I_{D-d}\|_2 \\
& \leq  C_1r + C_2r + C_3r \leq Cr,
\end{align*}
\YQ{where $ \|J_f(z) - J_f(x)\|_2\leq  \|J_f(z) - J_f(x)\|_F \leq C_1r$ by ( \ref{bound:Jfx_Jfz}) in Theorem \ref{thm:Hdist}, $\|J_f(x) - \YQ{\Psi_x^\alpha} \|_2 \leq \|J_f(x) - \YQ{\Psi_x^\alpha} \|_F \leq C_2 r$ by Corollary \ref{coro:Jf_Phix} and $\|W_x^T \YQ{\Psi_x^\alpha} W_x - I_{D-d}\|_2 \leq C_3 r$ by Proposition \ref{prop:eq_VW}.} Using Theorem 2.9.10 (the implicit function theorem) in \cite{hubbard2001vector}, $\phi$ exits. Carrying out the first derivative on $g\big( \zeta, \phi(\zeta) \big) = \bf{0}$, we obtain
\begin{align*}
    {\bf 0} & = \partial_s g(\zeta, \phi(\zeta))
    = J_g(\zeta, \phi(\zeta))
    \left(
    \begin{array}{c}
         \partial_s \zeta  \\
         \partial_s \phi(\zeta)
    \end{array}
    \right) \\
    & = W_x^T\Big( J_f(\zeta, \phi(\zeta)) - J_f(x) \Big)
    \left(
    \begin{array}{c}
         \partial_s \zeta  \\
         \partial_s \phi(\zeta)
    \end{array}
    \right)
    + W_x^TJ_f(x) 
    \left(
    \begin{array}{c}
         \partial_s \zeta  \\
         \partial_s \phi(\zeta)
    \end{array}
    \right) \\
    & = W_x^T\Big( J_f(\zeta, \phi(\zeta)) - J_f(x) \Big)
    \left(
    \begin{array}{c}
         \partial_s \zeta  \\
         \partial_s \phi(\zeta)
    \end{array}
    \right)
    + W_x^TU_x \Sigma_x ({\bf 0}, I_{D-d}) 
    \left(
    \begin{array}{c}
         \partial_s \zeta  \\
         \partial_s \phi(\zeta)
    \end{array}
    \right).
\end{align*}
This implies that
\begin{align*}
\partial_s \phi(\zeta) = -\Sigma_x^{-1}\big( W_x^TU_x \big)^{-1} W_x^T\Big( J_f(\zeta, \phi(\zeta)) - J_f(x) \Big) \left(
    \begin{array}{c}
         \partial_s \zeta  \\
         \partial_s \phi(\zeta)
    \end{array}
    \right).
\end{align*}
Calculating $\ell_2$-norm of the two sides of the above equality, we obtain
\begin{align*}
\|\partial_s \phi(\zeta) \|_2 
& = \Big\| \Sigma_x^{-1}\big( W_x^TU_x \big)^{-1} W_x^T\Big( J_f(\zeta, \phi(\zeta)) - J_f(x) \Big) 
\left(
    \begin{array}{c}
         \partial_s \zeta  \\
         \partial_s \phi(\zeta)
    \end{array}
    \right)
 \Big\|_2  \\
& \leq \big(1+O(r) \big) C \Big\| J_f(\zeta, \phi(\zeta)) - J_f(x) \Big\|_2  \leq C \|(\zeta, \phi(\zeta))-x\|_2
\end{align*}

Carrying out the second derivative on $g(\zeta, \phi(\zeta)) = \bf{0}$, we obtain
\begin{align*}
    {\bf 0} = \partial_t J_g(\zeta, \phi(\zeta))
    \left(
    \begin{array}{c}
         \partial_s \zeta  \\
         \partial_s \phi(\zeta)
    \end{array}
    \right)
    + J_g(\zeta, \phi(\zeta))
    \left(
    \begin{array}{c}
         {\bf 0}  \\
         \partial_t \partial_s \phi(\zeta)
    \end{array}
    \right).
\end{align*}
Letting $e_i$ denote the $i$-th column of $I_D$ and
\begin{align*}
    u = \left(\begin{array}{c} \partial_t \zeta \\ \partial_t \phi(\zeta) \end{array}\right),
\end{align*}
the $i$-th column of $\partial_t J_g(\zeta,\phi(\zeta))$ is
\begin{align*}
    \partial_t \partial_{e_i}g(\zeta,\phi(\zeta))
     = \|u\|_2 \partial_{\frac{u}{\|u\|_2}} \partial_{e_i} g(\zeta,\phi(\zeta))
     = \|u\|_2 W_x^T \partial_{\frac{u}{\|u\|_2}} \partial_{e_i} f(\zeta,\phi(\zeta)).
\end{align*}
In conjunction with $\|\partial_{\frac{u}{\|u\|_2}} \partial_{e_i} f(\zeta,\phi(\zeta))\|_2 \leq C$, as proved in Theorem \ref{thm:second_der_f}, $\|\partial_t \partial_{e_i}g(\zeta,\phi(\zeta))\|\leq C$, and therefore
\[
\|\partial_t J_g(\zeta,\phi(\zeta)\|_2 \leq C.
\]
Hence,
\begin{align*}
\partial_t\partial_s \phi(\zeta) = - \Sigma_x^{-1}\big( W_x^TU_x \big)^{-1} \left( \partial_t J_g(\zeta, \phi(\zeta)) 
\left(
    \begin{array}{c}
         \partial_s \zeta  \\
         \partial_s \phi(\zeta)
    \end{array}
    \right) + W_x^T\Big( J_f(\zeta, \phi(\zeta))-J_f(x)  \Big) \left(
    \begin{array}{c}
         \partial_s \zeta  \\
         \partial_s \phi(\zeta)
    \end{array}
    \right) \right),
\end{align*}
which implies
\begin{align*}
    \|\partial_t\partial_s \phi(\zeta)\|_2 \leq C(1+O(r)) \big( C + C\|z-x\|_2 \big) \leq C.
\end{align*}
\YQ{Note that the above proof relies on Theorem \ref{thm:Hdist}, Corollary \ref{coro:Jf_Phix}, Proposition \ref{prop:eq_VW} and Theorem \ref{thm:second_der_f}, which are valid when Lemma \ref{lma:xi}, Theorem \ref{thm:P_z} and Proposition \ref{prop:alpha_bound}(ii) hold. Hence, this proof is completed with probability at least $\delta_0(1-\delta)^2\big(1-(1-cr^d)^N \big)$.}
\end{proof}

\section{Gradient of $\|f(x)\|_2^2$}\label{app:gradient}

Let $F(x) = \|f(x)\|_2^2$, $d \cdot$ denote the differential and $G(x) = x - \sum_{i \in I_{x,r}} \alpha_i(x)x_i$, then
\begin{align*}
    d F(x) & = 2 \langle f(x), df(x) \rangle = 2 \big\langle \YQ{\Psi_x^\alpha} G(x), d\big(\YQ{\Psi_x^\alpha} G(x) \big)  \big\rangle  \\
    & = 2 \langle \YQ{\Psi_x^\alpha} G(x)G(x)^T, d\YQ{\Psi_x^\alpha} \rangle + 2 \langle \YQ{\Psi_x^\alpha} G(x), dG(x) \rangle  \\
    & = 2 \langle \YQ{\Psi_x^\alpha} G(x)G(x)^T, d\YQ{\Psi_x^\alpha} \rangle + 2\langle \YQ{\Psi_x^\alpha} G(x), dx-\sum_{i \in I_{x,r}} \big(d\alpha_i(x) x_i \big) \rangle,
\end{align*}
where 
\begin{align*}
 d\alpha_i(x)&  = \frac{d\tilde \alpha_i(x)}{\alpha(x)} - \frac{\tilde \alpha_i(x) d\alpha(x)}{\alpha(x)^2} = \frac{d\tilde \alpha_i(x)}{\alpha(x)} - \frac{\alpha_i(x) d\alpha(x)}{\alpha(x)} \\
    & = -\frac{2(d+2)}{r^2\alpha(x)}\langle \tilde \alpha_i(x)^{\frac{d+1}{d+2}}(x-x_i)-\alpha_i(x) \sum_{i}\tilde \alpha_i(x)^{\frac{d+1}{d+2}}(x-x_i), dx \rangle \\
    &:= \langle \frac{d \alpha_i(x)}{d x}, dx \rangle
\end{align*}
and $d\YQ{\Psi_x^\alpha}$ can be calculated as below. Let $\lambda_1 \geq \cdots \geq \lambda_n$ be the eigenvalues of $A_x$ and $\mu_1 > \cdots > \mu_{s}$ are the different values of $\{\lambda_i\}$. Suppose $\lambda_{n-d} > \lambda_{n-d+1}$, and $\mu_1 > \cdots > \mu_t$ are the different values of $\lambda_1 \geq \cdots \geq \lambda_{n-d}$. $P_{i,x}=V_{i,x}V_{i,x}^T$ is an orthogonal projection and columns of $V_{i,x}$ are the eigenvectors corresponding to $\mu_i$. Then, we have $\YQ{\Psi_x^\alpha} = \sum_{i=1}^t P_{i,x}$. By \cite{Shapiro1995}, 
\[
dP_{i,x} = \sum_{j=1}^s \frac{1}{\mu_j-\mu_i} P_{i,x} (d A_x) P_{j,x} + P_{j,x} (dA_x) P_{i,x},
\]
and thereby
\begin{align*}
d\YQ{\Psi_x^\alpha} & = \sum_{i=1}^t dP_{i,x} = \sum_{i=1}^t \sum_{j=t+1}^s \frac{1}{\mu_j-\mu_i} P_{i,x} (d A_x) P_{j,x}  + P_{j,x} (dA_x) P_{i,x} 
\end{align*}
Plug $d\YQ{\Psi_x^\alpha}$ into the first term of $dF(x)$,
\begin{align*}
\big\langle \YQ{\Psi_x^\alpha} G(x)G(x)^T, d\YQ{\Psi_x^\alpha} \big\rangle = \langle T, dA_x \rangle = \sum_{i \in I_{x,r}} \langle T, P_{x_i} \rangle  \langle\frac{d\alpha_i(x)}{dx}, dx \rangle ,
\end{align*}
where $T = \sum_{i=1}^t \sum_{j=t+1}^s \frac{1}{\mu_j-\mu_i} P_{i,x} \big( \YQ{\Psi_x^\alpha} G(x)G(x)^T \big) P_{j,x}  + P_{j,x} \big( \YQ{\Psi_x^\alpha} G(x)G(x)^T \big) P_{i,x}$. Plugging $d\alpha_i(x)$ into the second term of $dF(x)$, we obtain
\begin{align*}
\langle \YQ{\Psi_x^\alpha} G(x), dx-\sum_{i \in I_{x,r}} \big(d\alpha_i(x) x_i \big) \rangle
= \langle \YQ{\Psi_x^\alpha} G(x), dx\rangle - \sum_{i \in I_{x,r}} \langle \YQ{\Psi_x^\alpha} G(x), x_i\rangle \langle\frac{d\alpha_i(x)}{dx}, dx \rangle.
\end{align*}

As the summation of the first and second term, 
\begin{align*}
dF(x) = \Big\langle 2 \sum_{i \in I_{x,r}} \big( \langle T, P_{x_i}\rangle+\langle \YQ{\Psi_x^\alpha} G(x), x_i\rangle \big)\frac{d\alpha_i(x)}{dx} + \YQ{\Psi_x^\alpha} G(x), dx \Big\rangle
\end{align*}
So the gradient of $F(x)$ is
\begin{align} \label{eq:gradF}
{\rm grad}(x) = 2 \sum_{i \in I_{x,r}} \big( \langle T, P_{x_i}\rangle+\langle \YQ{\Psi_x^\alpha} G(x), x_i\rangle \big)\frac{d\alpha_i(x)}{dx} + \YQ{\Psi_x^\alpha} G(x).
\end{align}

\section{Results of Facial Image Denoising}
\label{app:face}
\begin{figure}[htbp]
\centering
\includegraphics[width=0.18\textwidth]{clear_rate03_id2}
\includegraphics[width=0.18\textwidth]{clear_rate03_id1}
\includegraphics[width=0.18\textwidth]{clear_rate03_id3}
\includegraphics[width=0.18\textwidth]{clear_rate03_id4}
\includegraphics[width=0.18\textwidth]{clear_rate03_id5}
\centering
\includegraphics[width=0.18\textwidth]{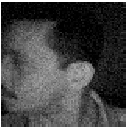}
\includegraphics[width=0.18\textwidth]{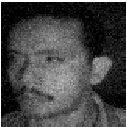}
\includegraphics[width=0.18\textwidth]{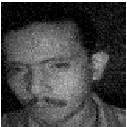}
\includegraphics[width=0.18\textwidth]{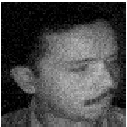}
\includegraphics[width=0.18\textwidth]{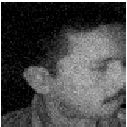}
\centering
\includegraphics[width=0.18\textwidth]{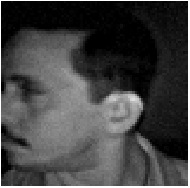}
\includegraphics[width=0.18\textwidth]{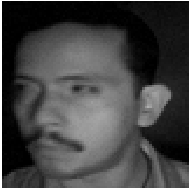}
\includegraphics[width=0.18\textwidth]{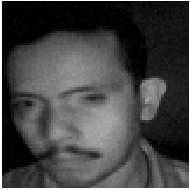}
\includegraphics[width=0.18\textwidth]{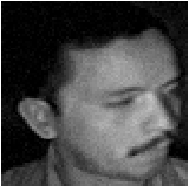}
\includegraphics[width=0.18\textwidth]{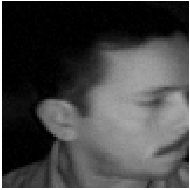}
\centering
\includegraphics[width=0.18\textwidth]{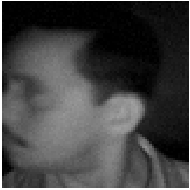}
\includegraphics[width=0.18\textwidth]{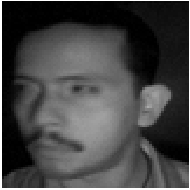}
\includegraphics[width=0.18\textwidth]{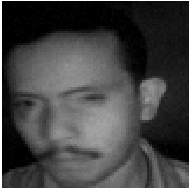}
\includegraphics[width=0.18\textwidth]{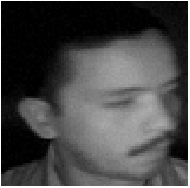}
\includegraphics[width=0.18\textwidth]{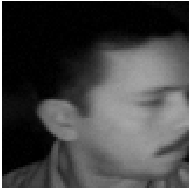}
\centering
\includegraphics[width=0.18\textwidth]{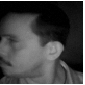}
\includegraphics[width=0.18\textwidth]{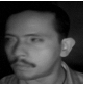}
\includegraphics[width=0.18\textwidth]{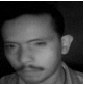}
\includegraphics[width=0.18\textwidth]{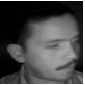}
\includegraphics[width=0.18\textwidth]{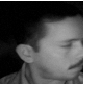}
\centering
\includegraphics[width=0.18\textwidth]{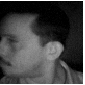}
\includegraphics[width=0.18\textwidth]{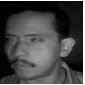}
\includegraphics[width=0.18\textwidth]{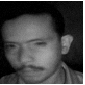}
\includegraphics[width=0.18\textwidth]{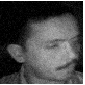}
\includegraphics[width=0.18\textwidth]{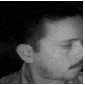}
\centering
\includegraphics[width=0.18\textwidth]{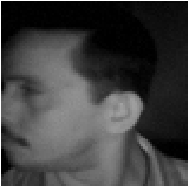}
\includegraphics[width=0.18\textwidth]{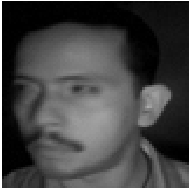}
\includegraphics[width=0.18\textwidth]{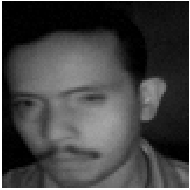}
\includegraphics[width=0.18\textwidth]{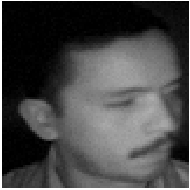}
\includegraphics[width=0.18\textwidth]{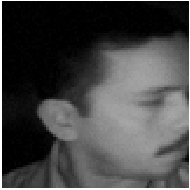}
\caption{Performance of facial image denoising with $\rho = 0.2$. The first row consists of original images while the second row features blurred images. The third to seventh rows contain deblurred images using km17, cf18, ya21(deg=1), ya21(deg=2) and our method, respectively. } \label{fig:face_0.2}
\end{figure}

\begin{figure}[htbp]
\centering
\includegraphics[width=0.18\textwidth]{clear_rate03_id2}
\includegraphics[width=0.18\textwidth]{clear_rate03_id1}
\includegraphics[width=0.18\textwidth]{clear_rate03_id3}
\includegraphics[width=0.18\textwidth]{clear_rate03_id4}
\includegraphics[width=0.18\textwidth]{clear_rate03_id5}
\centering
\includegraphics[width=0.18\textwidth]{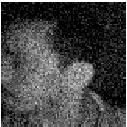}
\includegraphics[width=0.18\textwidth]{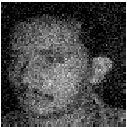}
\includegraphics[width=0.18\textwidth]{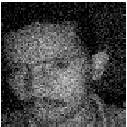}
\includegraphics[width=0.18\textwidth]{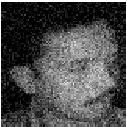}
\includegraphics[width=0.18\textwidth]{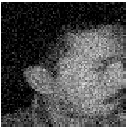}
\centering
\includegraphics[width=0.18\textwidth]{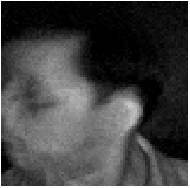}
\includegraphics[width=0.18\textwidth]{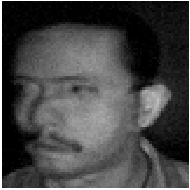}
\includegraphics[width=0.18\textwidth]{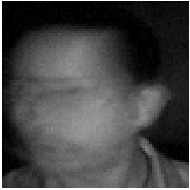}
\includegraphics[width=0.18\textwidth]{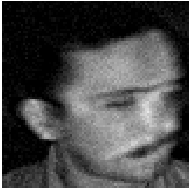}
\includegraphics[width=0.18\textwidth]{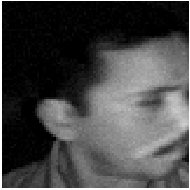}
\centering
\includegraphics[width=0.18\textwidth]{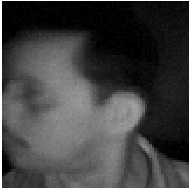}
\includegraphics[width=0.18\textwidth]{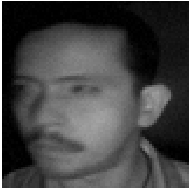}
\includegraphics[width=0.18\textwidth]{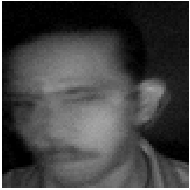}
\includegraphics[width=0.18\textwidth]{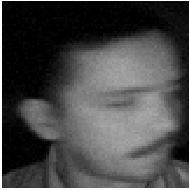}
\includegraphics[width=0.18\textwidth]{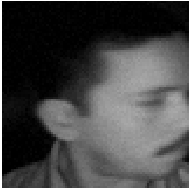}
\centering
\includegraphics[width=0.18\textwidth]{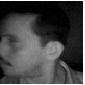}
\includegraphics[width=0.18\textwidth]{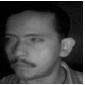}
\includegraphics[width=0.18\textwidth]{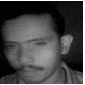}
\includegraphics[width=0.18\textwidth]{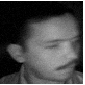}
\includegraphics[width=0.18\textwidth]{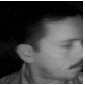}
\centering
\includegraphics[width=0.18\textwidth]{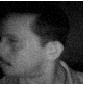}
\includegraphics[width=0.18\textwidth]{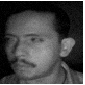}
\includegraphics[width=0.18\textwidth]{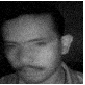}
\includegraphics[width=0.18\textwidth]{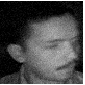}
\includegraphics[width=0.18\textwidth]{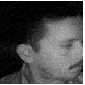}
\centering
\includegraphics[width=0.18\textwidth]{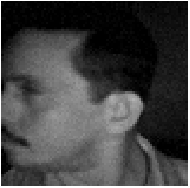}
\includegraphics[width=0.18\textwidth]{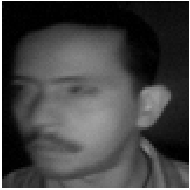}
\includegraphics[width=0.18\textwidth]{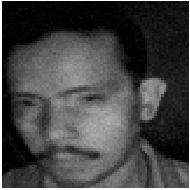}
\includegraphics[width=0.18\textwidth]{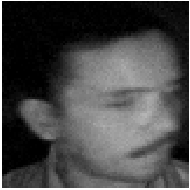}
\includegraphics[width=0.18\textwidth]{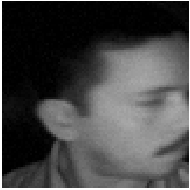}
\caption{Performance of facial image denoising with $\rho = 0.4$. The first row consists of original images while the second row again shows blurred images. The third to seventh rows contain deblurred images using km17, cf18, ya21(deg=1), ya21(deg=2) and our method, respectively. }  \label{fig:face_0.4}
\end{figure}
\newpage

\vskip 0.2in
\bibliography{references}

\end{document}